%% file: main.tex
\documentclass[11pt]{article}
\usepackage[utf8]{inputenc}
\usepackage[T1]{fontenc}
\usepackage[letterpaper,margin=1in]{geometry}
\usepackage[round,authoryear]{natbib}
\usepackage{enumitem}
\usepackage[hyphens]{url}
\usepackage{graphicx}
\usepackage{caption}
\usepackage{subcaption}
\usepackage{amsmath,amssymb,amsthm}
\usepackage{booktabs}
\usepackage{microtype}
\usepackage[hidelinks]{hyperref}
\input{glyphtounicode}
\newtheorem{theorem}{Theorem}
\newtheorem{proposition}{Proposition}
\newtheorem{lemma}{Lemma}
\newtheorem{corollary}{Corollary}
\newtheorem{definition}{Definition}
\newtheorem{remark}{Remark}

\newcommand{\E}{\mathbb{E}}
\newcommand{\supp}{\operatorname{supp}}

\title{The Cost of Conservation: Coordination-Memory Laws \\for Exact-Support Generation}
\author{
  Zhen Zhang \quad Amr Alanwar\\[0.5em]
  {\small School of Computation, Information and Technology}\\
  {\small Technical University of Munich, Germany}\\
  {\footnotesize\texttt{\{zhenzhang.zhang, alanwar\}@tum.de}}
}
\date{}

\begin{document}
\maketitle

\begin{abstract}
Many AI systems make decisions locally, even when every realized output must
obey an additive conservation law, such as selecting exactly a fixed number of
items. This constraint can be statistically invisible: small subsets of a
balanced fixed-budget output look increasingly independent, yet
communication-free coordinate-parallel generation needs exponentially many
pre-shared plans, while a sequential exact sampler needs only logarithmic
memory.
We study product measures conditioned on additive conservation laws in the
intermediate regime where one plan is selected before a fixed-order pass,
every plan is a bounded-state stochastic executor whose support is entirely
legal, and the mixture of plan laws approximates the target distribution in
total variation.  Our main result identifies the optimal asymptotic selector
rate, up to constant factors, with the killed spectral profile of the
conservation-difference walk.  The resulting coordination cost decreases as
an inverse power of live-state width, with an exponent determined by intrinsic
conservation rank rather than alphabet size; the law extends to noncentral
budgets and heterogeneous local scores.  The converse is driven by a
state-versus-resource-sum obstruction, while a rate-matching construction
compiles discrepancy control into exact-support finite-state plans.
Complementary results characterize block-parallel plan complexity and the
benefit of programmable output order.  Together, these results show precisely
how online memory substitutes for front-loaded coordination in exact-support
generation.
\end{abstract}

\section{Introduction}

Selecting exactly a prescribed number of candidates is straightforward when a
central sampler can update a remaining-budget counter. The same distribution is
harder to realize when decisions must be emitted once, in a fixed order, with
little persistent memory, and every execution must obey the quota. Under a
balanced budget, small subsets can look almost independent, and the full
constraint can be specified by a simple global score. Yet statistical and
specification simplicity need not imply execution simplicity. We study how
coordination committed before generation trades against online memory.

Feasibility, sampling the intended law, and realizing it under a restricted
architecture are different tasks; we study the third.  Before a fixed-order
pass a selector chooses a time-dependent finite-state plan from a shared
library; fresh private randomness is free, but anything affecting later
decisions must survive in the priced state, every plan has entirely legal
support, and the mixture must approximate the conditioned product law in total
variation.  A global joint sampler and a one-plan remaining-budget program are
valid high-state endpoints; the cost appears when online state is bounded and
repair is unavailable.

Figure~\ref{fig:coordination-memory-overview} separates these two resources
and previews the rank-controlled exchange law.
\input{sections/overview_figure.tex}

Our sharpest consequence concerns evaluation.  We construct two targets on one
alphabet, with the same one-token law, whose \(k\)-coordinate marginals stay
close for \(k\ll m/\sqrt r\) yet whose execution-memory requirements differ by
\(\Theta(\log r)\); \(m/\sqrt r\) is the sharp-order detectability threshold.
Matching a system's local output statistics therefore does not certify that it
can be run at comparable cost.  What produces this gap is an exchange rate: at
fixed conservation rank \(r\), cutting the asymptotic plan-selection rate by a
factor \(c\) costs \((r/2)\log_2 c\) additional persistent-state bits.

A direct AI instance is conditional-Poisson stochastic beam search (CPSBS),
which samples exactly \(k\) continuations with probability proportional to
their model weights \citep{meister2021conditional}.  Its global dynamic program
is our high-state endpoint; we price front-loaded coordination when decisions
are instead fixed-order, irrevocable, and memory-limited.  One plan is a known
bounded-width sampler \citep{kamp2011smallspace,chattopadhyay2022space}; our
object is the minimum number of truthful plans at fixed component width \(S\).
Flattening \(K\) plans into width \(KS\) erases this decomposition.

Exact support drives both bounds: prefixes with different consumed resources
cannot share a live state, so comparing two executors gives a killed
resource-difference walk whose escape profile caps one plan's target mass, and
the construction reverses that geometry.  Growing rank adds a partition-entropy
price; secondary results treat block-parallel execution and output order.

To summarize, our contributions are threefold:
\begin{itemize}
    \item We formulate exact-support generation as a coordination--memory
    problem, separating a preselected plan from persistent state during a
    fixed-order pass.
    \item We establish the law by matching a moving-layer spectral converse
    with an executable codeword-wise posterior construction, identifying the
    selector rate with the killed spectral profile up to constant factors and
    covering heterogeneous local scores and noncentral budgets.
    \item We show that fragmentation changes execution memory while staying
    invisible to local checks, and characterize block-parallel execution and
    programmable order.
\end{itemize}

The result is interface-specific, and Table~\ref{tab:scope} states its scope:
execution cost depends on where coordination may be stored and used, not on
constraint or local-statistical simplicity.
\input{sections/coordination_spine.tex}

\section{Finite-scale audits}\label{sec:exp}

\paragraph{Real-score fixed-budget curation validation.}
We exactly evaluate 200 stochastic 9-of-18 Gibbs slates built from 122,270
finite GPT-4 UltraFeedback scores \citep{cui2023ultrafeedback}; this is
target-side validation, not a natural workload.  At $\beta=1$
(Figure~\ref{fig:theory-evidence-real}), TV $0.05$ needs a median 5,938
coordinate-parallel plans (10--90 percentile: 2,939--11,477); blocks of size
$2,3,6$ need $454,87,7$, and the remaining-budget baseline one plan with at
most 4 persistent bits.  Figure~\ref{fig:theory-evidence-spectral} verifies the
rank exponents; exact enumeration checks both compilers and all identities.

\section{Related work and conclusion}\label{sec:related}

\paragraph{Sampling interfaces.}
One plan is a width-\(S\) oblivious read-once sampler
\citep{kamp2011smallspace,chattopadhyay2022space}; flattening \(K\) plans to
width \(KS\) hides our fixed-width minimum, and exact support specializes
\citet{goldreich2010huge}'s truthfulness.
CPSBS instead uses a global dynamic program for fixed-size product-of-weights
laws \citep{meister2021conditional,pervez2023scalable}; we price irrevocable
memory-limited realization and the minimal continuous quota message.
FlashSampling and reservoirs avoid full gather for different laws
\citep{ruiz2026flashsampling,jayaram2019weighted}.

\paragraph{Adjacent resource theories.}
Locality \citep{viola2012complexity,filmus2023sampling,kane2024locality,
kane2025symmetric,horacsek2026local} differs from persistent state.
Constrained coding prices one encoder graph \citep{chien1970,marcusRoth1991};
resolvability, synthesis, common information, automata, and distribution
matching price other resources \citep{hanVerdu1993,yagi2017,cuff2013synthesis,
yuTan2020exact,wyner1975common,kumar2014exact,balle2015weighted,schulte2016ccdm};
fixed trellises price another object \citep{kiefferLiao2010,henMerhav2005}.
Our increment couples a plan-mass spectral converse to exact posterior
executors; comparison and symmetrization are borrowed tools
\citep{goel2006spectral,coulhonGrigoryan1998,ahlswedeDueck1982}.
\paragraph{Conclusion.}
Online state replaces front-loaded coordination at a rank-controlled rate, at a
cost fragmentation hides.

\clearpage
\input{appendix}

\bibliographystyle{plainnat}
\bibliography{references}

\end{document}

%% file: sections/overview_figure.tex
\begin{figure*}[t]
\centering
\includegraphics[width=0.99\textwidth]{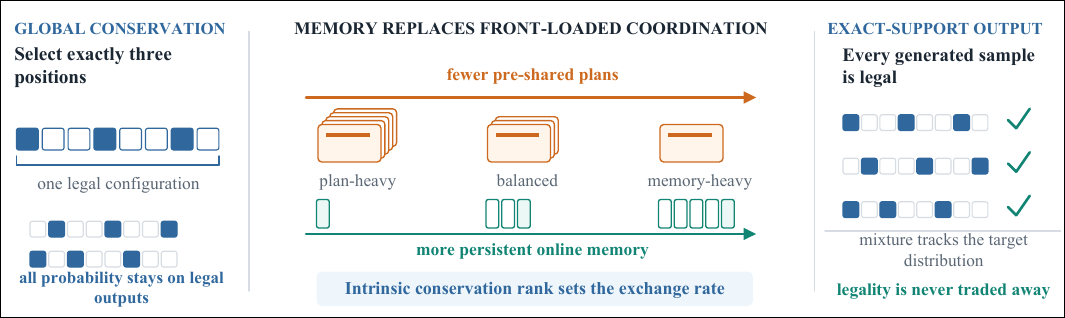}
\caption{Visual overview of the coordination--memory exchange in exact-support
generation. A global conservation rule restricts all probability to legal
configurations. The same target can be implemented with many pre-shared plans
and little persistent state, or with fewer plans and more online memory. Our
theory quantifies this exchange and shows that intrinsic conservation rank
controls its rate, while legality is maintained throughout.}
\label{fig:coordination-memory-overview}
\end{figure*}

%% file: sections/coordination_spine.tex
\section{Coordination-memory laws}\label{sec:coordrate}

\paragraph{Execution model.}
We now formalize the fixed-order interface behind the example in the
introduction.  A discrete plan \(Z\) is chosen before generation and selects
state sets \(Q_0,\ldots,Q_m\) and stochastic kernels
\(W_{t,Z}(x_t,q_t\mid q_{t-1})\).  Fresh private randomness is free.  Beyond
the front-loaded plan identity, all information that survives to the next step
must pass through \(q_t\).  Every
component has exact legal support.  We charge
\[
K=|\operatorname{supp}Z|,
\qquad
S=\max_t|Q_t|,
\]
or \(H(Z)\) instead of \(\log_2K\) for average selector cost.  Unreachable
states and states that cannot reach a legal endpoint are removed before
measuring \(S\).  We call each value of \(Z\), together with its induced
component law, a \emph{plan}; the set of plans is the \emph{library}.  The time
index and transition table are read-only, so \(\log_2K\) or \(H(Z)\) prices the
prepass program choice, not private seed length; a compact seed generating or
addressing the library is a provisioning mechanism, charged through its
executable description below.  Different plans may reuse state labels.

\paragraph{Operational interpretation.}
Equivalently, the plan is nonuniform advice selecting a compiled randomized
streaming transducer, and \(q_t\) is its runtime workspace.  An exact
fixed-order data plane must therefore spend front-loaded control or persistent
state; Table~\ref{tab:scope} records which relaxations remove that choice
altogether.
Each result below carries the appendix section proving it in full.

\paragraph{Relation and executable provisioning.}
One plan is exactly a width-\(S\) KRVZ sampler; \(K\) plans flatten to width
\(KS\), which hides the minimum \(K\) at prescribed component width.  The
selected transition table is read-only configuration: the theorem prices its
selector rate and runtime state, not its stored length or numerical precision.
This currency is not hidden for balanced fixed-alphabet targets.  For any
finite channel \(W_z\), input law \(Q_Z\), and output
\(Q_X=\sum_zQ_Z(z)W_z\), set
\(\delta_u=Q_{ZX}\{\log_2(W_Z(X)/Q_X(X))>u\}\).  If
\(Z_1,\ldots,Z_K\) are pairwise independent with marginal \(Q_Z\), then
\[
 \mathbb E\,d_{\rm TV}\!\left(K^{-1}\sum_iW_{Z_i},Q_X\right)
 \leq \delta_u+\tfrac12\sqrt{2^u/K}.                 \tag{P}
\]
Pairwise independence cancels cross-codeword covariances after truncation.  For
the balanced fixed-\(q\) compiler, let \(\theta,\pi_{\min}>0\) be its least
positive transition and initial masses; every supported plan has mass at least
\(\underline\pi_m=\pi_{\min}(\theta/q)^m\).  For any subexponential
\(\eta_m\downarrow0\), largest-remainder inverse-CDF quantization on a prime
\(p\ge\max\{K,2/(\eta_m\underline\pi_m)\}=2^{O_{q,S}(m)}\) gives
\(Q_Z(z)=(1\pm\eta_m)P_Z(z)\).  Thus TV and information density change by
\(o(1)\), so (P) preserves the selector exponent.  The affine family
\(Z_i=f_p(ai+b)\), \((a,b)\in\mathbb F_p^2\), is pairwise independent; dynamic
programming over discrepancy and the remaining histogram makes interval
decoding random-access from \(O_{q,S}(m)\) bits, not an exponential table.
The heterogeneous result remains a selector-rate existence theorem: we do not
claim the same compact descriptor for arbitrary position-specific tables.

\input{sections/plan_memory_frontier_main.tex}

\subsection{Secondary boundary: block-parallel execution}\label{sec:block-profile}

For a partition $\Pi=\{B_1,\ldots,B_g\}$ of a product base law, let $T_j$ be
the resource total in block $B_j$, set
$P_k=\rho(\cdot\mid\sum_jT_j=k)$, and let
$\nu_k=\mathcal L_{P_k}(T_1,\ldots,T_g)$.  Write
$K_{\Pi,\varepsilon}$ and $\mathsf{CE}_{\Pi,\varepsilon}$ for the minimum
count and selector entropy of block-product mixtures within TV $\varepsilon$,
and let $s_\mu(\eta)$ be the smallest support carrying mass $1-\eta$.

\begin{proposition}[Block-profile coordination sandwich]\label{thm:robustcc}
For $0\le\varepsilon<1/5$, the minimum number of block-product plans obeys
\[
 s_{\nu_k}(5\varepsilon)\le K_{\Pi,\varepsilon}(P_k)
 \le s_{\nu_k}(\varepsilon).
\]
Proved in App.~C.
\end{proposition}

The constant is explicit: a concentration bound on independent integer sums puts
a component missing $E_k$ with probability $\delta\le1/4$ within $e\delta$ TV of
a legal single-profile block product, and heavier components are replaced at
cost $1<4\delta$; averaging gives $5\varepsilon$.

\begin{corollary}[Where block-parallel randomness must live]\label{thm:commonentropy}
Exact generation has minimum selector entropy
$\mathsf{CE}_{\Pi}(P_k)=H(\nu_k)$.  For an
\(\varepsilon\)-accurate representation, let
\(M=|\supp\nu_k|\ge2\) and
\(\bar\delta=\min\{5\varepsilon,1-1/M\}\).  Then
\[
\begin{aligned}
 |\supp Z|&\ge s_{\nu_k}(5\varepsilon),\\
 H(Z)&\ge H(\nu_k)-h_2(\bar\delta)
 -\bar\delta\log_2(M-1),
\end{aligned}
\]
where \(h_2\) is binary entropy.  Proved in App.~C.
\end{corollary}
Exact support makes the independent block totals deterministic given \(Z\).
For approximation, put
\(\alpha_j=\max_r\Pr(R_j=r)\) and \(\lambda=\sum_j(1-\alpha_j)\).
The sharp Kolmogorov--Rogozin inequality \citep{juskevicius2024sharp} gives
\(\Pr\{\sum_jR_j=k\}\le F(\lambda)\), where
\(F(\lambda)=e^{-2\lambda}\{I_0(2\lambda)+I_1(2\lambda)\}\) and \(I_\nu\) is
the modified Bessel function.  For \(\delta\le1/4\),
\(1-F(\lambda)\ge\lambda/e\).  Hence \(\lambda\le e\delta\), the joint mode
has mass at least \(1-e\delta\), and it must be legal; larger \(\delta\) costs
at most \(4\delta\) by replacement.  Averaging gives the \(5\varepsilon\)
rounding; maximal coupling and Fano's inequality give the display.

\begin{corollary}[Locally product, globally expensive]\label{thm:localglobal}
For $0\le\varepsilon<1/5$, the balanced slice $P_m$, and every $h$-coordinate set $T$,
\[
 d_{\rm TV}((P_m)_T,\operatorname{Bern}(1/2)^{\otimes h})
 \le h(h-1)/(2m),
\]
yet coordinate-parallel execution requires
$|\supp Z|\ge(1-5\varepsilon)\binom{m}{m/2}$ and
$H(Z)\ge(1-5\varepsilon)m-O(\log m)$.
\end{corollary}
Thus local statistical similarity does not determine execution cost: coupling
sampling without replacement gives the local bound, while the profile sandwich
gives the global one.  Proved in App.~C.

\paragraph{Provisioning and scope.}

Let $L(\mathcal C_m)$ be the bit length of an executable description of the
whole library, including the precision used for its transition probabilities.
Reusing it for $N$ independent generations gives the amortized ledger
\[
 R_{\rm deploy}^{(N)}
 =\frac{H(Z)}m+\frac{L(\mathcal C_m)}{Nm},
 \tag{A}\label{eq:deployment-ledger}
\]
up to prefix-coding overhead; use $\log_2K$ for worst-case selection.  Every
selector-rate converse survives this extra charge.  Achievability becomes a
deployment statement only when the provisioning term is negligible at the
claimed scale.  For balanced fixed-$q$ targets and fixed $S$, the affine seeded
construction above gives an $O_{q,S}(m)$-bit random-access description, so the
ledger closes after reuse.  For arbitrary position-specific real
probabilities, the theorem guarantees a finite-state posterior for each
selected codeword but does not bound the precision or joint description of all
program-specific tables; that extension is an existence result.  No compact
descriptor theorem is claimed for the permutation library either.

Table~\ref{tab:scope} separates the priced interface from capabilities that
remove the obstruction; its comparison rows are efficient endpoints, not
counterexamples.

\begin{table}[t]
\centering
{\small
\setlength{\tabcolsep}{3.5pt}
\begin{tabular}{@{}p{.36\columnwidth}p{.30\columnwidth}p{.24\columnwidth}@{}}
\toprule
Interface & Coordination & Online state \\
\midrule
Fixed-order width-$S$ library \emph{(priced here)}
& $\Theta(S^{-2/r})$ selector rate & $\lceil\log_2S\rceil$ bits \\
Order-programmable library
& $\mathcal R^{\rm ord}_2(S)$; needs addressable output & $\lceil\log_2S\rceil$ bits \\
Coordinate-parallel product
& positive rate & constant \\
Remaining-count program
& one plan & $O_r(\log m)$ bits \\
Dropping, batchwise assignment, adaptive order, repair
& \emph{not priced} & unrestricted \\
\bottomrule
\end{tabular}
}
\caption{Scope of the coordination--memory bounds.  Rows two to four are
comparison endpoints; the last row lists the capabilities that remove the
obstruction entirely.}
\label{tab:scope}
\end{table}

%% file: sections/plan_memory_frontier_main.tex
\subsection{Core result: the spectral coordination-memory law}\label{sec:plan-memory}

We now solve the plan-versus-state problem in the execution model above.
For a target sequence \(P_m\), let \(K_m(S,\varepsilon)\) be the minimum
number of width-\(S\), fixed-order, exact-support component laws whose mixture
is within total variation \(\varepsilon\) of \(P_m\).  The selector rate is
\[
\mathcal R^{\rm fix}(S)
=\inf_{\varepsilon_m\downarrow0}\limsup_{m\to\infty}
\frac1m\log_2K_m(S,\varepsilon_m).
\]
The fixed-budget example asks for this function when \(P_m\) is a uniform
Hamming slice.  More generally, \(P_m\) will be a product law conditioned on
an additive resource total.

\paragraph{Why a killed walk appears.}
Differencing two independent source symbols gives a walk on the resource
lattice; a width-\(S\) executor retains at most \(S\) of its cumulative
differences, and the slowest-escaping such set governs both bounds.

Let $\mathcal A=[q]$, let $p_a>0$, write $p_{\min}=\min_a p_a$, and attach to
symbol $a$ a resource vector
$v_a\in\mathbb Z^d$.  The difference lattice
\[
 L=\operatorname{span}_{\mathbb Z}\{v_a-v_b:a,b\in\mathcal A\}
 \cong\mathbb Z^r
\]
has intrinsic conservation rank $r$.  After a harmless common translation,
take $v_1=0$ and $v_a\in L$.  For attainable endpoints $\ell_m$, define
$E_m=\{\sum_t v_{X_t}=\ell_m\}$, assume $P(E_m)=2^{-o(m)}$, and let
$P_m$ be the iid law $p^{\otimes m}$ conditioned on $E_m$.  Define
$\mathcal R^{\rm fix}_{p,v}(S)$ from the minimum number of exact-support
width-$S$ fixed-order laws whose mixture approaches $P_m$ in TV, using the
same vanishing-error exponential rate as above.

Formally, the symmetric difference walk is
\[
 (K_{p,v}f)(z)=\sum_{a,b}p_ap_bf(z+v_a-v_b).
\]
For $K^D_{p,v}=P_DK_{p,v}P_D$, define its killed spectral profile
\[
 \Lambda_{p,v}(S)=1-\sup_{|D|\le S}\rho(K^D_{p,v}).
\]
\begin{theorem}[Conservation-rank spectral law]\label{thm:plan-memory}
Fix an attainable endpoint sequence satisfying
$P\{\sum_t v_{X_t}=\ell_m\}=2^{-o(m)}$.  Every component is required to use
the fixed physical order, have exact legal support, and have at most $S$
reachable and co-reachable live states per layer; fresh private coins may not
persist outside that state.  Then, for every fixed $S$,
\[
 \frac{\Lambda_{p,v}(S)}{2\ln2}
 \leq\mathcal R^{\rm fix}_{p,v}(S)
 \leq\frac{8q}{p_{\min}\ln2}\Lambda_{p,v}(S).
\]
If $L$ has rank $r\ge1$, then for all sufficiently large $S$,
\[
 \mathcal R^{\rm fix}_{p,v}(S)=\Theta_{p,v}(S^{-2/r}).
\]
The exponent remains $-2/r$ for independent position-specific laws on the
same finite alphabet and resource vectors whenever
$\inf_{m,t,a}p_{m,t}(a)>0$ and the conditioned endpoint has subexponential
probability.  Thus intrinsic conservation rank,
not alphabet size or local-score heterogeneity, determines the memory
exponent.  For the uniform balanced binary specialization and odd $S$, as
$S\to\infty$,
\[
 -\log_2\cos\frac{\pi}{2S+1}
 \leq\mathcal R^{\rm fix}_{2}(S)
 \leq\frac{\pi^2}{2\ln2}S^{-2}(1+o(1)),
\]
and the lower endpoint equals
$\frac{\pi^2}{8\ln2}S^{-2}(1+o(1))$.
Proved in App.~F; binary case App.~E.
\end{theorem}

The heterogeneous extension is a selector-rate theorem: every selected codeword
has the finite-state realization in
Proposition~\ref{prop:posterior-compiler}, but we claim no compact
provisioning description for arbitrary position-specific probabilities.

Writing \(B=\log_2S\), the fixed-rank law is
\(\mathcal R^{\rm fix}(B)=\Theta(2^{-2B/r})\).  Up to model-dependent
constants, a \(c\)-fold reduction in selector rate therefore costs
\((r/2)\log_2c\) additional persistent bits.  This is the paper's core
exchange rate; its exponent depends only on intrinsic conservation rank.

\paragraph{Proof roadmap.}
The converse is layerwise.  Define
$(Tf)(z)=\sum_ap_af(z+v_a)$, so $T^*T=K_{p,v}$.  Exact support makes each
reachable state determine one prefix resource sum; otherwise one nonempty
suffix support would have to complete two different residual budgets.  If
$A_t\subset L$, $|A_t|\le S$, are the represented sums, then
\[
 \|P_{A_{t-1}}TP_{A_t}\|_2^2
 \le\rho(K^{A_t}_{p,v})\le1-\Lambda_{p,v}(S).
\]
Thus one plan captures at most
$P(E_m)^{-1}(1-\Lambda_{p,v}(S))^{m/2}$ target mass.  A union bound gives the
left side of the theorem.

The matching direction uses a weighted bounded-discrepancy compiler.  On a
finite connected $D\subset L$, take positive $f$ extended by zero, put
$h=f^2$, draw $D_0=A$ from $\pi(z)=h(z)/\sum_uh(u)$, and set
\[
 Q(b\mid x,z)=
 \frac{p_bh(z+v_x-v_b)}{\sum_cp_ch(z+v_x-v_c)},
 \qquad D_t=z+v_x-v_b.                              \tag{*}
\]
The positive identity choice $b=x$ makes this a valid channel confined to $D$.

\begin{proposition}[Codeword-wise exact posterior compiler]
\label{prop:posterior-compiler}
Let $X_1,\ldots,X_m$ be iid from $p$, generate $(A,B^m,D^m)$ by (*), set
$Z=(A,B^m)$, and condition on $E_m=\{\sum_tv_{X_t}=\ell_m\}$.  Every plan
$\zeta=(a,b^m)$ with positive conditional probability has a fixed-order
executor of width at most $|D|$ whose law is exactly
$P(X^m\mid Z=\zeta,E_m)$ and whose support is contained in $E_m$.
\end{proposition}
\begin{proof}
The cancellation identity
\[
 P(B_t=b,D_t=y\mid B^{t-1})
 =p_bh(y)/\textstyle\sum_uh(u)
\]
gives $B_t\sim p$ independently and $D_t\mid B^t\sim\pi$.  Moreover,
\[
 \sum_{s\le t}v_{X_s}=\sum_{s\le t}v_{B_s}+D_t-A.       \tag{1}
\]
For the fixed plan $\zeta=(a,b^m)$, define
\begin{align*}
 F_m(u)&=\mathbf1\!\left\{\sum_{s\le m}v_{b_s}+u-a=\ell_m\right\},\\
 F_t(u)&=\sum_xp_xQ(b_{t+1}\mid x,u)
 F_{t+1}(u+v_x-v_{b_{t+1}}).
\end{align*}
At a live state $u$ with $F_t(u)>0$, emit $x$ with probability
\[
 \frac{p_xQ(b_{t+1}\mid x,u)
 F_{t+1}(u+v_x-v_{b_{t+1}})}{F_t(u)}.                  \tag{2}
\]
Backward induction shows that $F_t(u)$ is the suffix likelihood given the
plan and discrepancy.  Thus (2) is the exact Bayes posterior.  Its only live
states are $u\in D$, and the terminal indicator together with (1) removes
every illegal word.
\end{proof}

Repeated auxiliary words may be merged by adding their mixture weights.  To
make the rate step explicit, let
$\ell_t=\log_2\{Q(B_t\mid X_t,D_{t-1})/p_{B_t}\}$ and
$G_m=P(E_m\mid A,B^m)$.  Before conditioning, the information density is
\[
 \imath(X^m;A,B^m)=\sum_{t=1}^m\ell_t+
 \log_2\frac{\pi(A)}{P(A\mid B^m)}.                 \tag{3}
\]
The finite discrepancy chain makes the first term concentrate at
$mI(X,D;B)$.  Since $A$ has at most $|D|$ values,
$P\{-\log_2P(A\mid B^m)>u\}\le|D|2^{-u}$, so the boundary term is
$o_P(m)$.  Conditioning gives the exact identity
\[
\begin{aligned}
 \imath_{E_m}&=\imath+\log_2P(E_m)-\log_2G_m,\\
 P_{A,B^m\mid E_m}\{G_m<2^{-m\delta}\}
 &\le\frac{2^{-m\delta}}{P(E_m)}.
\end{aligned}                                                    \tag{4}\]
Hence $P(E_m)=2^{-o(m)}$ leaves the information-spectrum rate unchanged.
Moreover,
\[
 I(X,D;B)\le\frac{8q}{p_{\min}\ln2}
 \frac{\langle f,(I-K_{p,v})f\rangle}{\|f\|_2^2}.
\]
Indeed, for $Q_b=p_ba_b^2/W$, the inequality $\ln u\le u-1$ gives the base-2 KL bound
$D_2(Q\|p)\le4(p_{\min}W\ln2)^{-1}
\sum_{b,c}p_bp_c(a_b-a_c)^2$; summing over $(z,x)$ is exactly the displayed
Dirichlet form.  One-shot Han--Verd\'u/Yagi resolvability therefore retains
$2^{m(I(X,D;B)+o(1))}$ auxiliary words whose posterior mixture converges to
$P_m$ in TV \citep{hanVerdu1993,yagi2017}.  Choosing $D,f$ near the killed
profile optimum proves the upper bound, and
Proposition~\ref{prop:posterior-compiler} turns every retained word into an
exact-support width-$S$ plan.

\paragraph{Constants and extensions.}
For binary $S=2R+1$, let
$D=\{-R,\ldots,R\}$, $h_d=\cos^2\{\pi d/(2(R+1))\}$, and
$H=\sum_{d\in D}h_d$.  Set
$u_j=\sin^2\{\pi j/(S+1)\}$ for $1\le j\le S$ and
$u_0=u_{S+1}=0$.  With
$\Psi(a,b)=(a+b)[1-h_2\{a/(a+b)\}]$ ($h_2$ is binary entropy), its exact
information is $c_S=(2H)^{-1}\sum_{j=0}^{S}\Psi(u_j,u_{j+1})$; a second-order
expansion and Riemann sum give $c_S=\pi^2S^{-2}/(2\ln2)(1+o(1))$.
Write $\mathcal E_K(f)=\langle f,(I-K_{p,v})f\rangle$.  Comparison paths
between the finite jump set and a lattice basis, followed by telescoping and
Cauchy--Schwarz, give
$C_1^{-1}\mathcal E_{\rm nn}(f)\leq\mathcal E_K(f)\leq
C_2\mathcal E_{\rm nn}(f)$, uniformly over finite support sets, with constants
depending only on $(p,v)$.  The lattice Nash inequality gives
$\mathcal E_{\rm nn}(f)\geq c_r|\supp f|^{-2/r}\|f\|_2^2$; a product sine on
an $R^r$ box attains $O_r(R^{-2})$, and $S\asymp R^r$ closes the exponent.
\begin{lemma}[Uniform heterogeneous stability]\label{lem:heterogeneous-main}
For fixed finite \(D\) and \(p_{m,t}(a)\ge\eta>0\), the time-varying version
of (*) has \(B_t\sim p_{m,t}\) and \(D_t\mid B^t\sim h/H\).  Its bounded
information-density increments concentrate uniformly, and its two-copy layer
Dirichlet forms are uniformly comparable to those of the uniform difference
walk.  Consequently both selector bounds retain exponent \(-2/r\); conditioning
on an endpoint of probability \(2^{-o(m)}\) does not change the rate.
\end{lemma}
Cancellation gives the marginals above.  At each fixed \(S\), finite-state
minorization (allowed-edge mass at least
\(\eta^2\min_Dh/\max_Dh\)) gives concentration; uniform Dirichlet comparison
and (4) give the two rate bounds.  Thus \(m\to\infty\) precedes
\(S\to\infty\): minorization may depend on \(S\), but the profile comparison
is uniform in \(S\).
\input{sections/rate_figure.tex}

\begin{corollary}[Finite-length quota-safe routing frontier]
\label{cor:routing-frontier}
Let $P_m$ be a target law on assignments with exactly $m/q$ tokens per expert
and define $s_{P_m}(\epsilon)=\min\{|A|:P_m(A)\ge1-\epsilon\}$.  For count-state
sets $A,B\subset\mathbb Z^{q-1}$, let
$M_{A,B}(u,v)=\mathbf1\{v-u\in\{0,e_1,\ldots,e_{q-1}\}\}$ and
$\Gamma_q(S)=\sup_{|A|,|B|\le S}\|M_{A,B}\|_2$.  Every width-$S$ exact
balanced plan has support at most $\Gamma_q(S)^m$.  For sufficiently large $S$,
$\Gamma_q(S)\le q-c_qS^{-2/(q-1)}$; for $q=2$,
$\Gamma_2(S)=2\cos\{\pi/(2S+1)\}$.  A fixed-order no-drop router with $B$
persistent bits and $C$ bits selecting a precomputed streaming program obeys
\[
 C\ge\left[\log_2s_{P_m}(\epsilon)
      -m\log_2\Gamma_q(2^B)\right]_+.                 \tag{5}
\]
whenever its output law is within TV $\epsilon$ of $P_m$.  For the uniform
balanced target,
$\mathcal R_{\rm route}(B)=\Theta_q(2^{-2B/(q-1)})$.  Moreover, constants
$A_q,D_q>0$ exist such that, with
\[
\begin{aligned}
 \Delta_m&=C_m+D_q\!\left(1+\log_2(m+1)
 +\log_2\frac1{1-\epsilon_m}\right),\\
 B_m&\ge\frac{q-1}{2}\log_2\frac{m}{\Delta_m}-A_q.
\end{aligned}                                                    \tag{6}\]
whenever $\Delta_m=o(m)$ and $\epsilon_m$ stays below one.  Proved in App.~L.
\end{corollary}
\begin{proof}
The prefix-resource injection used in the main converse maps every legal
output to a path through at most $2^B$ sums per layer; moving-layer path
counting gives the stated support bound.  The union $T$ over
$2^C$ values therefore satisfies
$|T|\le2^C\Gamma_q(2^B)^m$ and $Q_m(T)=1$; TV closeness gives
$P_m(T)\ge1-\epsilon$, proving (5).  For the uniform target,
$\log_2s_{P_m}(\epsilon)=m\log_2q-\frac{q-1}{2}\log_2m
+O_q(1+\log\frac1{1-\epsilon})$ by Stirling.  Substitution into (5) gives
(6), while Theorem~\ref{thm:plan-memory} supplies the matching asymptotic order.
\end{proof}
Thus bounded state forces $C_m=\Omega_q(m)$, whereas sublinear control forces
growing state.  Dropping, batchwise assignment, and dynamic dispatch buy out,
respectively, exact support, fixed-order streaming, and bounded persistent
state; Switch, Expert Choice, and MegaBlocks make these different purchases
\citep{fedus2022switch,zhou2022expertchoice,gale2023megablocks}.  The corollary
prices the restricted interface rather than claiming these systems instantiate it.

\paragraph{Sharded decoding consequence.}
The same conservation law also fixes how much a shard must say.  Under CPSBS
over a vocabulary split across \(g\) devices, the accepted counts
\(R=(R_1,\ldots,R_g)\) have mass proportional to
\(\mathbf1\{\sum_jr_j=k\}\prod_je_{j,r_j}\), where \(e_{j,r}\) is the degree-\(r\)
elementary-symmetric polynomial of shard \(j\)'s weights, and the shards are
conditionally independent given \(R\).  Sampling therefore needs no logit
gather, only a quota profile.

\begin{theorem}[Minimal continuous shard interface]
\label{thm:shard-interface}
Fix a shard such that both it and its complement hold at least \(k\) candidates,
and let it send a deterministic continuous summary
\(M_j(w_{I_j})\in\mathbb R^d\) before any sampling.  If a one-round coordinator
must reproduce the exact quota law for
every positive weight vector, then \(d\geq k\), even when all complementary
weights are revealed to it.  Sending
\((e_{j,1},\ldots,e_{j,k})\) attains \(d=k\).  There exists a compact
\(k\)-cube of ordered local weights on which the coefficient-to-quota map
and its inverse are Lipschitz; on any such family, the optimal deterministic
message length at quota-law TV error
\(\eta\) is \(k\log_2(1/\eta)+O_k(1)\).  Proved in App.~M.
\end{theorem}

With $a_r=e_r(w_{I_j})$, $b_r=e_r(w_{\mathcal B\setminus I_j})$, and
$p_r=\Pr(R_j=r)$, the identity
$a_r=(p_r/p_0)(b_k/b_{k-r})$ recovers the first $k$ local coefficients.
After fixing the other weights, these recover an ordered $k$-tuple, making
$M_j$ a continuous injection on an open subset of \(\mathbb R^k\), so
invariance of domain gives \(d\ge k\).  The elementary-symmetric Jacobian has
determinant \(\prod_{i<\ell}|w_i-w_\ell|\); a compact cube with gaps bounded
away from zero is bi-Lipschitz, and packing/quantizing it gives the bit bounds.

\subsection{Structural consequence: fragmentation and local invisibility}
The fixed-rank constants above can themselves scale.  Let
$\mathbf d=(d_1,\ldots,d_J)$ be a partition of $r$.  One local action selects
one of $d_j+1$ categories independently in every group $j$, and we condition
each group to use its categories equally often.  Write
$\mathcal R^{\rm fix}_{\mathbf d}(S)$ for the resulting selector rate and set
\[
 \alpha_j=d_j/r,\qquad
 \mathsf N(\mathbf d)=e^{H(\alpha_1,\ldots,\alpha_J)}
 =\frac{r}{(\prod_jd_j^{d_j})^{1/r}}.
\]
Here $\mathsf N$ is the rank-weighted effective number of blocks.

\begin{theorem}[Rank-partition conservation geometry]
\label{thm:growing-rank-geometry}
There are universal constants $d_0,W_0,c,C,c_0>0$ such that the following
holds.  Put
\[
 L_{\mathbf d}(S)=(\prod_jd_j^{d_j})^{1/r}S^{2/r},
 \qquad w_j=\sqrt{L_{\mathbf d}(S)/d_j}.
\]
Whenever every $w_j\geq W_0$, and $w_j\leq\exp(c_0d_j)$ for every
$d_j\geq d_0$,
\begin{equation}
 c\,\mathsf N(\mathbf d)S^{-2/r}
 \leq\mathcal R^{\rm fix}_{\mathbf d}(S)
 \leq C\,\mathsf N(\mathbf d)S^{-2/r}.
 \label{eq:growing-rank-law}
\end{equation}
Consequently, at any sufficiently small fixed selector-rate target $\eta$,
the minimum persistent-memory budget in this interior regime is
\begin{equation}
 B_\eta(\mathbf d)
 =\frac{r}{2\ln2}H(\alpha_1,\ldots,\alpha_J)+\Theta_\eta(r).
 \label{eq:partition-memory-main}
\end{equation}
Proved in App.~G.
\end{theorem}

The hypotheses hold simultaneously, rather than only in iterated limits.  For
$r_n=n$, $m_n=2(n+1)\lceil e^{n^4}/(2(n+1))\rceil$, and fixed $W>W_0$, take
$S_n^{\rm mono}=\lceil W^n\rceil$ and
$S_n^{\rm frag}=\lceil(W\sqrt n)^n\rceil$.  Both rates are $\Theta(W^{-2})$,
but their memory budgets differ by
$\frac n2\log_2n+o(n\log n)$ bits.
Equation~\eqref{eq:partition-memory-main} gives a continuum, not only two
examples: $k$ equal rank blocks cost
$\frac r2\log_2k+\Theta_\eta(r)$ bits.  Its endpoints are
$\frac r2\log_2r+\Theta_\eta(r)$ bits for $r$ independent binary budgets and
$\Theta_\eta(r)$ for one exclusive $(r+1)$-category quota.  Resource-neutral label duplication equalizes endpoint alphabets and one-step
entropies.  For a rank-$d$ group at scale $W$, blocking gives
$\Pr\{\imath_m>m(c_{d,W}+C_0W^{-2})\}\leq4e^{-m/\Psi(d,W)}+
W^d2^{-cm/W^2}$, where $\log\Psi(d,W)=O(d^2W^2\log(dW))$.
Endpoint conditioning and a union over $g_m$ groups cost at most
$g_m(m+1)^{d_m+1}$.  The tail therefore vanishes when
$d_m^2W_m^2\log(d_mW_m)=o(\log m)$ and $\log g_m=o(m^{1/2})$; the displayed
array satisfies both simultaneously.

\begin{theorem}[Local statistics do not price conservation]
\label{thm:local-conservation-twins}
For \(m\) divisible by \(2(r+1)\), there are two rank-\(r\) exact
conservation targets on the same \(2^r(r+1)\)-symbol alphabet, with the same
uniform one-token law, such that every \(T\subseteq[m]\), \(|T|=k\), obeys
\[
 d_{\rm TV}(P^{\rm frag}_{T},P^{\rm mono}_{T})
 \leq
 \sqrt{\frac{r\,k(k-1)}{(m-1)(m-k+1)}}.
\]
The fragmented target has \(r\) balanced binary budgets and needs
\(\frac r2\log_2r+\Theta_\eta(r)\) memory bits at fixed selector rate; the
monolithic target has one balanced \((r+1)\)-category quota and needs only
\(\Theta_\eta(r)\).  Along the simultaneous regime above, their local
distance tends to zero when \(k\sqrt r/m\to0\), but tends to one when
\(k=o(m)\) and \(k\sqrt r/m\to\infty\).  Thus \(m/\sqrt r\) is the
sharp-order local detectability threshold despite a
\(\Theta(\log r)\)-factor execution-memory gap.  Proved in App.~G.
\end{theorem}

Here is the proof mechanism.  Let $K_j$ be block $j$'s difference kernel.  For
any product-walk state set $D$, positivity and the trace bound give, for every
integer $t$,
\[
 \rho(K_D)^t\le\operatorname{tr}(K_D^t)
 \le |D|\prod_j K_j^t(0,0).                         \tag{7}
\]
For $q_j=d_j+1$ and $t/q_j\ge2$, Stirling's multinomial-mode bound gives
\[
 K_j^t(0,0)\le e\sqrt{2(d_j+1)}
 \left(\frac{2(d_j+1)}{t}\right)^{d_j/2}.
\]
Taking $t=\lceil A L_{\mathbf d}(S)\rceil$ for a sufficiently large universal
$A$, the theorem's scale assumption ensures $t/q_j\ge2$ and makes the trace
bound at most $e^{-c_1r}$.  Since $t\le2A L_{\mathbf d}(S)$,
$-\log\rho(K_D)\ge c\mathsf N(\mathbf d)S^{-2/r}$.
For the reverse direction, let $N_j$ be the uniform $q_j=d_j+1$ category
histogram after $n_j=\lfloor\alpha d_jW_j^2\rfloor$ draws and set
$D_{W_j}=\{z:P(N_j=z)\ge W_j^{-d_j}\}$.  For $A>1$, Stirling and multinomial
concentration give $P(N_j=z)\ge e^{-n_jD(z/n_j\|u_{q_j})-q_j}
(1+n_j/q_j)^{-q_j/2}$ and
$P\{n_jD(N_j/n_j\|u_{q_j})\ge Ad_j\}\le e^{-d_j(A-1-\log A)}$.
Hence $|D_{W_j}|\le W_j^{d_j}$ and
$\delta_j=P(N_j\notin D_{W_j})\le e^{-cd_j}$.  For
$Z_j\sim N_j\mid\{N_j\in D_{W_j}\}$,
$P(U=a\mid N_j+e_U=y)=y_a/(n_j+1)$ gives
$I(U;N_j+e_U)\le d_j/(n_j+1)$; conditioning changes this by at most
$2\delta_j\log q_j+2h_2(\delta_j)$.  Thus
$I(U;Z_j+e_U)\le CW_j^{-2}$ uniformly in $d_j$.
Choosing $W_j^2=L_{\mathbf d}(S)/d_j$ makes
$\prod_j|D_{W_j}|\le S$ and
$\sum_jW_j^{-2}=\mathsf N(\mathbf d)S^{-2/r}$; the posterior compiler then
proves the upper bound.  These are matching operational bounds, not only a
spectral analogy \citep{goel2006spectral,agrawal2020multinomial,
hanVerdu1993,yagi2017}.

Both targets use $(u,c)\in\{0,1\}^r\times[r+1]$.  The fragmented law has
$r$ independent balanced columns in $u$ and iid uniform $c$; the monolithic
law has balanced $c$ and iid uniform $u$.  These refinements share the uniform
one-token law.  If $S_j$ counts observed ones in coordinate $j$, then for
$T=\sum_j(S_j-k/2)^2$ the counts are respectively hypergeometric and binomial,
and
\[
\begin{aligned}
 \mathbb E_{\rm mono}T-\mathbb E_{\rm frag}T
 &=\frac{rk(k-1)}{4(m-1)},\\
 \max\{\operatorname{Var}_{\rm mono}T,
        \operatorname{Var}_{\rm frag}T\}&=O(rk^2).
\end{aligned}                                                    \tag{8}\]
Chebyshev gives TV tending to one when $k\sqrt r/m\to\infty$.  Conversely,
each target has KL to the shared iid law at most
$rk(k-1)/\{2(m-1)(m-k+1)\}$, a bound due to \citet{stam1978distance} in the
form restated by \citet[Eq.~(6)]{johnson2025sampling}; Pinsker twice
and the triangle inequality give the displayed TV bound.
Equation~\eqref{eq:partition-memory-main} supplies the execution gap.
Two consequences connect the theorem to scored selection.  If
$p_\theta(a)\propto p_a e^{\langle\theta,v_a\rangle}$, then on every attainable
endpoint $\ell$,
\[
 p_\theta^{\otimes m}(x^m\mid\textstyle\sum_tv_{x_t}=\ell)
 =p^{\otimes m}(x^m\mid\textstyle\sum_tv_{x_t}=\ell). \tag{9}
\]
Thus a noncentral budget uses the profile of its mean-matching Gibbs tilt.  A
weighted size-$k$ slate
$P(\mathcal I)\propto\exp(\beta\sum_{i\in\mathcal I}z_i)$, $|\mathcal I|=k$,
is exactly a heterogeneous Bernoulli product conditioned on its count; under
bounded local odds it has the rank-one law $\Theta(S^{-2})$.

Finite-rank CPU calculations recover slopes $-1.964$ and $-0.948$ in irregular
rank-one and rank-two examples.  Independent grouped-law checks verify the
trace, connectivity, entropy, and covariance identities to machine precision.
For the original balanced-$q$ specialization, the seeded construction gives the
random library an $O_{q,S}(m)$-bit executable description; this descriptor
result does not cover arbitrary real-valued transition tables.

\begin{corollary}[Programmable-order binary constant]
\label{cor:programmable-order}
For odd $S$, suppose each plan may hard-code a permutation of the output
coordinates before running a width-$S$ exact-support executor.  The optimal
binary selector rate is
\[
 \mathcal R^{\rm ord}_2(S)
 =-\log_2\cos\frac{\pi}{2S+1}
 =\frac{\pi^2}{8\ln2}S^{-2}(1+o(1)).                \tag{10}
\]
Thus addressable output order closes the factor-four asymptotic constant gap
that remains for immutable fixed-order streams.  Proved in App.~I.
\end{corollary}
\begin{proof}[Proof sketch]
The fixed-order support converse is permutation invariant and gives the lower
bound.  Conversely, take the connected width-$S$ count trellis whose path
growth is $\{2\cos(\pi/(2S+1))\}^m$.  Coordinate permutations act transitively
on the balanced slice, so averaging the permuted plan laws is exactly uniform.
A random sublibrary and the same second-moment covering bound as (P) retain
$2^{m(-\log_2\cos(\pi/(2S+1))+o(1))}$ permutations, proving the upper bound.
The permutation is part of the plan description and requires addressable
outputs; it is not free on an immutable stream.
\end{proof}

The remaining fixed-order gap is not a tuning artifact.  Replacing the
sine-square stationary weights by $g(j/(S+1))^2$, for any $g\in C^2[0,1]$
positive inside and vanishing at both endpoints, gives
$(S+1)^2c_S(g)\to\frac1{2\ln2}\int_0^1g'^2/\int_0^1g^2$, which the Dirichlet
Poincar\'e inequality bounds below by $\frac{\pi^2}{2\ln2}$, with equality only
at $g\propto\sin(\pi x)$.  The construction is thus leading-constant optimal in
this family: closing the gap needs a different architecture or a stronger
converse, not a better weight profile.

%% file: sections/rate_figure.tex
\begin{figure*}[t]
\centering
\begin{subfigure}[t]{0.32\textwidth}
  \centering
  \includegraphics[width=\linewidth]{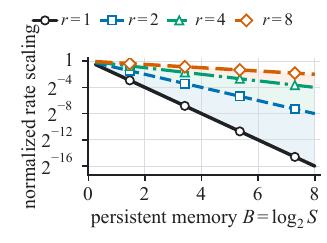}
  \subcaption{}
  \label{fig:theory-evidence-law}
\end{subfigure}\hfill
\begin{subfigure}[t]{0.32\textwidth}
  \centering
  \includegraphics[width=\linewidth]{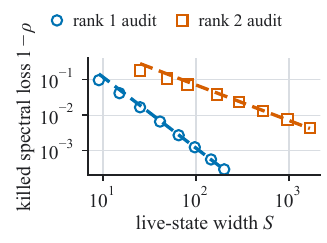}
  \subcaption{}
  \label{fig:theory-evidence-spectral}
\end{subfigure}\hfill
\begin{subfigure}[t]{0.32\textwidth}
  \centering
  \includegraphics[width=\linewidth]{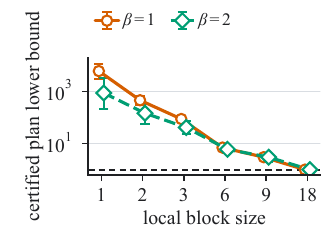}
  \subcaption{}
  \label{fig:theory-evidence-real}
\end{subfigure}
\caption{Theory and finite-scale evidence for the coordination--memory law.
Panel (a) visualizes the theorem-implied rate. In (b), exact computations on
nonuniform, irregular rank-one and rank-two walks give tail slopes $-1.964$
and $-0.948$, close to the predicted $-2$ and $-1$ (dashed lines). In (c),
points are median certified lower bounds over pinned UltraFeedback GPT-4 score
batches and error bars span the 10th--90th percentiles. Panels (b)--(c) are
numerical audits rather than runtime measurements.}
\label{fig:theory-evidence}
\end{figure*}

%% file: appendix.tex
\appendix

\makeatletter
\let\c@proposition\c@theorem
\let\c@lemma\c@theorem
\let\c@corollary\c@theorem
\let\c@definition\c@theorem
\let\c@remark\c@theorem
\@addtoreset{theorem}{section}
\makeatother
\setcounter{theorem}{0}
\renewcommand{\thetheorem}{\thesection.\arabic{theorem}}
\renewcommand{\theproposition}{\thetheorem}
\renewcommand{\thelemma}{\thetheorem}
\renewcommand{\thecorollary}{\thetheorem}
\renewcommand{\thedefinition}{\thetheorem}
\renewcommand{\theremark}{\thetheorem}
\providecommand*{\theHtheorem}{}
\renewcommand*{\theHtheorem}{appendix.\Alph{section}.\arabic{theorem}}
\providecommand*{\theHproposition}{}
\renewcommand*{\theHproposition}{\theHtheorem}
\providecommand*{\theHlemma}{}
\renewcommand*{\theHlemma}{\theHtheorem}
\providecommand*{\theHcorollary}{}
\renewcommand*{\theHcorollary}{\theHtheorem}
\providecommand*{\theHdefinition}{}
\renewcommand*{\theHdefinition}{\theHtheorem}
\providecommand*{\theHremark}{}
\renewcommand*{\theHremark}{\theHtheorem}

\section*{Claim-to-proof and scope index}

This appendix is keyed to the cross-references in the main text.  The main
text states every result and gives its proof roadmap; the appendix supplies the
complete derivations.  This index is only a navigation layer and introduces no
additional claim.

\paragraph{Interface covered by the main law.}
The selector--state bounds concern the fixed-order bounded-state interface of
Sec.~\ref{sec:plan-memory}: every selected component has legal support, fresh
within-step randomness is free, persistent information is charged through live
state, and read-only transition tables are nonuniform configuration.  The
mixture may approximate the target in total variation, but no selected
component may emit an illegal output.  Growing-state autoregression, joint
scheduling, token dropping, programmable order, and target-aware repair spend
different resources; their relation to the theorem is formalized in
Secs.~\ref{app:executor-model} and~\ref{app:interface-ledger}.

\begin{enumerate}[label=\textbf{M\arabic*.},leftmargin=*,itemsep=2.5pt,topsep=3pt]
\small
\interlinepenalty=10000

\item \textbf{Conservation-rank spectral law}
(Main Theorem~\ref{thm:plan-memory}, p.~\pageref{thm:plan-memory}).
The fixed-order converse is proved in
Sec.~\ref{sec:moving-layer-converse} (p.~\pageref{sec:moving-layer-converse});
the binary matching construction is in
Sec.~\ref{sec:fixed-frontier-v2} (p.~\pageref{sec:fixed-frontier-v2});
the general-rank and heterogeneous-position result is in
Sec.~\ref{sec:general-conservation-rank-app}
(p.~\pageref{sec:general-conservation-rank-app}).

\item \textbf{Codeword-wise exact posterior compiler}
(Main Proposition~\ref{prop:posterior-compiler},
p.~\pageref{prop:posterior-compiler}).
Weighted cancellation, Bayes reversal, finite-state realization, and
codeword-wise legal support are proved in
Secs.~\ref{sec:fixed-frontier-v2} and
\ref{sec:general-conservation-rank-app}.

\item \textbf{Uniform heterogeneous stability}
(Main Lemma~\ref{lem:heterogeneous-main},
p.~\pageref{lem:heterogeneous-main}).
Sec.~\ref{sec:general-conservation-rank-app} proves uniform minorization,
bounded information-density increments, and Dirichlet-form comparison with
constants independent of position and block length under the stated
probability floor.

\item \textbf{Finite-length quota-safe routing frontier}
(Main Corollary~\ref{cor:routing-frontier},
p.~\pageref{cor:routing-frontier}).
The finite-length support-count converse and resource accounting are in
Sec.~\ref{app:quota-safe-routing}
(p.~\pageref{app:quota-safe-routing}); the underlying balanced-type
finite-state law is in Sec.~\ref{sec:qary-plan-memory-law}
(p.~\pageref{sec:qary-plan-memory-law}).

\item \textbf{Minimal continuous shard interface}
(Main Theorem~\ref{thm:shard-interface},
p.~\pageref{thm:shard-interface}).
Sec.~\ref{app:conditional-poisson-beam}
(p.~\pageref{app:conditional-poisson-beam}) proves matching recovery and
collision lower bounds for the stated regular conditional-Poisson family.

\item \textbf{Rank-partition conservation geometry}
(Main Theorem~\ref{thm:growing-rank-geometry},
p.~\pageref{thm:growing-rank-geometry}).
Finite-rank estimates, a simultaneous triangular array, and the matching
compiler are proved in Sec.~\ref{sec:growing-rank-geometry-app}
(p.~\pageref{sec:growing-rank-geometry-app}).

\item \textbf{Local statistics do not price conservation}
(Main Theorem~\ref{thm:local-conservation-twins},
p.~\pageref{thm:local-conservation-twins}).
The common-alphabet construction and quantitative local-view bound are in
Sec.~\ref{sec:gr-local-twins} (p.~\pageref{sec:gr-local-twins}).

\item \textbf{Programmable-order constant}
(Main Corollary~\ref{cor:programmable-order},
p.~\pageref{cor:programmable-order}).
The orbit construction is in Sec.~\ref{app:orbit-symmetrization}
(p.~\pageref{app:orbit-symmetrization}), and its sharp support law is in
Sec.~\ref{app:soft-frontier} (p.~\pageref{app:soft-frontier}).
Neither statement is used to claim a sharp fixed-order constant.

\item \textbf{Block-profile coordination and local/global separation}
(Main Proposition~\ref{thm:robustcc} and
Corollaries~\ref{thm:commonentropy}--\ref{thm:localglobal},
pp.~\pageref{thm:robustcc}--\pageref{thm:localglobal}).
Dimension-free rounding, entropy continuity, and local-view coupling are
proved in Sec.~\ref{app:coordination-proofs}
(p.~\pageref{app:coordination-proofs}).
\end{enumerate}

\section*{Technical dependency and boundary index}

The entries below isolate the steps most easily obscured by the compressed
main-paper roadmaps.  Each entry states the exact location that closes the
step and, where relevant, the boundary of the resulting claim.

\paragraph{A. Converse and achievability closure.}
\begin{enumerate}[label=\textbf{A\arabic*.},leftmargin=*,itemsep=2pt,topsep=2pt]
\small
\item \textbf{State--resource-sum obstruction.}
Definition~\ref{def:fixed-order-executor} fixes the execution model and
removes non-co-reachable states; Lemma~\ref{lem:moving-layer-support} proves
that exact support forces each live state to determine one prefix resource
sum and derives the one-plan support bound.

\item \textbf{Relation to known small-space samplers.}
Proposition~\ref{prop:small-space-relation} identifies one binary plan with a
known bounded-width sampler in both directions and states the overheads.
Flattening a $K$-plan library gives width at most $KS$ but does not preserve
the prescribed width of each truthful component.

\item \textbf{Killed-walk converse.}
Lemma~\ref{lem:app-general-moving} gives the moving-layer operator bound,
Lemma~\ref{lem:moving-layer-support} converts it to a support bound, and
Lemma~\ref{lem:periodic-fk-specialization} supplies the periodic-graph
Faber--Krahn inequality used in the conversion.

\item \textbf{Intrinsic-rank exponent.}
Lemma~\ref{lem:app-general-rank-profile} proves both sides of
$\Theta(S^{-2/r})$: a lattice Nash inequality gives the lower bound and a
product-sine box gives the upper bound under a common Dirichlet-form
normalization.

\item \textbf{Information concentration and one-shot covering.}
Lemma~\ref{lem:information-stability-v2} and
Eqs.~\eqref{eq:additive-info-v2}--\eqref{eq:finite-a-tail-v2} establish
concentration.  Equation~\eqref{eq:finite-resolvability-v2} is the one-shot
resolvability statement actually used.

\item \textbf{Endpoint conditioning.}
Equation~\eqref{eq:conditioning-id-v2} is an exact decomposition of the
conditioned information density, and
Eq.~\eqref{eq:event-posterior-tail-v2} controls its residual tail.

\item \textbf{Dirichlet information cost and heterogeneous positions.}
Lemma~\ref{lem:app-general-information} proves the information--Dirichlet
comparison; Lemma~\ref{lem:app-heterogeneous-uniformity} proves the uniform
minorization and bounded increments needed for position-specific
probabilities.
\end{enumerate}

\paragraph{B. Constants and finite regimes.}
\begin{enumerate}[label=\textbf{B\arabic*.},leftmargin=*,itemsep=2pt,topsep=2pt]
\small
\item \textbf{Finite state width versus large-width asymptotics.}
Main Theorem~\ref{thm:plan-memory} gives explicit profile bounds for every
fixed $S$.  Only the simplified power law $\Theta(S^{-2/r})$ takes
$S\to\infty$ after the block-length limit.

\item \textbf{The binary factor-four gap.}
Proposition~\ref{prop:binary-profile-optimality} proves that the displayed
construction is leading-constant optimal throughout the smooth stationary
nearest-neighbor family.  Closing the remaining gap requires a different
construction or a stronger converse, not retuning.

\item \textbf{The robust $5\varepsilon$ constant.}
Lemma~\ref{lem:app-almost-fixed} derives the rounding loss from an explicit
concentration-function inequality; the two cases contribute at most
$e\delta$ and $4\delta$.

\item \textbf{Nonvacuous growing-rank window.}
Lemmas~\ref{lem:gr-profile-lower}--\ref{lem:gr-information} construct the
constants.  Theorem~\ref{thm:gr-triangular} gives one simultaneous
blocklength--rank array satisfying all hypotheses, and
Corollary~\ref{cor:gr-triangular-separation} rules out an iterated-limit
artifact.
\end{enumerate}

\paragraph{C. Resource accounting and scope.}
\begin{enumerate}[label=\textbf{C\arabic*.},leftmargin=*,itemsep=2pt,topsep=2pt]
\small
\item \textbf{Plan description is separated from selector rate.}
The core law prices selector bits and persistent live state.  For balanced
fixed-alphabet targets,
Proposition~\ref{prop:pairwise-one-shot-covering},
Lemma~\ref{lem:affine-seeded-library}, and
Corollary~\ref{cor:seeded-posterior-deployment} give an
$O_{q,S}(m)$-bit random-access library.  No compact-description theorem is
claimed for arbitrary position-specific real transition tables.

\item \textbf{A one-plan autoregressive sampler is another frontier point.}
Sec.~\ref{app:interface-ledger} shows that remaining-count autoregression
uses one plan while its live width grows with the horizon; it therefore pays
online memory rather than contradicting the selector--state law.

\item \textbf{Deployed routers leave the restricted interface explicitly.}
Corollary~\ref{cor:quota-safe-escape} and
Sec.~\ref{app:quota-safe-routing} identify which premise is relaxed by
dropping, batchwise assignment, dynamic dispatch, or target-aware repair.
The result prices an interface and is not stated as a lower bound for every
deployed router.

\item \textbf{Shard-summary injectivity.}
Theorem~\ref{thm:cpsbs-interface} first recovers the
elementary-symmetric coefficients in Eq.~\eqref{eq:cpsbs-coeff-recovery};
only then is invariance of domain applied to obtain the dimension bound.

\item \textbf{Programmable order and coverage are secondary boundaries.}
Lemma~\ref{lem:centered-corridor-support} and
Corollary~\ref{cor:order-programmable-frontier} are independent boundary
results and are not premises of the fixed-order spectral law.
\end{enumerate}

\paragraph{D. Interpretation of the numerical evidence.}
\begin{enumerate}[label=\textbf{D\arabic*.},leftmargin=*,itemsep=2pt,topsep=2pt]
\small
\item \textbf{Real-score validation is target-side.}
Sec.~\ref{app:realscore} specifies the heterogeneous-score target and exact
enumeration.  It establishes that the smooth-support lower bound survives
nonuniform real scores; it is not a learned-model, latency, or deployment
claim.

\item \textbf{The finite calculations avoid estimation noise.}
All subsets in each reported slate are enumerated exactly, so the displayed
profile masses and plan-count bounds are not Monte Carlo estimates.  The
spectral calculations test the finite-$S$ slopes against the theorem's
predicted exponents.

\item \textbf{The AI-system connection is through explicit interfaces.}
Sec.~\ref{app:conditional-poisson-beam} proves that the stated
conditional-Poisson beam law is a heterogeneous product conditioned on its
count and derives its minimum one-round shard summary.
Sec.~\ref{app:quota-safe-routing} derives the finite-length
control--state trilemma for exact-quota streaming allocation.  Each result
states the capability that must be added to leave the priced interface.
\end{enumerate}

\clearpage

\input{sections/executor_model_appendix.tex}
\input{sections/interface_ledger_appendix.tex}

\section{Full proofs for the coordination-rate results}
\label{app:coordination-proofs}

This section supplies the details omitted from the main paper.  Throughout,
the execution partition is $\Pi=\{B_1,\ldots,B_g\}$ and total variation is
$d_{\rm TV}(P,Q)=\sup_A|P(A)-Q(A)|$.

\begin{lemma}[Dimension-free stability of an almost fixed sum]
\label{lem:app-almost-fixed}
Let $R_1,\ldots,R_g$ be independent integer-valued random variables.  If
\[
 \Pr\!\left\{\sum_{j=1}^gR_j=k\right\}=1-\delta,
 \qquad 0\le\delta\le\frac14,
\]
then there are modal values $r_j\in\arg\max_r\Pr(R_j=r)$ such that
\begin{align*}
 \sum_jr_j&=k,\\
 \Pr\{(R_1,\ldots,R_g)=(r_1,\ldots,r_g)\}&\ge1-e\delta.
\end{align*}
Consequently, a block-product component that misses the conservation law
with probability $\delta$ is within $e\delta$ TV of a block-product law on
one legal profile.
\end{lemma}

\begin{proof}
Put $\alpha_j=\max_r\Pr(R_j=r)$ and
$\lambda=\sum_j(1-\alpha_j)$.  Convolution cannot increase maximal atom
mass, hence $\alpha_j\ge1-\delta\ge3/4$ for each $j$.

Use the concentration convention
$\mathcal Q_h(Y)=\sup_x\Pr\{Y\in(x,x+h]\}$.  Because $2R_j$ lies on the even
lattice, an interval of length two contains at most one atom, and hence
$\mathcal Q_2(2R_j)=\alpha_j$ and
$\mathcal Q_2(2\sum_jR_j)=\max_s\Pr(\sum_jR_j=s)$.  Apply the sharp
Kolmogorov--Rogozin inequality
\citep[Corollary~1.4]{juskevicius2024sharp} at span two to
$2R_1,\ldots,2R_g$.  Its extremal concentration is the symmetric Skellam law
of a Poisson difference, so
\begin{align*}
 1-\delta
 &\le \max_s\Pr\!\left(\sum_jR_j=s\right)
 \le F(\lambda),\\
 F(\lambda)&=e^{-2\lambda}
 \{I_0(2\lambda)+I_1(2\lambda)\},
\end{align*}
where $I_q$ is the modified Bessel function.  The recurrence for $I_q$
gives
$F'(\lambda)=e^{-2\lambda}\{I_2(2\lambda)-I_0(2\lambda)\}<0$.
Since $F(1/2)=0.67367\ldots<3/4$, we have $\lambda<1/2$.
The event $\{N_1=0,N_2=1\}$, for independent
$N_1,N_2\sim\operatorname{Pois}(\lambda)$, is outside
$\{N_1-N_2\in\{0,1\}\}$.  Therefore
\[
 \delta\ge1-F(\lambda)
 \ge \lambda e^{-2\lambda}
 \ge \lambda/e,
\]
and $\lambda\le e\delta$.

Independence and the union bound now give
\[
 q:=\prod_j\alpha_j\ge1-\lambda\ge1-e\delta>\delta.
\]
If $\sum_jr_j\ne k$, the joint modal atom would belong to the off-slice
event and have probability at most $\delta$, a contradiction.  Thus the
modal profile is legal.  Conditioning every block on $R_j=r_j$ preserves
block independence, and conditioning on their joint modal event changes the
component by exactly $1-q\le e\delta$ in TV.
\end{proof}

\paragraph{Proof of Proposition~\ref{thm:robustcc}.}
Let $Q=\sum_{z=1}^Kw_zQ_z$ be a mixture of $\Pi$-block-product laws with
$d_{\rm TV}(Q,P_k)\le\varepsilon$, and set
$\delta_z=Q_z(E_k^c)$ for $E_k=\{\sum_jT_j=k\}$.  Because $P_k(E_k)=1$,
\[
 \sum_zw_z\delta_z=Q(E_k^c)\le\varepsilon.
\]
If $\delta_z\le1/4$, Lemma~\ref{lem:app-almost-fixed} rounds $Q_z$ to a
legal single-profile block-product law at cost at most
$e\delta_z<3\delta_z$.  If $\delta_z>1/4$, replace it by any profile in
$\supp\nu_k$, at cost at most $1<4\delta_z$.  The rounded mixture $\bar Q$
therefore satisfies
\[
 d_{\rm TV}(Q,\bar Q)\le4\varepsilon,
 \qquad
 d_{\rm TV}(P_k,\bar Q)\le5\varepsilon.
\]
Its profile law is supported on at most $K$ profiles.  TV contracts under
the profile map, so these profiles have $\nu_k$-mass at least
$1-5\varepsilon$.  Hence $K\ge s_{\nu_k}(5\varepsilon)$.

Conversely, let $A$ be a smallest profile set with
$\nu_k(A)\ge1-\varepsilon$.  Given profile $r$, the conditioned-product
target factorizes as
\[
 P_k(\,\cdot\mid R=r)
 =\bigotimes_{j=1}^g\rho_j(\,\cdot\mid T_j=r_j).
\]
Thus $P_k(\cdot\mid R\in A)$ is a mixture of $|A|$ block-product laws, and
its TV distance from $P_k$ is $1-\nu_k(A)\le\varepsilon$.  This proves the
upper bound and the theorem. \qed

\paragraph{Proof of Corollary~\ref{thm:commonentropy}.}
In an exact representation, every positive-weight component must itself be
supported on $E_k$.  Conditional on $Z=z$, the block totals are independent
and their sum is constant.  Hence
\[
 0=\operatorname{Var}\!\left(\sum_jT_j\mid Z=z\right)
  =\sum_j\operatorname{Var}(T_j\mid Z=z),
\]
so the complete profile is a deterministic function of $Z$.  Therefore
$H(Z)\ge H(\nu_k)$.  Equality is attained by drawing the profile from
$\nu_k$ and then sampling blocks independently from the displayed
conditional laws.

If $M=|\supp\nu_k|=1$, this already gives zero exact common entropy and the
approximate lower bound is the trivial $H(Z)\ge0$.  Assume below that $M\ge2$.
For an $\varepsilon$-accurate representation, apply the preceding rounding
component by component.  It constructs a deterministic map $Y=f(Z)$ into
$\supp\nu_k$ whose law $\bar\nu$ obeys
$d_{\rm TV}(\bar\nu,\nu_k)\le5\varepsilon$.  Hence
$H(Z)\ge H(Y)=H(\bar\nu)$ and
$|\supp Z|\ge s_{\nu_k}(5\varepsilon)$.  A maximal coupling of
$Y\sim\bar\nu$ and $R\sim\nu_k$, followed by Fano's inequality, gives
\[
 H(\nu_k)\le H(\bar\nu)+h_2(\delta)
 +\delta\log_2(M-1),
 \qquad \delta\le5\varepsilon.
\]
The correction is increasing until $\delta=1-1/M$, giving exactly the
$\bar\delta$ stated in the main paper. \qed

\subsection{Fixed block size and the closed coordination rate}
For the uniform fixed-budget law $U_{m,k}$, take $\rho_j$ uniform on each
binary block and $T_j$ equal to its Hamming weight.  The profile law is
\[
 \nu_{\Pi,k}(r)=
 \frac{\prod_{j=1}^g\binom{|B_j|}{r_j}}{\binom{m}{k}},
 \qquad \sum_jr_j=k.
\]
The finite-sample smooth-support sandwich follows immediately from
Proposition~\ref{thm:robustcc}.

Now let every block have fixed size $b$, let $m=gb$, and let $k=gb/2$.
If $R_1,\ldots,R_g$ are i.i.d. with law
$p_b=\operatorname{Bin}(b,1/2)$, then $\nu_{\Pi,k}$ is the conditional law
of $R^g$ given $\sum_jR_j=gb/2$.  Under the product law,
\[
 -\frac1g\log_2p_b^{\otimes g}(R^g)
 \longrightarrow H_2(p_b)
\]
with exponentially small atypical probability.  The lattice local central
limit theorem gives
$\Pr(\sum_jR_j=gb/2)=\Theta_b(g^{-1/2})$; conditioning therefore preserves
the asymptotic equipartition property.  Since the conditioning contributes
only $O(\log g)$ to every conditional log probability, for every fixed
$0<\eta<1$,
\begin{align*}
 \log_2s_{\nu_{\Pi,k}}(\eta)&=gH_2(p_b)+o(g),\\
 H(\nu_{\Pi,k})&=gH_2(p_b)+o(g).
\end{align*}
The two sides of the sandwich yield
\[
 \lim_{m\to\infty}\frac1m
 \log_2K_{\Pi,\varepsilon}(U_{m,m/2})
 =\frac{H_2(p_b)}b=:c_b.
\]
At zero approximation the common-entropy rate is also $c_b$.  The standard
binomial entropy expansion
$H_2(p_b)=\frac12\log_2(\pi e b/2)+O(1/b)$ gives
$c_b=(\log_2b)/(2b)+O(1/b)$.

For completeness, the robust common-entropy statement also implies
\[
 \liminf_{m\to\infty}\frac1m
 \mathsf{CE}_{\Pi,\varepsilon}(U_{m,m/2})
 \ge c_b-\frac{5\varepsilon}{b}\log_2(b+1).
\]
Indeed the profile alphabet has at most $(b+1)^g$ points, so the Fano
correction is at most
$5\varepsilon g\log_2(b+1)+o(m)$.  We do not claim the displayed robust
entropy constant is sharp.

\subsection{Local equivalence and sequential escape}
\paragraph{Proof of Corollary~\ref{thm:localglobal}.}
Fix a population containing $m/2$ zeros and $m/2$ ones.  Sampling $r$
indices without replacement produces $(P_m)_S$; sampling with replacement
produces $r$ independent fair bits.  Couple the procedures until the first
repeated index.  Then
\begin{align*}
 d_{\rm TV}((P_m)_S,(G_m)_S)
 &\le1-\frac{(m)_r}{m^r}\\
 &\le\sum_{1\le a<c\le r}\Pr(I_a=I_c)
 =\frac{r(r-1)}{2m}.
\end{align*}
For singleton execution blocks, the profile is the output itself and is
uniform on $\binom{m}{m/2}$ points.  Proposition~\ref{thm:robustcc} and
Corollary~\ref{thm:commonentropy} therefore give
\begin{align*}
 |\supp Z|&\ge
 \left\lceil(1-5\varepsilon)\binom{m}{m/2}\right\rceil,\\
 H(Z)&\ge \log_2\binom{m}{m/2}-h_2(5\varepsilon)\\
 &\hspace{1.2em}-5\varepsilon\log_2\!\left(\binom{m}{m/2}-1\right),
\end{align*}
where replacing $5\varepsilon$ by
$\min\{5\varepsilon,1-1/M\}$ covers all finite $m$.  Stirling's formula
gives the rates in the main statement.  The quadratic reward has the
central slice as its maximizer set.  An affine reward on the hypercube fixes
coordinates with nonzero coefficients and leaves the others free, so its
maximizers form a subcube; a nontrivial central slice is not a subcube.
This proves the theorem. \qed

For a weighted fixed-budget law
$P(x)\propto\prod_iw_i^{x_i}\mathbf 1\{\sum_ix_i=k\}$, define suffix
partition functions
\[
 A_{i,r}=A_{i+1,r}+w_iA_{i+1,r-1},
 \quad A_{m+1,0}=1.
\]
An exact sequential sampler with remaining budget $r$ emits one with
probability $w_iA_{i+1,r-1}/A_{i,r}$.  The table has $O(mk)$ cells.  This
construction does not reduce output entropy; it proves only that sequential
state and front-loaded common randomness are different resources.

\paragraph{Unrestricted joint baselines.}
An explicit categorical policy $Q_{\rm cat}$ on the complete joint alphabet
sets $Q_{\rm cat}(x)=|L|^{-1}\mathbf1\{x\in L\}$ and therefore realizes a
uniform legal target with TV zero.  It is a representation ceiling, not a
bounded-width plan-library baseline.  The remaining-budget recursion above is
the corresponding structured autoregressive baseline for additive laws: it
also samples the target exactly with one plan, but its per-layer state width
grows with $m$.  These two endpoints are included to prevent the finite-$S$
lower bounds from being read as architecture-independent impossibility claims.

\input{sections/plan_memory_spectral_converse_appendix.tex}
\input{sections/plan_memory_binary_appendix.tex}
\input{sections/general_conservation_rank_appendix.tex}
\input{sections/growing_rank_geometry_appendix.tex}
\input{sections/soft_frontier_appendix.tex}
\input{sections/orbit_symmetrization_appendix.tex}
\input{sections/plan_memory_qary_appendix.tex}
\input{sections/seeded_library_appendix.tex}
\input{sections/quota_safe_router_appendix.tex}
\input{sections/conditional_poisson_beam_appendix.tex}

\section{Real-score fixed-budget validation}
\label{app:realscore}
The target-side validation uses the public binarized training-preference
release derived from UltraFeedback \citep{cui2023ultrafeedback}.  Its
documentation records release under the MIT license; the underlying scores are
GPT-4 AI feedback, not human ground truth.
We concatenate finite GPT-4 chosen and rejected absolute scores,
yielding 122,270 values, standardize them globally, and sample 200 batches
of $m=18$ distinct responses with seed 20260718.  We redistribute no response
text, only derived profile counts and aggregate statistics.  In each batch the target
over nine-item slates is
\[
 P_\beta(S)\propto
 \exp\!\left\{\beta\sum_{i\in S}z_i\right\}
 \mathbf 1\{|S|=9\},
 \qquad \beta\in\{0.5,1,2\}.
\]
All $\binom{18}{9}=48{,}620$ singleton profiles are enumerated exactly;
for larger blocks, elementary-symmetric polynomials compute profile masses.
There is no learned model, Monte Carlo estimate, or fitted parameter in
these values.

At $\varepsilon=.05$ and $\beta=1$, the median lower/upper smooth-support
bounds are 5,938/17,946 for singleton execution, and the exact profile
entropy is 13.42 bits.  Median lower bounds for block sizes
$1,2,3,6,9,18$ are respectively
$5{,}938,454,87,7,3,1$.  At singleton granularity, the median lower bounds
remain $17{,}236.5$ for $\beta=.5$ and $859.5$ for $\beta=2$.  Thus the
effect is not an artifact of a uniform slice, although stronger score tilts
concentrate the target and reduce the required plan count.  The exact
remaining-budget dynamic program uses 190 table cells (100 reachable prefix
count states).

%% file: sections/executor_model_appendix.tex
\section{Formal execution model}\label{app:executor-model}

This section fixes the stochastic-state convention used by every finite-memory
result.  It rules out an unpriced persistent random seed while leaving
within-step private randomization free.

\begin{definition}[Fixed-order stochastic executor]
\label{def:fixed-order-executor}
For a plan value $z$, a length-$m$ executor consists of finite state sets
$Q_{0,z},\ldots,Q_{m,z}$, an initial law on $Q_{0,z}$, and time-inhomogeneous
Markov kernels
\[
 W_{t,z}(x_t,q_t\mid q_{t-1}),
 \qquad t=1,\ldots,m.
\]
Its output law is obtained by summing
\[
 P_z(q_0)\prod_{t=1}^m
 W_{t,z}(x_t,q_t\mid q_{t-1})
\]
over the state path.  Equivalently, conditional on $(z,q_t)$, the future state
and output sequence is independent of all earlier states, outputs, and random
coins.  A coin may be consumed in evaluating one kernel, but information from
that coin can persist only through $q_t$ and is therefore charged to the state
width.  The live width is
$\max_t|Q^{\rm live}_{t,z}|$, where live states are both reachable and
co-reachable.  The executor is exact-support for a legal language
$L$ when $P_z(x)>0$ implies $x\in L$.
\end{definition}

A plan library is a finite mixture $P=\sum_z\pi_zP_z$ of such executors.
Identical transition tables may be merged, producing an arbitrary mixture
weight and never increasing the number of distinct plans.  Thus a uniform
random codebook with repeated codewords, as used in the resolvability proof,
is admissible after duplicate merging.
\begin{proposition}[Exact relation to small-space samplers]
\label{prop:small-space-relation}
For binary outputs, the following statements hold.
\begin{enumerate}[leftmargin=*,nosep]
\item Conditional on one plan value, a width-$S$ executor is exactly a
width-$S$ KRVZ sampler.
\item A library of $K$ width-$S$ plans is sampled by one KRVZ sampler of width
at most $KS$.
\item A width-$S$ KRVZ sampler with support contained in the legal language is
a one-plan exact-support library.
\end{enumerate}
Consequently, ordinary sampler width captures a flattened implementation but
not the minimum library size at a prescribed component width.
\end{proposition}

\begin{proof}
The first statement follows by identifying layer-$t$ states with
$Q_{t,z}$ and the probability of every output-labeled edge with
$W_{t,z}(x_t,q_t\mid q_{t-1})$; this is precisely the KRVZ definition.
For the second, the unique initial state jointly samples $z$ and the first
output transition, after which layer $t$ retains the pair $(z,q_t)$ in the
disjoint union $\bigcup_z\{z\}\times Q_{t,z}$.  This reproduces the mixture
law and has at most $KS$ live states per layer.  The third statement is the
same identification with a degenerate selector.  The ROBP-sampler relation
then follows, with the width and randomness overheads stated by
\citet[Theorem~13]{chattopadhyay2022space}.
\end{proof}

\begin{remark}[Why flattening does not solve the paper's problem]
For the uniform balanced slice, a remaining-count sampler has one plan and
width $m/2+1$.  At width one, however, each exact-support component is a
product law with deterministic total, hence a point mass; exact generation
requires all $\binom{m}{m/2}$ legal words as plans.  A single flattened-width
parameter records these endpoints but does not give the selector-state
frontier between them.  The support requirement is the sequential version of
truthfulness in the sense of \citet[Definition~2.8]{goldreich2010huge}; the
new claim is not that support inclusion itself is new.
\end{remark}

For an exact type class, Markov sufficiency has an immediate consequence used
in the moving-layer converse: every live state determines one prefix-count
vector.  If two different count vectors reached the same state, any
co-reachable suffix in that state's future support would have to complete two
different residual histograms, contradicting exact support.  Fresh private
randomness does not change this argument; a persistent seed not represented in
the state would violate the definition above.

%% file: sections/interface_ledger_appendix.tex
\section{Interface scope and provenance}
\label{app:interface-ledger}

Table~\ref{tab:scope} in the main text already records the resource
currencies and architectural escape routes.  Here we retain only the formal
component-model relation, nearest-model provenance, and boundary conditions
needed to interpret the proofs.

For binary outputs, Proposition~\ref{prop:small-space-relation} shows that a
single plan is exactly a KRVZ sampler, while a $K$-plan width-$S$ library
flattens to width at most $KS$.  This establishes the overlap rather than
hiding it.  The selector-state frontier is the constrained decomposition of
that flattened sampler into truthful low-width components, a quantity not
preserved by ordinary width.  Truthfulness here means support inclusion, as in
\citet[Definition~2.8]{goldreich2010huge}.  The core selector rate treats
transition tables as nonuniform read-only configuration.  General heterogeneous
achievability is therefore an existence result unless the separate descriptor
term $L(\mathcal C_m)$ is also controlled.

\begin{table*}[t]
\centering
{\small
\setlength{\tabcolsep}{4pt}
\begin{tabular}{@{}p{.23\textwidth}p{.31\textwidth}p{.38\textwidth}@{}}
\toprule
Object & Prior result used & Increment carried by this paper \\
\midrule
One bounded-memory component
& KRVZ small-space source \citep{kamp2011smallspace}; KRVZ/ROBP simulation
  \citep[Def.~12, Thm.~13]{chattopadhyay2022space}
& no novelty claim for the component model \\
Legal support
& truthful implementation as support inclusion
  \citep[Def.~2.8]{goldreich2010huge}
& require truthfulness after conditioning on every selected plan \\
Selector-state frontier
& flattening gives width at most $KS$
& minimize $K$ at fixed truthful component width $S$; prove the killed-profile
  upper and lower rates \\
\bottomrule
\end{tabular}
}
\caption{Nearest-model provenance.  The novelty claim is the last row, not the
finite-state component or support-inclusion concepts in the first two rows.}
\label{tab:small-space-provenance}
\end{table*}
\subsection{Boundary audit for the main claims}
\label{app:boundary-audit}
\begingroup\small

\paragraph{Order of limits and constants.}
The selector rate first takes blocklength $m\to\infty$ at fixed live width
$S$.  The exponent $-2/r$ then describes sufficiently large $S$ at fixed
intrinsic rank $r$.  Main Theorem~\ref{thm:plan-memory} gives explicit
constant-factor bounds for every fixed $S$; it does not claim a sharp leading
constant in the general lattice case.  In the balanced binary fixed-order
case, the displayed upper and lower asymptotic constants differ by four.
Main Corollary~\ref{cor:programmable-order} closes that constant gap only after
plan-specific output order is added as a priced interface capability.

\paragraph{Exact support versus approximation.}
Only the aggregate library law is TV-approximate: every component satisfies
$\supp(P_z)\subseteq L$.  Thus mixing cannot hide illegal outputs.

\paragraph{Flattening does not remove the selector price.}
Proposition~\ref{prop:small-space-relation} trades a $K$-plan width-$S$ library
for one sampler of width at most $KS$; it does not preserve prescribed component
width.  A lower bound on $K$ at fixed $S$ is therefore compatible with ordinary
flattened-width bounds.
\paragraph{What the state and selector currencies exclude.}
The persistent-state budget counts reachable and co-reachable live states.
Fresh coins may affect one transition but may persist only through that state.
The selector counts prechosen transition programs; the core theorem treats
their read-only tables as nonuniform configuration.  Therefore the
heterogeneous extension is an existence theorem, not a compact-description
theorem.  The seeded construction in Sec.~\ref{app:seeded-library} supplies a
separate executable-description result for its stated balanced specialization.

\paragraph{Why a one-plan autoregressive sampler is not a counterexample.}
A remaining-count dynamic program samples the balanced target exactly with one
plan, while its live width grows with the horizon.  It occupies a different
point of the same coordination--memory plane.  Likewise, a global scheduler,
target-aware repair, token dropping, or dynamic dispatch buys out one of the
fixed-order exact-support premises instead of violating the lower bound.

\paragraph{What the empirical material establishes.}
The numerical audits verify finite spectral slopes, exact profile masses, and
reported medians.  They are not evidence of learned-model quality or wall-clock
speedups.  Figure~\ref{fig:theory-evidence} in the main text is explicitly
labeled as a numerical audit, and Sec.~\ref{app:realscore} records the pinned
data object, deterministic computation, and aggregation protocol.
\endgroup

%% file: sections/plan_memory_spectral_converse_appendix.tex

\section{A moving-layer spectral converse}
\label{sec:moving-layer-converse}

Fix $q\ge2$, put $d=q-1$, and write
$v_1=e_1,\ldots,v_d=e_d,v_q=0$.  For finite
$A,B\subset\mathbb Z^d$, let
\begin{equation}
 M_{A,B}(x,y)=\mathbf 1\{y-x\in\{v_1,\ldots,v_q\}\}.
 \label{eq:moving-incidence}
\end{equation}

\begin{lemma}[Periodic-graph Faber-Krahn specialization]
\label{lem:periodic-fk-specialization}
Let $\Gamma_q$ have left and right copies of $\mathbb Z^d$ and edges
$x_L\sim(x+v_a)_R$ for $a\in[q]$.  There is a constant
$\kappa_q>0$, depending only on $q$, such that every nonempty finite vertex
set $H\subset\Gamma_q$ satisfies
\begin{equation}
 \lambda_1^{\rm D}(H;qI-A_{\Gamma_q})
 \ge \kappa_q |H|^{-2/d}.
 \label{eq:periodic-fk-specialization}
\end{equation}
In particular the constant is uniform over the location, shape, and number of
connected components of $H$.
\end{lemma}

\begin{proof}
Translation by $\mathbb Z^d$ acts cocompactly on $\Gamma_q$ with two vertex
orbits, and $\Gamma_q$ is connected with degree $q$.  The comparison paths can
be made explicit.  On the left orbit, compare the unit step
$x_L\to(x+e_i)_L$ with
$x_L\to(x+e_i)_R\to(x+e_i)_L$; on the right orbit use
$x_R\to x_L\to(x+e_i)_R$.  Every path has length two.  An $e_i$-increment
edge is used by at most two translated comparison paths, and a zero-increment
edge by at most $2d$.  Thus the path congestion is at most $2d$, and
telescoping plus Cauchy-Schwarz bounds the sum of the two nearest-neighbor
$\mathbb Z^d$ Dirichlet forms by $4d$ times the form on $\Gamma_q$.

If $N_d$ denotes any explicit Nash constant on $\mathbb Z^d$, applying that
inequality to the two orbits gives, for every finitely supported $f$ on
$\Gamma_q$,
\[
 \|f\|_2^{2+4/d}
 \le C'_q\,\langle f,(qI-A_{\Gamma_q})f\rangle
       \|f\|_1^{4/d}
\]
with $C'_q\le C d\,2^{2/d}N_d$ for an absolute constant $C$ determined by
the Dirichlet-form convention (the comparison above gives $C=4$ under our
normalization).  Hence all constants below are effective functions of $q$;
we do not optimize them.
If $\supp f\subseteq H$, then
$\|f\|_1\le |H|^{1/2}\|f\|_2$; division and minimization over nonzero such
$f$ prove \eqref{eq:periodic-fk-specialization} with
$\kappa_q=(C'_q)^{-1}$.  This is the uniform Faber-Krahn consequence of the
Nash inequality; the general Nash/Faber-Krahn equivalence on graphs is
developed by \citet{coulhonGrigoryan1998}.  The comparison above verifies the
extra geometric input for this particular periodic graph rather than assuming
that volume growth alone is sufficient.
\end{proof}

\begin{lemma}[Uniform moving-layer deficit]
\label{lem:moving-layer-deficit}
For each fixed $q$, there is $c_q>0$ such that, whenever
$|A|,|B|\le S$,
\begin{equation}
 \|M_{A,B}\|_2\le q-c_qS^{-2/d}.
 \label{eq:moving-layer-deficit}
\end{equation}
For $q=2$, the sharper bound
\begin{equation}
 \|M_{A,B}\|_2\le
 2\cos\!\left(\frac{\pi}{2S+1}\right)
 \label{eq:binary-moving-layer-deficit}
\end{equation}
holds.
\end{lemma}

\begin{proof}
Form the infinite bipartite graph $\Gamma_q$ with left and right copies of
$\mathbb Z^d$ and edges $x_L\sim(x+v_a)_R$ for $a\in[q]$.  This is a
connected, $q$-regular periodic graph of volume-growth degree $d$.  On the
induced vertex set $H=A_L\cup B_R$, the adjacency matrix is
\[
 \begin{pmatrix}0&M_{A,B}\\M_{A,B}^{\mathsf T}&0\end{pmatrix},
\]
  whose spectral radius is $\|M_{A,B}\|_2$.  Therefore the first Dirichlet
eigenvalue of the combinatorial Laplacian $qI-A_{\Gamma_q}$ on $H$ is
  $q-\|M_{A,B}\|_2$.  Lemma~\ref{lem:periodic-fk-specialization} gives
\[
 q-\|M_{A,B}\|_2\ge \kappa_q|H|^{-2/d};
\]
  since $|H|\le2S$, this proves \eqref{eq:moving-layer-deficit} with
  $c_q=\kappa_q2^{-2/d}$.  The preceding lemma already covers disconnected
  $H$.

For $q=2$, $\Gamma_2$ is the infinite path.  Every induced subgraph on at
most $2S$ vertices is a union of paths, so its spectral radius is no larger
than that of the path on $2S$ vertices, namely
$2\cos(\pi/(2S+1))$.
This value is the elementary spectrum of a finite path; no external
constrained-coding capacity formula is invoked for the constant.
\end{proof}

\begin{lemma}[One-plan support bound]
\label{lem:moving-layer-support}
Let one fixed-order width-$S$ stochastic plan be supported entirely on the
balanced $q$-ary type class of length $m$.  Its output support has size at
most
\begin{equation}
 \big(q-c_qS^{-2/(q-1)}\big)^m.
 \label{eq:one-plan-support}
\end{equation}
For $q=2$, the right side can be replaced by
$[2\cos(\pi/(2S+1))]^m$.
\end{lemma}

\begin{proof}
By the Markov sufficiency in Definition~\ref{def:fixed-order-executor}, every
reachable machine state determines the current prefix-count vector.
Indeed, if two prefixes with different count vectors reached the same state,
the state would have the same nonempty future-suffix support after both
prefixes, while no suffix can complete both residual histograms.  This would
contradict exact support on the prescribed type class.  States carrying the
same count vector may be merged for an upper bound.

Let $A_t\subset\mathbb Z^{q-1}$ be the count vectors represented at layer
$t$; then $|A_t|\le S$.  Every supported word traces a path through the
incidence matrices $M_t=M_{A_{t-1},A_t}$, and hence the number of supported
words is at most
\[
 e_0^{\mathsf T}M_1\cdots M_me_{\star}
 \le\prod_{t=1}^m\|M_t\|_2.
\]
Here $e_{\star}$ denotes the basis vector at the terminal count
$(m/q,\ldots,m/q)\in\mathbb Z^{q-1}$.
Lemma~\ref{lem:moving-layer-deficit} completes the proof.  The argument is
layerwise, so the sets $A_t$ may move arbitrarily with time.
\end{proof}

%% file: sections/plan_memory_binary_appendix.tex

\section{A fixed-order plan-memory frontier}\label{sec:fixed-frontier-v2}

Let
\[
 \Omega_m=\left\{x\in\{0,1\}^m:
 \sum_{t=1}^m x_t=m/2\right\},
\]
and let $U_m=\operatorname{Unif}(\Omega_m)$,
where $m$ is even.  A fixed-order width-$S$ plan is a time-inhomogeneous
stochastic generator in the sense of Definition~\ref{def:fixed-order-executor} that
emits $x_1,\ldots,x_m$ from left to right, uses
private randomness, and has at most $S$ live internal states at every layer.
Its transition tables may depend on the front-loaded plan value and on time.
Every output in its support must lie in $\Omega_m$.

Let $K^{\rm fix}_{m,\epsilon}(S)$ be the least number of such component laws
whose mixture, with arbitrary weights, is within total variation $\epsilon$
of $U_m$.  Define
\[
 \mathcal R^{\rm fix}_2(S)
 =\lim_{\epsilon\downarrow0}\limsup_{\substack{m\to\infty\\m\ {
m even}}}
   \frac1m\log_2 K^{\rm fix}_{m,\epsilon}(S).
\]

\begin{theorem}[Fixed-order plan-memory frontier]\label{thm:fixed-frontier-v2}
For odd $S$,
\begin{equation}
 1-\log_2\!\left(2\cos\frac{\pi}{2S+1}\right)
 \ \le\ \mathcal R^{\rm fix}_2(S)\ \le\ c_S,
 \label{eq:finite-S-frontier}
\end{equation}
where $c_S$ is the finite sum in \eqref{eq:cS-finite-sum} below.  Consequently,
\begin{equation}
 \frac{\pi^2}{8\ln2}\,S^{-2}(1+o(1))
 \ \le\ \mathcal R^{\rm fix}_2(S)\ \le\
 \frac{\pi^2}{2\ln2}\,S^{-2}(1+o(1)),
 \label{eq:frontier-asymptotic}
\end{equation}
and hence $\mathcal R^{\rm fix}_2(S)=\Theta(S^{-2})$.
The same order holds for general $S$ by monotonicity.
\end{theorem}

The constants in \eqref{eq:frontier-asymptotic} are related to classical
bounded-running-digital-sum capacity calculations
\citep{chien1970,immink2004}.  Our claim is the operational exchange law in
\eqref{eq:finite-S-frontier}: a mixture of fixed-order, conservation-respecting
generators attains the $S^{-2}$ order, closing the $S^{-2}$ versus $S^{-1}$
gap left by the block-residual construction.  We do not claim a new constrained-
coding eigenvalue or a sharp leading constant.

\subsection{A stationary bounded-discrepancy channel}

Write $S=2R+1$ and $\mathcal D=\{-R,\ldots,R\}$.  For $d\in\mathcal D$, set
\[
 h_d=\cos^2\!\left(\frac{\pi d}{2(R+1)}\right),
 \qquad \pi_d=\frac{h_d}{H},
\]
where $H=\sum_{d\in\mathcal D}h_d=(S+1)/2$,
and put $h_d=0$ outside $\mathcal D$.  Let $X_1,\ldots,X_m$ initially be iid
fair bits.  Draw an initial discrepancy $A=D_0\sim\pi$, independently of
$X^m$.  Given $(D_{t-1},X_t)=(d,x)$, emit an auxiliary bit $B_t=b$ with
\begin{equation}
 q(b\mid x,d)
 =\frac{h_{d+x-b}}{h_{d+x}+h_{d+x-1}},
 \qquad b\in\{0,1\},
 \label{eq:test-channel-v2}
\end{equation}
and update $D_t=d+x-b$.  The denominator is positive and the channel never
leaves $\mathcal D$.  The front-loaded plan will be $Z_m=(A,B^m)$.

\begin{lemma}[Fair auxiliary output]\label{lem:fair-output-v2}
Under the iid source and stationary initialization above, $B^m$ is iid
$\operatorname{Bern}(1/2)$ and, for every $t$,
\[
 P(D_t=d\mid B^t=b^t)=\pi_d.
\]
\end{lemma}

\begin{proof}
Suppose $D_{t-1}\mid B^{t-1}=b^{t-1}$ has law $\pi$.  For $j\in\mathcal D$,
the two ways to obtain $(B_t,D_t)=(0,j)$ give
\[
 p_{0,j}:=P(B_t=0,D_t=j\mid B^{t-1}).
\]
\begin{align*}
 p_{0,j}&=\frac12\left[
 \pi_j\frac{h_j}{h_j+h_{j-1}}
 +\pi_{j-1}\frac{h_j}{h_j+h_{j-1}}
 \right]\\
 &=\frac12\pi_j.
\end{align*}
The same calculation using predecessors $j$ and $j+1$ gives
\[
 P(B_t=1,D_t=j\mid B^{t-1})=\pi_j/2.
\]
Thus $B_t$ is fair and independent of
$B^{t-1}$, while $D_t\mid B^t$ again has law $\pi$.  Induction starts from
$D_0=A\sim\pi$.
\end{proof}

\subsection{Bayes reversal gives width-$S$ plans}

Let $E_m=\{\sum_tX_t=m/2\}$.  Condition the joint law generated above on
$E_m$, so that $X^m\sim U_m$.  Denote the resulting law by $P^{E_m}$.

\begin{lemma}[Exact posterior finite-state realization]\label{lem:bayes-compiler-v2}
For every even $m$,
\begin{equation}
 U_m(x)=\sum_z P^{E_m}_{Z_m}(z)P^{E_m}(X^m=x\mid Z_m=z).
 \label{eq:bayes-mixture-v2}
\end{equation}
For every $z=(a,b^m)$ of positive probability, the conditional law
$P^{E_m}(X^m\mid Z_m=z)$ is generated exactly in the fixed left-to-right
order by a width-$S$ stochastic machine.
\end{lemma}

\begin{proof}
Equation \eqref{eq:bayes-mixture-v2} is the law of total probability.  Fix
$z=(a,b^m)$.  For a candidate prefix, let
\[
 d_0=a,\qquad d_t=a+\sum_{s\le t}(x_s-b_s).
\]
Bayes' rule and \eqref{eq:test-channel-v2} give
\begin{equation}
\begin{aligned}
 P^{E_m}(x^m\mid z)&\ \propto\
 \mathbf 1\!\left\{\sum_tx_t=m/2\right\}\\
 &\quad\times\prod_{t=1}^m
 \frac{h_{d_t}}{h_{d_{t-1}+x_t}+h_{d_{t-1}+x_t-1}}.
\end{aligned}
 \label{eq:posterior-local-weight-v2}
\end{equation}
The denominator in \eqref{eq:posterior-local-weight-v2} is essential for the
exact posterior factorization.  Every factor is local
in $(t,d_{t-1},x_t)$ once $(a,b^m)$ is fixed.  A backward suffix partition
function over $d_t\in\mathcal D$ therefore gives the exact next-bit
probabilities.  The start state $a$ is part of the plan, and at time $t$ the
count is recovered as
\[
 \sum_{s\le t}x_s=\sum_{s\le t}b_s+d_t-a.
\]
Thus the balancing endpoint is enforced without an additional counter, and
only the $S$ possible values of $d_t$ are live.  Sampling from the resulting
suffix-DP transitions generates \eqref{eq:posterior-local-weight-v2} exactly.
\end{proof}

\subsection{Information rate and stability}

Let $\mathrm h_2$ denote binary entropy in bits and define
\begin{equation}
 c_S
 =\sum_{d\in\mathcal D}\frac{\pi_d}{2}
   \sum_{x\in\{0,1\}}
   \left[1-\mathrm h_2\big(q(1\mid x,d)\big)\right].
 \label{eq:cS-expectation}
\end{equation}
Equivalently, put $g_j=\sin(\pi j/(S+1))$, $u_j=g_j^2$ for
$j=1,\ldots,S$, and $u_0=u_{S+1}=0$.  With
\[
 \Psi(a,b)=(a+b)\left[1-\mathrm h_2\!\left(\frac{a}{a+b}\right)\right],
 \qquad \Psi(0,0)=0,
\]
we have the exact finite sum
\begin{equation}
 c_S=\frac{1}{2H}\sum_{j=0}^{S}\Psi(u_j,u_{j+1}).
 \label{eq:cS-finite-sum}
\end{equation}

\begin{lemma}[Information stability before and after balancing]
\label{lem:information-stability-v2}
For fixed $S$ and $Z_m=(A,B^m)$, the unconditioned information density obeys
\[
 \frac1m\imath(X^m;Z_m)\longrightarrow c_S
\]
with exponentially small deviation probability.  Under $P^{E_m}$,
\begin{equation}
 \frac1m\imath_{E_m}(X^m;Z_m)\longrightarrow c_S
 \quad\text{in probability}.
 \label{eq:slice-info-stability-v2}
\end{equation}
Hence the spectral sup-mutual-information rate of the conditioned plan-to-
output channel is $c_S$.
\end{lemma}

\begin{proof}
Lemma~\ref{lem:fair-output-v2} gives $P(B^m=b^m)=2^{-m}$.  Therefore
\begin{align}
 \imath(X^m;A,B^m)
 &=\log_2\frac{\pi_A\prod_{t=1}^m q(B_t\mid X_t,D_{t-1})}
 {2^{-m}P(A\mid B^m)}\notag\\
 &=\sum_{t=1}^m\log_2\!\left(
   2q(B_t\mid X_t,D_{t-1})\right)\notag\\
 &\quad+\log_2\frac{\pi_A}{P(A\mid B^m)}.
 \label{eq:additive-info-v2}
\end{align}
The discrepancy chain is primitive: choosing $b=x$ gives a positive self-loop
at every $d$; choosing $(x,b)=(1,0)$ or $(0,1)$ gives a positive move to
$d+1$ or $d-1$ whenever that state lies in $\mathcal D$.  Hence all states
communicate and the self-loop makes the period one.  On the finite set of
positive-probability edges, $\log_2(2q(b\mid x,d))$ is bounded; formal moves
that leave $\mathcal D$ are simply absent.  Thus the first term is a bounded
edge-additive functional of a finite primitive Markov chain.  Its stationary
mean is
\[
 1-H(B_t\mid X_t,D_{t-1})=c_S.
\]
For completeness, exponential concentration follows directly by tilting
the finite transition matrix: for $\theta$ near zero, its Perron root is
analytic, equals one at zero, and has logarithmic derivative $c_S$.
Chernoff's argument on either side of the mean then gives, for every
$\delta>0$, constants $C_{S,\delta},\eta_{S,\delta}>0$ such that the first
term deviates from $mc_S$ by $m\delta$ with probability at most
$C_{S,\delta}e^{-\eta_{S,\delta}m}$.

The final term in \eqref{eq:additive-info-v2} is sublinear with an
exponential tail.  Indeed, for $u>0$,
\begin{equation}
 P\{-\log_2P(A\mid B^m)>u\}
 \le S2^{-u},
 \label{eq:finite-a-tail-v2}
\end{equation}
because, conditional on each $b^m$, at most $S$ posterior atoms can be
smaller than $2^{-u}$ and their total contribution is at most $S2^{-u}$.
Also $\log_2\pi_A$ is bounded for fixed $S$.  This proves the unconditioned
claim.

For conditioning, the information densities satisfy the exact identity
\begin{equation}
\begin{aligned}
 \imath_{E_m}(X^m;Z_m)
 &=\imath(X^m;Z_m)+\log_2P(E_m)\\
 &\quad-\log_2P(E_m\mid Z_m).
\end{aligned}
 \label{eq:conditioning-id-v2}
\end{equation}
The central-binomial estimate gives $P(E_m)=\Theta(m^{-1/2})$, so the middle
term is $o(m)$.  Exponential concentration transfers through this
polynomial-probability conditioning event:
\[
 P\!\left(\left.\left|m^{-1}\imath-c_S\right|>\delta\ \right|E_m\right)
 \le \frac{C_{S,\delta}e^{-\eta_{S,\delta}m}}{P(E_m)}\longrightarrow0.
\]
Finally, for every $\delta>0$,
\begin{equation}
 P_{Z_m\mid E_m}\!\left(P(E_m\mid Z_m)<2^{-m\delta}\right)
 \le \frac{2^{-m\delta}}{P(E_m)}\longrightarrow0.
 \label{eq:event-posterior-tail-v2}
\end{equation}
Thus the last term of \eqref{eq:conditioning-id-v2} is also $o_P(m)$ under
$P^{E_m}$, proving \eqref{eq:slice-info-stability-v2}.
\end{proof}

\begin{lemma}[Quadratic rate]\label{lem:quadratic-rate-v2}
As $S\to\infty$,
\begin{equation}
 c_S=\frac{\pi^2}{2\ln2}\,S^{-2}(1+o(1)).
 \label{eq:cS-asymptotic-v2}
\end{equation}
\end{lemma}

\begin{proof}
For nearby positive $a,b$,
\[
 \Psi(a,b)
 =\frac{(a-b)^2}{2\ln2\,(a+b)}
 +O\!\left(\frac{(a-b)^4}{(a+b)^3}\right).
\]
Set $\Delta=(S+1)^{-1}$ and $u(x)=\sin^2(\pi x)$.  The two boundary edges in
\eqref{eq:cS-finite-sum} contribute $O(\Delta^3)$ after division by $2H$.
On the remaining edges, a Riemann-sum argument gives
\begin{align*}
 \sum_{j=0}^{S}\frac{(u_{j+1}-u_j)^2}{u_j+u_{j+1}}
 &=\frac{\Delta}{2}\int_0^1\frac{u'(x)^2}{u(x)}\,dx+o(\Delta)\\
 &=\pi^2\Delta+o(\Delta).
\end{align*}
The Taylor remainders contribute $o(\Delta)$ to this sum.  Quadratic boundary
vanishing keeps $u'^2/u$ bounded at the endpoints.
Since $H=(S+1)/2=(2\Delta)^{-1}$, substitution into
\eqref{eq:cS-finite-sum} yields
\[
c_S=\frac{\pi^2\Delta^2}{2\ln2}+o(\Delta^2),
\]
which is \eqref{eq:cS-asymptotic-v2}.
\end{proof}

\begin{proposition}[Retuning the stationary profile cannot close the constant gap]
\label{prop:binary-profile-optimality}
Replace the sine-square weights in the stationary fair-output channel by
$h_j^{(S)}=g(j/(S+1))^2$, where $g\in C^2[0,1]$ is nonzero,
$g(0)=g(1)=0$, and $g>0$ on $(0,1)$.  Let $c_S(g)$ be the resulting exact
information rate, obtained from \eqref{eq:cS-finite-sum} after replacing
$u_j$ by $h_j^{(S)}$.  Then
\[
 \lim_{S\to\infty}(S+1)^2c_S(g)
 =\frac{1}{2\ln2}
   \frac{\int_0^1 g'(x)^2\,dx}{\int_0^1g(x)^2\,dx}
 \ge\frac{\pi^2}{2\ln2}.
\]
Equality holds for $g(x)\propto\sin(\pi x)$.  Thus the current construction is
leading-constant optimal within this smooth stationary nearest-neighbor
fair-output family; closing the factor-four theorem gap requires a different
construction architecture or a stronger converse, not merely a new stationary
weight profile.
\end{proposition}

\begin{proof}
The cancellation in Lemma~\ref{lem:fair-output-v2} uses only
$\pi_j\propto h_j$ and therefore remains valid for these profiles.  Uniformly
away from the endpoints, the expansion defining $\Psi$ gives
\[
 \Psi(h_j,h_{j+1})
 =\frac{g'(j\Delta)^2}{\ln2}\Delta^2+o(\Delta^2),
 \qquad \Delta=(S+1)^{-1}.
\]
The quadratic endpoint vanishing supplies the same summable bound in the two
boundary layers.  Hence the numerator of \eqref{eq:cS-finite-sum} is
$\Delta(\ln2)^{-1}\int_0^1g'^2+o(\Delta)$, while
$2H=2\Delta^{-1}\int_0^1g^2+o(\Delta^{-1})$.  Their ratio gives the limit.
The final inequality is the Dirichlet Poincar\'e inequality on $[0,1]$, whose
ground state is $\sin(\pi x)$.
\end{proof}

\subsection{Achievability and converse}

\begin{proposition}[Fixed-order achievability]\label{prop:fixed-ach-v2}
$\mathcal R^{\rm fix}_2(S)\le c_S$.
\end{proposition}

\begin{proof}
View $Z_m\sim P^{E_m}_{Z_m}$ as the input to the general channel
$W_m(x\mid z)=P^{E_m}(X^m=x\mid Z_m=z)$, whose output is $U_m$.  The standard
one-shot truncation form of finite-length variational-distance resolvability
for general sources and channels \citep{hanVerdu1993,yagi2017}, proved
explicitly in Lemma~\ref{lem:qary-one-shot-resolvability}, gives that, for every integer
$M$ and threshold $r$ in bits, there is a deterministic map from a uniform
$M$-point random variable to channel inputs such that
\begin{equation}
 d_{\rm TV}(\widetilde P_{X^m},U_m)
 \le P^{E_m}\{\imath_{E_m}(X^m;Z_m)>r\}
    +\frac12\sqrt{\frac{2^r}{M}}.
 \label{eq:finite-resolvability-v2}
\end{equation}
Choose $r=m(c_S+\delta)$ and
$M=\lceil2^{m(c_S+2\delta)}\rceil$.  Lemma
\ref{lem:information-stability-v2} sends the first term to zero, and the
second is at most $2^{-m\delta/2}/2$.  Every selected channel input
$z=(a,b^m)$ indexes a fixed-order width-$S$ component by Lemma
\ref{lem:bayes-compiler-v2}.  Repeated codewords are harmless: merging equal
$z$ values creates an arbitrary-weight mixture on no more than $M$ distinct
transition tables, which is allowed by the plan-library definition.  Letting
$m\to\infty$ and then $\delta,\epsilon\downarrow0$ proves the claim.
\end{proof}

\begin{proposition}[Support converse]\label{prop:fixed-converse-v2}
\[
 \mathcal R^{\rm fix}_2(S)
 \ge 1-\log_2\!\left(2\cos\frac{\pi}{2S+1}\right).
\]
\end{proposition}

\begin{proof}
Lemma~\ref{lem:moving-layer-support} gives at most
$\lambda_S^m$ supported words per component, where
$\lambda_S=2\cos(\pi/(2S+1))$.  If a $K$-component mixture is within TV
$\epsilon$ of the flat law $U_m$, the union of its supports contains at least
$(1-\epsilon)|\Omega_m|$ words.  Therefore
\[
 K\lambda_S^m\ge(1-\epsilon)\binom{m}{m/2}.
\]
Taking the per-symbol logarithm and using
$\log_2\binom{m}{m/2}=m-O(\log m)$ proves the finite-$S$ lower bound.
Finally,
\[
 1-\log_2\!\left(2\cos\frac{\pi}{2S+1}\right)
 =\frac{\pi^2}{8\ln2}S^{-2}(1+o(1)).
\]
\end{proof}

Propositions~\ref{prop:fixed-ach-v2} and
\ref{prop:fixed-converse-v2}, followed by Lemma
\ref{lem:quadratic-rate-v2}, prove Theorem
\ref{thm:fixed-frontier-v2}.

\begin{remark}[What remains open]
The constants in \eqref{eq:frontier-asymptotic} differ by a factor of four.
The lower bound is an order-zero support converse, while the upper bound is
the order-one information rate of one explicit test channel.  The theorem
closes the exponent gap, not the exact-constant problem.  A direct numerical
optimization of arbitrary positive stationary profiles at
$S\in\{3,5,9,17,33\}$ improves the finite-$S$ rate by only $1.1$-$2.9\%$;
this numerical optimization and Proposition~\ref{prop:binary-profile-optimality}
show why this retuning cannot improve the leading constant.
\end{remark}

\begin{remark}[Finite verification]
Exhaustive enumeration for $m=8$ and radii $R=1,2,3$ gives source-channel
normalization errors below $4\times10^{-15}$, an iid-fair $B^m$ marginal
error below $3\times10^{-18}$, Bayes reconstruction TV below
$2\times10^{-16}$, and errors below $3\times10^{-15}$ in both
\eqref{eq:additive-info-v2} and \eqref{eq:conditioning-id-v2}.  At
$(m,R)=(10,2)$, the same errors remain below $7\times10^{-15}$ and sampled
codebook median TV falls from $0.279$ at $M=1$ to $0.042$ at $M=64$.
\end{remark}

%% file: sections/general_conservation_rank_appendix.tex

\section{General conservation-rank spectral law}
\label{sec:general-conservation-rank-app}

Let $\mathcal A=[q]$, let $p_a>0$, write $p_{\min}=\min_a p_a$, and attach a resource vector
$v_a\in\mathbb Z^d$ to each symbol.  Work intrinsically on the difference
lattice
\[
 L=\operatorname{span}_{\mathbb Z}\{v_a-v_b:a,b\in\mathcal A\}
 \cong\mathbb Z^r.
\]
Subtract $v_1$ from every increment.  This translates every length-$t$ resource
sum by the deterministic vector $tv_1$ and changes neither endpoint fibers nor
any plan requirement.  We may therefore assume $v_1=0$ and $v_a\in L$.
For an attainable endpoint $\ell_m$, write
$E_m=\{\sum_t v_{X_t}=\ell_m\}$ and assume $P(E_m)=2^{-o(m)}$.
The target is $P_m=p^{\otimes m}(\cdot\mid E_m)$.

Define the convolution and symmetric difference-walk operators
\begin{align}
 (Tf)(z)&=\sum_a p_af(z+v_a),\label{eq:app-general-T}\\
 (Kf)(z)&=\sum_{a,b}p_ap_bf(z+v_a-v_b),\label{eq:app-general-K}
\end{align}
so that $T^*T=K$.  For finite $D\subset L$, let $K_D=P_DKP_D$ and define
\begin{align}
 \lambda(D)
 &=\inf_{\substack{f\ne0\\\supp f\subseteq D}}
 \frac{\langle f,(I-K)f\rangle}{\|f\|_2^2}
 =1-\rho(K_D),\notag\\
 \Lambda_{p,v}(S)
 &=\inf_{1\le |D|\le S}\lambda(D)
 =1-\sup_{|D|\le S}\rho(K_D).
 \label{eq:app-general-profile}
\end{align}

\subsection{Moving-layer converse}

\begin{lemma}[General moving-layer deficit]
\label{lem:app-general-moving}
For finite $A,B\subset L$ with $|A|,|B|\le S$,
\[
 \|P_ATP_B\|_{2\to2}^2\le1-\Lambda_{p,v}(S).
\]
\end{lemma}

\begin{proof}
Because $P_A\preceq I$,
\begin{align*}
 (P_ATP_B)^*(P_ATP_B)
 &=P_BT^*P_ATP_B\\
 &\preceq P_BT^*TP_B=P_BKP_B.
\end{align*}
The largest eigenvalue of the last operator is at most
$1-\Lambda_{p,v}(S)$ by definition.
\end{proof}

The support of one exact-support width-$S$ plan has at most $S$ represented
resource sums at every layer.  Indeed, if two positive-probability prefixes
with different resource sums reached the same state, the same nonempty suffix
support would have to complete both to $\ell_m$, which is impossible.  Let
$A_t$ be the represented sums at layer $t$.  Under the unconditioned product
law, the supported-path mass is at most
\begin{align*}
 &\left\langle e_0,
 P_{A_0}TP_{A_1}T\cdots P_{A_{m-1}}TP_{A_m}e_{\ell_m}
 \right\rangle\\
 &\hspace{35mm}\le(1-\Lambda_{p,v}(S))^{m/2}.
\end{align*}
After conditioning, one plan therefore has target mass at most
$P(E_m)^{-1}(1-\Lambda_{p,v}(S))^{m/2}$.  If a $K$-plan mixture is within TV
$\varepsilon$ of $P_m$, the union of its supports has target mass at least
$1-\varepsilon$.  A union bound gives
\[
 K\ge(1-\varepsilon)P(E_m)
 (1-\Lambda_{p,v}(S))^{-m/2}.
\]
Taking rates and using $-\log(1-u)\ge u$ proves
$\mathcal R^{\rm fix}_{p,v}(S)\ge\Lambda_{p,v}(S)/(2\ln2)$.

\subsection{A marginal-preserving discrepancy compiler}

Fix a finite connected $D\subset L$ and a strictly positive $f$ on $D$,
extended by zero.  Put $h=f^2$, $H=\sum_zh(z)$, and
$\pi(z)=h(z)/H$.  Draw $D_0=A\sim\pi$.  Given $D_{t-1}=z$ and
$X_t=x\sim p$, define
\begin{equation}
 Q(b\mid x,z)
 =\frac{p_bh(z+v_x-v_b)}{\sum_cp_ch(z+v_x-v_c)},
 \qquad D_t=z+v_x-v_b.
 \label{eq:app-general-channel}
\end{equation}
The identity choice $b=x$ keeps the state at $z$, so the denominator is
positive and the chain remains in $D$.

\begin{lemma}[Weighted cancellation]
\label{lem:app-general-cancellation}
For every $t$, $B_t$ has law $p$ independently of $B^{t-1}$, and
$D_t\mid B^t$ has law $\pi$.
\end{lemma}

\begin{proof}
Assume $D_{t-1}\mid B^{t-1}\sim\pi$.  For fixed $b,y$, the predecessor
paired with source symbol $x$ is $y-v_x+v_b$.  Hence
\begin{align*}
 &P(B_t=b,D_t=y\mid B^{t-1})\\
 &=\sum_xp_x\frac{h(y-v_x+v_b)}H
 \frac{p_bh(y)}{\sum_cp_ch(y+v_b-v_c)}
 =p_b\frac{h(y)}H,
\end{align*}
because the numerator sum equals the denominator after renaming $x$ as $c$.
\end{proof}

Condition this joint law on $E_m$ and front-load the plan $Z=(A,B^m)$.
For every positive-probability plan,
\begin{equation}
 \sum_{s\le t}v_{X_s}
 =\sum_{s\le t}v_{B_s}+D_t-A.
 \label{eq:app-general-resource-identity}
\end{equation}
Thus the plan and current discrepancy determine the remaining resource
budget.  The suffix recursion
\begin{align*}
 F_m(u)&=\mathbf1\left\{
 \sum_{s\le m}v_{b_s}+u-a=\ell_m\right\},\\
 F_t(u)&=\sum_xp_xQ(b_{t+1}\mid x,u)
 F_{t+1}(u+v_x-v_{b_{t+1}})
\end{align*}
gives the posterior transition
\begin{align*}
 &P(X_{t+1}=x\mid D_t=u,Z,E_m)\\
 &\quad=\frac{p_xQ(b_{t+1}\mid x,u)}{F_t(u)}
 F_{t+1}(u+v_x-v_{b_{t+1}}).
\end{align*}
It uses only the $|D|$ possible discrepancy states and every supported word
lies in $E_m$.

\subsection{Information cost and resolvability}

The stationary one-step information cost is
\[
 c_D(f)=\sum_{z,x}\pi(z)p_x
 D_2\!\left(Q(\cdot\mid x,z)\middle\|p\right).
\]

\begin{lemma}[Information-Dirichlet comparison]
\label{lem:app-general-information}
\[
 c_D(f)\le\frac{8q}{p_{\min}\ln2}
 \frac{\langle f,(I-K)f\rangle}{\|f\|_2^2}.
\]
\end{lemma}

\begin{proof}
Fix $(z,x)$ and write
$a_b=f(z+v_x-v_b)$ and $W=\sum_bp_ba_b^2$.  Then
$Q_b=p_ba_b^2/W$.  The inequality $\ln u\le u-1$, Jensen's inequality, and
$a_b^2\le W/p_{\min}$ give
\[
 D_2(Q\|p)
 \le\frac{4}{p_{\min}W\ln2}
 \sum_{b,c}p_bp_c(a_b-a_c)^2.
\]
The identity term gives $W\ge p_xf(z)^2$.  Multiply by
$\pi(z)p_x$, sum, translate $y=z+v_x$, and use the Dirichlet-form identity:
\begin{align*}
 c_D(f)
 &\le\frac{4q}{p_{\min}H\ln2}
 \sum_{y,b,c}p_bp_c(f(y-v_b)-f(y-v_c))^2\\
 &=\frac{8q}{p_{\min}\ln2}
 \frac{\langle f,(I-K)f\rangle}{\|f\|_2^2}.
\end{align*}
\end{proof}

For completeness, put
$\ell_t=\log_2[Q(B_t\mid X_t,D_{t-1})/p_{B_t}]$.
Before endpoint conditioning the information density is
\begin{align}
 \imath(X^m;A,B^m)
 &=\sum_{t=1}^m\ell_t
 +\log_2\frac{\pi(A)}{P(A\mid B^m)}.
 \label{eq:app-general-info-density}
\end{align}
On a connected component, the discrepancy chain is finite, irreducible, and
aperiodic, and the first sum is a bounded edge-additive functional with mean
$mc_D(f)$ and exponentially small deviations.  Since $A$ has at most $|D|$
values,
\[
 P\{-\log_2P(A\mid B^m)>u\}\le |D|2^{-u},
\]
so the final term is $o_P(m)$.  Endpoint conditioning preserves the limit.
Writing $G_m=P(E_m\mid A,B^m)$,
\begin{align*}
 \imath_{E_m}&=\imath+\log_2P(E_m)-\log_2G_m,\\
 P_{A,B^m\mid E_m}\{G_m<2^{-m\delta}\}
 &\le2^{-m\delta}/P(E_m)\longrightarrow0.
\end{align*}
The one-shot truncation proof for general-source/channel resolvability
\citep{hanVerdu1993,yagi2017} now gives a codebook of
$2^{m(c_D(f)+o(1))}$ plans whose posterior mixture converges in TV to $P_m$.
Every selected plan has the exact width-$|D|$ implementation above.  Choosing
a connected component of a near-minimizer in
\eqref{eq:app-general-profile} proves the upper profile bound in
Theorem~\ref{thm:plan-memory}.

\subsection{Intrinsic-rank profile and heterogeneous positions}

\begin{lemma}[Finite-range lattice profile]
\label{lem:app-general-rank-profile}
If the differences generate a rank-$r$ lattice, then
$\Lambda_{p,v}(S)=\Theta_{p,v}(S^{-2/r})$.
\end{lemma}

\begin{proof}
Choose a basis $e_1,\ldots,e_r$ of $L$.  Let $\mathcal E_K$ be the
Dirichlet form of $K$ and let
\[
 \mathcal E_{\rm nn}(f)=\frac12\sum_{z,i}
 (f(z+e_i)-f(z))^2.
\]
Every basis step is a bounded-length word in the finite difference jump set,
and every difference jump is a bounded-length word in the basis steps.
Telescoping, Cauchy-Schwarz, fixed positive edge weights, and translation
invariance therefore give constants $C_1,C_2$ with
\[
 C_1^{-1}\mathcal E_{\rm nn}(f)
 \le\mathcal E_K(f)\le C_2\mathcal E_{\rm nn}(f).
\]
The lattice Nash inequality, equivalently the graph Faber-Krahn profile
\citep{coulhonGrigoryan1998}, gives
$\mathcal E_{\rm nn}(f)/\|f\|_2^2\ge c_r|\supp f|^{-2/r}$.
For the reverse direction, a product sine on an $R^r$ lattice box has
nearest-neighbor Rayleigh quotient $O_r(R^{-2})$.  The form comparison and
$S\asymp R^r$ complete the proof.
\end{proof}

Now let position $t$ have law $p_{m,t}$ with
$\inf_{m,t,a}p_{m,t}(a)\ge\eta>0$, and assume that the conditioned endpoint
event has probability $2^{-o(m)}$.  The layer-specific Dirichlet forms are
uniformly comparable to the same nearest-neighbor form, so the converse above
has a uniform $c_{\eta,v}S^{-2/r}$ deficit.  For achievability, use one box
$D_R$, one product-sine $f_R$, and replace $p$ by $p_{m,t}$ in
\eqref{eq:app-general-channel} at time $t$.  Weighted cancellation is
pointwise in the position law, hence $B_t\sim p_{m,t}$ independently and
$D_t\mid B^t\sim\pi_R$.

\begin{lemma}[Uniformity for heterogeneous positions]
\label{lem:app-heterogeneous-uniformity}
Fix a finite connected discrepancy set $D$ and $h>0$ on $D$.  Put
$h_-:=\min_Dh$, $h_+:=\max_Dh$, and
$\kappa:=\eta^2h_-/h_+$.  Let $M_{m,t}$ be the unconditional discrepancy
kernel induced by the time-$t$ channel,
\[
 M_{m,t}(z,y)=
 \sum_{x,b:\,z+v_x-v_b=y}
 p_{m,t}(x)Q_{m,t}(b\mid x,z).
\]
Every allowed difference edge, including every self-loop, has probability at
least $\kappa$.  If $L_D$ is the diameter of the difference graph on $D$, then
for every starting time $t$ and $z,y\in D$,
\[
 (M_{m,t+1}\cdots M_{m,t+L_D})(z,y)\ge\kappa^{L_D},
\]
where shorter paths are padded by self-loops.  Thus these inhomogeneous blocks
obey a common Doeblin minorization with constant
$\alpha_D=|D|\kappa^{L_D}$ and the uniform law on $D$.  Moreover, every
positive one-step information density satisfies
\[
 \left|\log_2\frac{Q_{m,t}(b\mid x,z)}{p_{m,t}(b)}\right|
 \le \log_2\frac{h_+}{\eta h_-}.
\]
All constants depend only on $(\eta,D,h,v)$, not on $(m,t)$.
\end{lemma}

\begin{proof}
Write
$W_{m,t}(x,z)=\sum_cp_{m,t}(c)h(z+v_x-v_c)$.  The identity choice $c=x$
and positivity give
$\eta h_-\le W_{m,t}(x,z)\le h_+$.  If
$y=z+v_x-v_b\in D$, its contribution to $M_{m,t}(z,y)$ is
\[
 p_{m,t}(x)p_{m,t}(b)h(y)/W_{m,t}(x,z)\ge\kappa.
\]
The difference graph is symmetric, connected by assumption, and contains all
self-loops.  A path of length at most $L_D\le |D|-1$ therefore connects any
ordered pair; padding and multiplying the edge bounds proves the block bound
and hence the minorization.  Finally,
\[
 \frac{h_-}{h_+}
 \le\frac{Q_{m,t}(b\mid x,z)}{p_{m,t}(b)}
 =\frac{h(z+v_x-v_b)}{W_{m,t}(x,z)}
 \le\frac{h_+}{\eta h_-}
\]
on every positive transition, proving the information-density bound.
\end{proof}

For each fixed $R$, Lemma~\ref{lem:app-heterogeneous-uniformity} gives
uniform geometric mixing in $(m,t)$.  Blocking controls the state-dependent
part of the information-density sum, while its conditionally centered part is
a bounded martingale difference.  Standard exponential concentration follows
with constants depending on $(\eta,R,v)$ but not on $(m,t)$.  Applying
Lemma~\ref{lem:app-general-information} at each layer and the box Rayleigh
quotient gives average cost $O_{\eta,v}(R^{-2})$.  Endpoint conditioning and
resolvability then proceed as above.  The order of limits is important: the
blocklength tends to infinity for each fixed $R$ before $R$ grows.  This proves
the heterogeneous rank law stated in the main paper.

Finally, two exact representations explain the scope.  Exponential tilting
$p_a^{(\theta)}\propto p_a\exp(\langle\theta,v_a\rangle)$ multiplies every
word on a fixed endpoint fiber by the same factor, so noncentral budgets use
the profile of their mean-matching tilt.  Likewise, for a weighted size-$k$
slate, independent Bernoulli odds
$p_i/(1-p_i)=\exp(\beta z_i+\lambda)$ produce, after conditioning on
$\sum_iX_i=k$, mass proportional to
$\exp(\beta\sum_{i:X_i=1}z_i)$.  This proves the two corollaries used in the
main paper.

%% file: sections/growing_rank_geometry_appendix.tex
\section{Growing-rank conservation geometry}
\label{sec:growing-rank-geometry-app}

This section proves Theorem~\ref{thm:growing-rank-geometry}.  All logarithms
without a subscript are natural.  We first derive dimension-uniform one-group
bounds.  A rank-$d$ group has $q=d+1$ equiprobable categories: category $0$
has resource vector $0\in\mathbb Z^d$ and category $j$ has resource vector
$e_j$.  For block lengths divisible by $q$, we condition the group to use
every category exactly $m/q$ times.  We then tensor groups of possibly
different ranks $\mathbf d=(d_1,\ldots,d_J)$, whose total difference-lattice
rank is $r=\sum_jd_j$.

The proof has two dimension-uniform parts.  The converse retains the logarithm
of the killed spectral radius rather than linearizing it.  The construction
uses connected multinomial probability superlevel sets whose cardinalities
and information costs are simultaneously controlled.  Optimizing their widths
under one total state budget produces the entropy of the rank partition.

\subsection{An exact log-spectral converse}

Let $K$ be the symmetric action-difference walk and put
\[
 \rho_*(S)=\sup_{|D|\leq S}\rho(P_DKP_D).
\]

\begin{lemma}[Operational log-spectral bound]
\label{lem:gr-log-spectral}
Whenever the conditioned endpoint has probability $2^{-o(m)}$,
\[
 \mathcal R^{\rm fix}(S)
 \geq \frac{-\log\rho_*(S)}{2\ln2}.
\]
\end{lemma}

\begin{proof}
The moving-layer argument in Lemma~\ref{lem:app-general-moving} shows, before
the relaxation $-\log u\geq1-u$, that one width-$S$ plan captures at most
$P(E_m)^{-1}\rho_*(S)^{m/2}$ target mass.  A mixture at total variation at most
$\epsilon$ must cover target mass at least $1-\epsilon$, and therefore needs at
least
\[
 (1-\epsilon)P(E_m)\rho_*(S)^{-m/2}
\]
plans.  Take $m^{-1}\log_2$, then let $m\to\infty$ along admissible block
lengths and $\epsilon\downarrow0$.
\end{proof}

For one group, let
$N_t\sim\operatorname{Mult}(t;1/q,\ldots,1/q)$.  Its collision probability is
the return probability of the category-difference walk.

\begin{lemma}[Uniform multinomial mode]
\label{lem:gr-mode}
If $t/q\geq2$, then
\[
 \max_z\Pr\{N_t=z\}
 \leq e\sqrt{2q}\left(\frac{2q}{t}\right)^{d/2}.
\]
\end{lemma}

\begin{proof}
A multinomial mode has cell counts differing by at most one.  If its positive
counts are $n_1,\ldots,n_q$, the assumption implies
$n_j\geq t/(2q)$.  The elementary Stirling bounds
\[
 t!\leq e\sqrt t\,(t/e)^t,
 \qquad
 n_j!\geq\sqrt{2\pi n_j}\,(n_j/e)^{n_j}
\]
give
\[
 \Pr\{N_t=z\}
 \leq e\sqrt t\prod_{j=1}^q n_j^{-1/2}
       \exp\{-tD(z/t\|u_q)\}.
\]
Use $D\geq0$ and $n_j\geq t/(2q)$.  The factor $(2\pi)^{-q/2}$ only improves
the displayed bound.
\end{proof}

\begin{lemma}[Grouped killed-profile lower bound]
\label{lem:gr-profile-lower}
There is a universal $c>0$ such that every $g,d\geq1$ and all sufficiently
large $W$ satisfy
\[
 -\log\sup_{|D|\leq W^{gd}}\rho(K_D)\geq c\,gW^{-2}.
\]
\end{lemma}

\begin{proof}
The $g$-group kernel is a tensor product.  Its $t$-step return probability is
the $g$th power of the collision probability of $N_t$, and
$\sum_z\Pr(N_t=z)^2\leq\max_z\Pr(N_t=z)$.  Moreover, $K_D$ is positive
semidefinite because the unrestricted difference kernel is $T^*T$.  Hence
\[
 \rho(K_D)^t\leq\operatorname{tr}(K_D^t)
 \leq |D|K^t(0,0).
\]
Take $t=\lceil256qW^2\rceil$.  Lemma~\ref{lem:gr-mode} and
$|D|\leq W^{gd}$ yield
\[
 \rho(K_D)^t
 \leq\left[e\sqrt{2q}\,128^{-d/2}\right]^g
 \leq e^{-c_1gd}
\]
for a universal $c_1>0$; the smallest exponent occurs at $d=1$.  Since
$t\leq257qW^2$ and $d/q\geq1/2$, division by $t$ proves the claim.
\end{proof}

Lemmas~\ref{lem:gr-log-spectral} and \ref{lem:gr-profile-lower} give the lower
bound for equal groups.  Keeping $-\log\rho$ is essential: replacing it by
$1-\rho$ before taking a growing-rank product can lose the factor $g$.  The
same trace argument gives the nonuniform form needed below.

\begin{lemma}[Rank-partition killed-profile bound]
\label{lem:gr-partition-profile}
Let $\mathbf d=(d_1,\ldots,d_J)$, $r=\sum_jd_j$, and
\[
 \mathsf N(\mathbf d)
 =\frac{r}{(\prod_jd_j^{d_j})^{1/r}}.
\]
There are universal $c,C_0>0$ such that, whenever
\[
 L_{\mathbf d}(S)
 =\left(\prod_jd_j^{d_j}\right)^{1/r}S^{2/r}
 \geq C_0d_{\max},
\]
the product-simplex difference kernel satisfies
\[
 -\log\sup_{|D|\le S}\rho(K_D)
 \geq c\,\mathsf N(\mathbf d)S^{-2/r}.
\]
\end{lemma}

\begin{proof}
The product kernel is positive semidefinite.  For every $|D|\le S$,
\[
 \rho(K_D)^t\leq\operatorname{tr}(K_D^t)
 \leq S\prod_{j=1}^J K_j^t(0,0).
\]
The return probability in group $j$ is a multinomial collision probability,
so Lemma~\ref{lem:gr-mode} gives, when $t/(d_j+1)\geq2$,
\[
 K_j^t(0,0)
 \leq e\sqrt{2(d_j+1)}
 \left(\frac{2(d_j+1)}t\right)^{d_j/2}.
\]
Now $J\leq r$, $\log(2(d_j+1))\leq(\log4)d_j$, and
$d_j\leq d_j+1\leq2d_j$.  Taking
$t=\lceil A L_{\mathbf d}(S)\rceil$ for a sufficiently large universal $A$
makes the trace upper bound at most $e^{-c_1r}$.  The assumption on
$L_{\mathbf d}(S)$ ensures the mode-bound condition and also
$t\leq2A L_{\mathbf d}(S)$.  Dividing by $t$ gives
\[
 -\log\rho(K_D)
 \geq c_2\frac r{L_{\mathbf d}(S)}
 =c_2\mathsf N(\mathbf d)S^{-2/r}.
\]
\end{proof}

\subsection{A dimension-uniform compiler}

We first record a point-probability bound.  It is deliberately elementary so
that the constants in the state-set construction are explicit.

\begin{lemma}[Multinomial point probability]
\label{lem:gr-point}
For $N_n\sim\operatorname{Mult}(n;u_q)$ and every histogram $z$,
\[
 \Pr\{N_n=z\}
 \geq \exp\{-nD(z/n\|u_q)-q\}(1+n/q)^{-q/2}.
\]
\end{lemma}

\begin{proof}
The identity
\[
 \Pr\{N_n=z\}e^{nD(z/n\|u_q)}
 =\frac{n!}{n^n}\prod_{j:z_j>0}\frac{z_j^{z_j}}{z_j!}
\]
and the bounds $n!\geq(n/e)^n$ and
$k!\leq e\,k^{k+1/2}e^{-k}$ imply that the right-hand side is at least
\[
 e^{-q}\prod_{j=1}^q(1+z_j)^{-1/2}.
\]
The conclusion follows from
$\prod_j(1+z_j)\leq(1+n/q)^q$ by AM-GM.
\end{proof}

\begin{lemma}[Probability-threshold state set]
\label{lem:gr-threshold}
There are universal $d_0,W_0,c_0,c_1>0$ and $\alpha\in(0,1)$ such that the
following holds.  If $d\geq d_0$, $W_0\leq W\leq e^{c_0d}$,
$n=\lfloor\alpha dW^2\rfloor$, and
\[
 D_W=\{z:\Pr(N_n=z)\geq W^{-d}\},
\]
then
\[
 |D_W|\leq W^d,
 \qquad
 \Pr\{N_n\notin D_W\}\leq e^{-c_1d}.
\]
The set $D_W$ is connected under category transfers $e_i-e_j$.
\end{lemma}

\begin{proof}
The cardinality bound follows by summing the probabilities over $D_W$.  For
connectivity, if two cells differ by at least two, moving one count from the
larger cell to the smaller one multiplies the multinomial mass by
$z_i/(z_j+1)\geq1$.  Iteration reaches a mode without leaving the superlevel
set, and the modes are connected by equal-mass transfers.

Agrawal's finite-sample KL inequality \citep{agrawal2020multinomial} implies,
for every $A>1$,
\[
 \Pr\{nD(N_n/n\|u_q)\geq Ad\}
 \leq\exp\{-d(A-1-\log A)\}.
\]
Fix $A>1$, take
$0<c_0<(A-1-\log A)/4$, and then choose $\alpha$ small enough that
\[
 A+2+c_0+\tfrac12\log(2\alpha)<0.
\]
Increase $W_0$ until $\alpha W^2\geq1$.  On the complementary KL event,
Lemma~\ref{lem:gr-point}, $q=d+1\leq2d$, and
$\log W\leq c_0d$ give, with
$\Delta=A+2+c_0+\tfrac12\log(2\alpha)<0$,
\begin{align*}
 -\log\Pr\{N_n=z\}
 &\leq Ad+q+\frac q2\log(1+\alpha W^2)\\
 &\leq d\log W+d\Delta\\
 &\leq d\log W.
\end{align*}
Thus the KL event contains $D_W^c$, proving the tail bound with
$c_1=A-1-\log A$.
\end{proof}

The next identity exploits the exact cancellation compiler from
Section~\ref{sec:general-conservation-rank-app}.  For a positive histogram
weight $h$, draw $Z\sim h/\sum h$ and let $U$ be an independent uniform
category.  Given source category $X$ and state $Z$, the compiler selects $B$
with probability proportional to $h(Z+e_X-e_B)$.

\begin{lemma}[Selector-information identity]
\label{lem:gr-selector-info}
The stationary one-step information cost of the uniform-category compiler is
\[
 I(X,Z;B)=I(U;Z+e_U).
\]
\end{lemma}

\begin{proof}
Put $Y=Z+e_X$.  Conditional on $(X,Z)$, the compiler draws $B$ according to
\[
 \Pr\{B=b\mid X,Z\}
 =\frac{h(Y-e_b)}{\sum_ch(Y-e_c)}.
\]
Bayes' rule gives the same vector for
$\Pr\{U=b\mid Z+e_U=Y\}$.  Averaging its entropy over $Y$, and using that both
$B$ and $U$ are uniform, proves the identity.
\end{proof}

\begin{lemma}[Dimension-uniform information cost]
\label{lem:gr-information}
Let $Z\sim N_n$ conditional on $N_n\in D_W$ under the conditions of
Lemma~\ref{lem:gr-threshold}.  There is a universal $C_1$ such that
\[
 I(U;Z+e_U)\leq C_1W^{-2}
\]
in nats.
\end{lemma}

\begin{proof}
Without truncation, exchangeability gives
\[
 \Pr\{U=j\mid N_n+e_U=y\}=\frac{y_j}{n+1}.
\]
Consequently,
\begin{align*}
 I(U;N_n+e_U)
 &=\E D\!\left(\frac{N_{n+1}}{n+1}\middle\|u_q\right)\\
 &\leq\E\chi^2\!\left(\frac{N_{n+1}}{n+1},u_q\right)
 =\frac d{n+1}\leq\frac1{\alpha W^2}.
\end{align*}
Let $\delta=\Pr\{N_n\notin D_W\}$.  Conditioning changes the joint law of
$(U,N_n+e_U)$ by total variation at most $\delta$.  Finite-input
conditional-entropy continuity therefore gives
\[
 \left|I(U;Z+e_U)-I(U;N_n+e_U)\right|
 \leq2\delta\log q+2h_2(\delta).
\]
Here $\delta\leq e^{-c_1d}$ and $W\leq e^{c_0d}$, where the constants can be
chosen with $2c_0<c_1$.  The final display is therefore $O(W^{-2})$ uniformly
in the stated regime.
\end{proof}

\begin{proof}[Proof of Theorem~\ref{thm:growing-rank-geometry}]
Lemma~\ref{lem:gr-partition-profile} and the operational log-spectral bound
give the lower inequality.

For the upper inequality, put
\[
 W_j=\left\lfloor
 \sqrt{L_{\mathbf d}(S)/d_j}\right\rfloor.
\]
After increasing the universal threshold $W_0$, the assumptions in the main
theorem give $W_j\geq W_0$ and
$W_j\geq\frac12\sqrt{L_{\mathbf d}(S)/d_j}$.  Moreover,
\[
 \prod_jW_j^{d_j}
 \leq\prod_j\left(\frac{L_{\mathbf d}(S)}{d_j}\right)^{d_j/2}=S.
\]
For $d_j\geq d_0$, use the discrepancy weight
$h_j(z)=\Pr\{N_{n_j}=z\}\mathbf1\{z\in D_{W_j}\}$.  Lemma
\ref{lem:gr-threshold} gives a connected state set of size at most
$W_j^{d_j}$, while Lemmas~\ref{lem:gr-selector-info} and
\ref{lem:gr-information} bound its one-step selector information by
$C_1W_j^{-2}$ nats.  For $d_j<d_0$, the fixed-dimensional profile theorem and
the compiler in Section~\ref{sec:general-conservation-rank-app} give the same
bound after maximizing the constants and width thresholds over this finite
set of dimensions.

Weighted cancellation, posterior realization, endpoint conditioning, and
one-shot resolvability proceed groupwise.  Taking the Cartesian product of the
group codebooks multiplies state counts and adds information costs, while
total variation between product laws is at most the sum of the group errors.
Therefore
\begin{align*}
 \mathcal R^{\rm fix}_{\mathbf d}(S)
 &\leq C_2\sum_jW_j^{-2} \\
 &\leq4C_2\sum_j\frac{d_j}{L_{\mathbf d}(S)}
 =4C_2\mathsf N(\mathbf d)S^{-2/r}.
\end{align*}
\end{proof}

\subsection{Partition entropy, covariance volume, and memory}

Let $\Sigma$ be the covariance of one resource increment in an intrinsic
lattice basis and define
\[
 \mathsf G(p,v)=r\det(\Sigma)^{1/r}.
\]
Translation leaves $\Sigma$ unchanged, and a unimodular basis change has
determinant one in absolute value, so $\mathsf G$ is lattice invariant.  For
one uniform $d+1$ category group,
\begin{align*}
 \Sigma_d&=\frac1{d+1}I_d
 -\frac1{(d+1)^2}\mathbf1\mathbf1^{\mathsf T},\\
 \det(\Sigma_d)&=(d+1)^{-(d+1)}.
\end{align*}
For the rank partition $\mathbf d=(d_1,\ldots,d_J)$, block diagonality gives
\[
 \mathsf G_{\mathbf d}
 =r\prod_j(d_j+1)^{-(d_j+1)/r}.
\]
Since
\[
 \frac{\mathsf G_{\mathbf d}}{\mathsf N(\mathbf d)}
 =\prod_j\left[d_j(d_j+1)^{-(d_j+1)/d_j}\right]^{d_j/r}
\]
and every bracket lies in $[1/4,1]$, we obtain
$\mathsf N(\mathbf d)/4\leq\mathsf G_{\mathbf d}\leq\mathsf N(\mathbf d)$.
This proves the covariance statement in the main theorem.  We do not claim
that covariance volume universally determines the rate outside products of
simplex groups.

\begin{lemma}[Resource-neutral label refinement]
\label{lem:gr-label-refinement}
Replace each action $a$ by finitely many labels $(a,j)$, give every copy the
same resource vector $v_a$, and assign labels according to a conditional law
$r(j\mid a)$.  For every state width $S$, this refinement leaves
$\mathcal R^{\rm fix}(S)$ unchanged.
\end{lemma}

\begin{proof}
From a base executor, draw the refinement label privately after emitting the
base action.  This adds neither live state nor a plan index and realizes the
refined conditioned target.  Conversely, project each refined label $(a,j)$
to $a$.  Projection preserves the state width and exact endpoint support, and
total variation contracts under projection.  The two transformations give the
two rate inequalities.
\end{proof}

\begin{corollary}[Memory is priced by rank-partition entropy]
\label{cor:gr-memory-separation}
Fix any sufficiently small constant selector-rate target $\eta>0$.  For rank
partitions whose optimizing widths lie in the regime of
Theorem~\ref{thm:growing-rank-geometry}, the minimum persistent-memory budget
satisfies
\[
 B_\eta(\mathbf d)
 =\frac{r}{2\ln2}H(d_1/r,\ldots,d_J/r)+\Theta_\eta(r).
\]
Thus $k$ equal rank blocks cost
$\frac r2\log_2k+\Theta_\eta(r)$ bits.  The endpoints are
$\frac r2\log_2r+\Theta_\eta(r)$ for $r$ independent binary budgets and
$\Theta_\eta(r)$ for one exclusive $(r+1)$-category quota.  The endpoint
sources can be refined to have the same conservation rank, alphabet size, and
one-step entropy without changing either threshold.
\end{corollary}

\begin{proof}
Solving Eq.~\eqref{eq:growing-rank-law} for $B=\log_2S$ at fixed $\eta$ gives
the first display; universal multiplicative constants change $B$ by only
$O(r)$.  For $k$ equal blocks the normalized rank partition is uniform, so
its entropy is $\log k$.  The two endpoints correspond to $k=r$ and $k=1$.

For the refinement claim, set
\[
 Q_r=\operatorname{lcm}(2^r,r+1).
\]
Duplicate every binary-vector action $Q_r/2^r$ times and every simplex action
$Q_r/(r+1)$ times, with uniform private labels.  Both refined one-step sources
are uniform on $Q_r$ symbols and hence have entropy $\log_2Q_r$; Lemma
\ref{lem:gr-label-refinement} preserves their operational rates.
\end{proof}

\subsection{Local statistics do not reveal the memory class}
\label{sec:gr-local-twins}

\begin{proof}[Proof of Theorem~\ref{thm:local-conservation-twins}]
Use the common label
\[
 (u,c)\in\{0,1\}^r\times[r+1].
\]
For the fragmented target, the \(r\) columns of \(u\) are independent
uniform balanced binary strings and the \(c\)-labels are iid uniform.  For
the monolithic target, \(c\) is a uniform string with exactly \(m/(r+1)\)
copies of each category and the \(u\)-labels are iid uniform.  These are
resource-neutral refinements of rank partitions \((1,\ldots,1)\) and
\((r)\), so Lemma~\ref{lem:gr-label-refinement} preserves their operational
rates and both one-coordinate laws are uniform on the common alphabet.

For an urn of \(m\) balls with \(a\) colors, the relative entropy between
the ordered \(k\)-draw laws without and with replacement is bounded by
\[
 D(H_{m,k}\Vert M_{m,k})
 \leq\frac{(a-1)k(k-1)}
 {2(m-1)(m-k+1)}
 \label{eq:gr-stam}
\]
\citep[Eq.~(6)]{johnson2025sampling}, restating the sampling bound of
\citet{stam1978distance}.  Relative entropy adds across the \(r\)
independent balanced binary columns.  Thus the fragmented-to-iid divergence
and the monolithic-to-iid divergence are each at most
\[
 \frac{r\,k(k-1)}{2(m-1)(m-k+1)}.
\]
The conditional uniform copy labels preserve these divergences.  Pinsker's
inequality for each target followed by the triangle inequality proves the
displayed upper bound in the theorem.

It remains to show that its scale is sharp.  Project the common label onto
\(u=(u_1,\ldots,u_r)\).  For each bit column \(j\), let \(S_j\) be the
number of ones among the inspected coordinates and set
\[
 T=\sum_{j=1}^r(S_j-k/2)^2.
\]
Under the monolithic target the \(S_j\) are independent
\(\operatorname{Bin}(k,1/2)\) variables.  Under the fragmented target they
are independent \(\operatorname{Hypergeom}(m,m/2,k)\) variables.  Hence
\begin{align*}
 \mathbb E_{\rm mono}T&=\frac{rk}{4},&
 \mathbb E_{\rm frag}T&=\frac{rk(m-k)}{4(m-1)},\\
 \mathbb E_{\rm mono}T-\mathbb E_{\rm frag}T
 &=\frac{rk(k-1)}{4(m-1)}.
\end{align*}
For a binomial column,
\(\operatorname{Var}((S_j-k/2)^2)=k(k-1)/8\).
The sampling-without-replacement Hoeffding bound gives
\(\Pr\{|S_j-k/2|\ge t\}\le2e^{-2t^2/k}\); integrating this tail yields
\(\mathbb E(S_j-k/2)^4\le k^2\).  Therefore both variances of \(T\) are
\(O(rk^2)\).  A test that thresholds \(T\) halfway between the two means has
total error
\[
 O\!\left(\frac{m^2}{rk^2}\right).
\]
It vanishes when \(k\sqrt r/m\to\infty\), so data processing implies that
the full local-marginal TV tends to one.  The upper bound tends to zero in
the opposite regime.

Finally, Corollary~\ref{cor:gr-triangular-separation} supplies, along one
admissible blocklength-rank sequence, the two memory thresholds stated in the
main paper.  For example \(k=\lfloor m/r\rfloor\) still lies in the
indistinguishable regime, since \(k\sqrt r/m=O(r^{-1/2})\), while the memory
ratio diverges as \(\Theta(\log r)\).
\end{proof}

\paragraph{Bounded local transcripts.}
If an observer receives \(n\) independent fresh outputs and adaptively chooses
at most \(k\) coordinates of each, sequential maximal coupling bounds the TV
between its two transcript laws by \(n\) times the one-query bound in
Theorem~\ref{thm:local-conservation-twins}.  Hence local evaluation cannot
certify the execution-memory class whenever that product vanishes.  This is
an operational consequence of the conservation law, not a new finite
exchangeability theorem.

\subsection{A simultaneous blocklength-rank window}

The operational rate above takes block length to infinity for each fixed
conservation geometry.  We now give a conservative window in which the same
law holds when geometry and block length vary together.  The condition is not
optimized; its role is to rule out an order-of-limits artifact.

For one rank-$d$ threshold compiler, put $q=d+1$,
\begin{align*}
 n&=\lfloor\alpha dW^2\rfloor,
 &N&=2n,\\
 a&=q^{-2}W^{-d},
 &\lambda&=a^N,\\
 L&=d\log_2W+\log_2q,\\
 \Psi(d,W)&=C N^2W^4(L+1)^2\lambda^{-2},
\end{align*}
where $\alpha$ is from Lemma~\ref{lem:gr-threshold} and $C$ is a sufficiently
large universal constant.

\begin{lemma}[Uniform information-density concentration]
\label{lem:gr-triangular-information}
Let $\imath_m$ be the one-group compiler information density before endpoint
conditioning, and let its one-step mean satisfy
$c_{d,W}\leq C_0W^{-2}$.  Whenever $m\geq\Psi(d,W)$,
\begin{align*}
 &\Pr\!\left\{\imath_m>
 m(c_{d,W}+C_0W^{-2})\right\}\\
 &\qquad\leq4e^{-m/\Psi(d,W)}+W^d2^{-c m/W^2}.
\end{align*}
After conditioning on the balanced endpoint, the right side is multiplied by
at most $(m+1)^q$, with an additional term of the form
$(m+1)^q2^{-c m/W^2}$.
\end{lemma}

\begin{proof}
Every positive histogram weight in $D_W$ lies in $[W^{-d},1]$.  Consequently,
each allowed category transfer, and a self-loop used for padding, has marginal
discrepancy-transition probability at least $a=q^{-2}W^{-d}$.  Every retained
histogram reaches a mode in at most $n$ balancing transfers without leaving
$D_W$.  Reversing the corresponding path for a second histogram connects any
pair in at most $N=2n$ transfers.  Thus, for the uniform law $\nu$ on $D_W$,
\[
 P^N(z,\cdot)\geq\lambda\nu(\cdot),\qquad\lambda=a^N.
\]

Write
$\ell_t=\log_2[Q(B_t\mid X_t,D_{t-1})/p_{B_t}]$.  On every sampled edge,
$|\ell_t|\leq L$, while
$g(z)=\mathbb E[\ell_t\mid D_{t-1}=z]$ lies in $[0,\log_2q]$.
The uniformly recurrent Markov-chain Hoeffding inequality
\citep{glynnOrmoneit2002} applied to $\sum_tg(D_{t-1})$, together with
Azuma's inequality for $\sum_t(\ell_t-g(D_{t-1}))$, gives
\begin{align*}
 &\Pr\!\left\{\sum_t\ell_t-mc_{d,W}>c mW^{-2}\right\}\\
 &\qquad\leq4\exp\!\left\{-
 \frac{c m\lambda^2}{N^2W^4(L+1)^2}\right\}.
\end{align*}
The remaining information-density boundary term is at most
$-\log_2P(A\mid B^m)$.  Since $|D_W|\leq W^d$, its upper tail at level
$c m/W^2$ is at most $W^d2^{-c m/W^2}$.

The balanced multinomial histogram is a mode among at most $(m+1)^q$ types,
so its endpoint probability is at least $(m+1)^{-q}$.  Conditioning therefore
multiplies the preceding tail by at most $(m+1)^q$.  Finally, for
$G_m=P(E_m\mid A,B^m)$,
\[
 P_{A,B^m\mid E_m}\{G_m<2^{-u}\}
 \leq2^{-u}/P(E_m),
\]
which gives the additional term.
\end{proof}

\begin{theorem}[Triangular-array grouped law]
\label{thm:gr-triangular}
Let $(g_m,d_m,W_m)$ vary along block lengths divisible by $d_m+1$, with
$d_m\geq d_0$ and $W_0\leq W_m\leq e^{c_0d_m}$.  Assume the remaining
dimension-uniform compiler conditions in Theorem~\ref{thm:growing-rank-geometry}
and
\begin{equation}
 \begin{aligned}
  &\log g_m+(d_m+1)\log(m+1)+d_m\log W_m\\
  &\qquad=o\!\left(\frac{m}{\Psi(d_m,W_m)}\right).
 \end{aligned}                                           \label{eq:gr-triangular-condition}
\end{equation}
Then width-$W_m^{g_md_m}$ exact-support plan libraries attain vanishing total
variation error with selector rate at most $C g_mW_m^{-2}$, and every such
vanishing-error library has rate at least $c g_mW_m^{-2}$.  In particular,
the simpler conditions
\begin{equation}
 d_m^2W_m^2\log(d_mW_m)=o(\log m),
 \qquad\log g_m=o(m^{1/2})                              \label{eq:gr-triangular-simple}
\end{equation}
are sufficient.
\end{theorem}

\begin{proof}
For one group, apply the one-shot truncation inequality
\[
 \mathbb E\|\widehat P_m-P_m\|_{\rm TV}
 \leq P\{\imath_m>u\}+\tfrac12\sqrt{2^u/M}
\]
with Lemma~\ref{lem:gr-triangular-information} and
$u=m(c_{d_m,W_m}+C_0W_m^{-2})$.  The condition makes the first term
$o(1/g_m)$.  Choose
\begin{align*}
 \log_2M&=u+2\log_2g_m+\omega_m,\\
 \omega_m&\longrightarrow\infty,
 &\omega_m&=o(m/W_m^2).
\end{align*}
Then the square-root term is also $o(1/g_m)$ and
$\log_2M=O(m/W_m^2)$.  Taking the Cartesian product of the $g_m$ one-group
codebooks multiplies state and plan counts, while product-law TV is at most
the sum of the group errors.  The extra $g_m\log g_m/m$ rate is
$o(g_mW_m^{-2})$ under the same condition.

For the converse, the finite-block moving-layer inequality and
Lemma~\ref{lem:gr-profile-lower} give
\begin{align*}
 \frac1m\log_2K_m
 &\geq c g_mW_m^{-2}
 -\frac{g_m(d_m+1)\log_2(m+1)}m\\
 &\qquad+\frac1m\log_2(1-\varepsilon_m),
\end{align*}
where $\varepsilon_m\to0$ is the library's total-variation error.
The correction is $o(g_mW_m^{-2})$ under
Eq.~\eqref{eq:gr-triangular-condition}.  Finally,
$\log\Psi(d,W)=O(d^2W^2\log(dW))$.  Under
Eq.~\eqref{eq:gr-triangular-simple}, $\Psi(d_m,W_m)=m^{o(1)}$, so the right
side of Eq.~\eqref{eq:gr-triangular-condition} is $m^{1-o(1)}$ and dominates
its left side.
\end{proof}

\begin{corollary}[The separation is not an iterated-limit artifact]
\label{cor:gr-triangular-separation}
Let $r_m\to\infty$ satisfy
$r_m^2\log(r_m+1)=o(\log m)$ and fix a sufficiently small selector-rate
target $\eta>0$.  Along this same blocklength sequence, $r_m$ independent
balanced binary budgets require
$\frac{r_m}{2}\log_2r_m+\Theta_\eta(r_m)$ memory bits, whereas one balanced
$(r_m+1)$-category quota requires only $\Theta_\eta(r_m)$ bits.  Both use
vanishing-error libraries whose every component has exact legal support.
\end{corollary}

\begin{proof}
For the exclusive quota, use Theorem~\ref{thm:gr-triangular} with $g_m=1$,
$d_m=r_m$, and constant $W_m=\Theta_\eta(1)$.  For the binary family, apply
the same minorization proof to the sine-square corridor groupwise.  Its
smallest positive profile weight is $\Theta(W^{-2})$ and its path diameter is
$O(W)$, so the corresponding concentration cost satisfies
$\log\Psi_2(W)=O(W\log W)$.  Take $g_m=r_m$ and
$W_m=\Theta_\eta(\sqrt{r_m})$.  The finite-block log-spectral converse gives
the matching lower bounds.  The displayed growth condition is stronger than
both concentration requirements.
\end{proof}

\paragraph{Scope and finite checks.}
The partition-entropy theorem is dimension-uniform only in its stated
interior compiler regime; no universal covariance descriptor is claimed
outside products of simplex groups.  The triangular condition is deliberately
conservative rather than a sharp finite-block threshold.  Spectral-profile
bounds, multinomial concentration, Markov concentration, and resolvability are
borrowed tools.  The new claim is their matching exact-support composition,
the entropy law across rank decompositions, and its finite-block persistence.
Independent finite enumeration verifies positive semidefiniteness, the trace
inequality, threshold-set cardinality and connectivity, the
selector-information identity, integer allocation, and mixed-group killed
kernels; these checks verify the algebra but are not used as proofs.

%% file: sections/soft_frontier_appendix.tex

\section{Zero-distortion support law for programmable order}
\label{app:soft-frontier}

Let
\begin{align*}
 \Omega_m&=\{x\in\{0,1\}^m:\textstyle\sum_i x_i=m/2\},\\
 \lambda_S&=2\cos\frac{\pi}{2S+1}.
\end{align*}
and put
\[
 \iota_S=1-\log_2\lambda_S
          =-\log_2\cos\frac{\pi}{2S+1}.
\]

\begin{lemma}[Sharp zero-distortion support fraction]
\label{lem:centered-corridor-support}
For odd $S$, every fixed-order width-$S$ exact-support binary plan language
$\mathcal L_m\subseteq\Omega_m$ satisfies
\[
 |\mathcal L_m|\le \lambda_S^m.
\]
There is a centered width-$S$ count corridor $\mathcal C_{m,S}$ such that,
along even $m$,
\[
 \frac{|\mathcal C_{m,S}|}{|\Omega_m|}
 =2^{-m\iota_S+o(m)}.
\]
Its uniform law is generated by a fixed-order width-$S$ suffix-count dynamic
program.
\end{lemma}

\begin{proof}
The upper bound is Lemma~\ref{lem:moving-layer-support}: distinct reachable
prefix sums cannot merge under exact support, and the sharp moving-layer path
bound is $\lambda_S^m$.

For attainment, keep at every interior layer the $S$ feasible prefix counts
nearest to half the elapsed time, truncating only in the first and last $O(S)$
layers.  Away from those endpoints the two-step transfer alternates
$U=\tfrac12(I+N)$ and $L=JUJ$, where $N$ is the upper shift and $J$ reverses
coordinates.  The standard tridiagonal determinant recursion gives
\[
 \rho(UL)^{1/2}=\cos\frac{\pi}{2S+1}.
\]
Perron--Frobenius theory therefore makes the fair-iid probability of the
corridor $\cos(\pi/(2S+1))^{m+o(m)}$; the $O(S)$ endpoint layers change only
the subexponential factor.  Every fair binary word has mass $2^{-m}$, while
$|\Omega_m|=2^{m-o(m)}$, which gives the displayed support fraction.
Finally, suffix path counts within the corridor determine the exact next-bit
probabilities from the current live count, yielding at most $S$ states per
layer and the uniform law on $\mathcal C_{m,S}$.
\end{proof}

This lemma is the only neighborhood fact needed for the
programmable-order result.  No positive-distortion covering theorem is claimed
in the main paper.

%% file: sections/orbit_symmetrization_appendix.tex
\section{Order-programmable achievability by orbit symmetrization}
\label{app:orbit-symmetrization}

An \emph{order-programmable} plan fixes a permutation before private coins are
drawn and emits addressable coordinates in that order.  This changes the
available exact-support languages without increasing live-state width.

\begin{corollary}[Sharp order-programmable generation rate]
\label{cor:order-programmable-frontier}
For odd $S$, the minimum selector rate for TV-accurate generation of the
uniform balanced binary law is
\begin{align*}
 \mathcal R^{\rm ord}_2(S)
 &=\iota_S=-\log_2\cos\frac{\pi}{2S+1},\\
 &=\frac{\pi^2}{8\ln2}S^{-2}(1+o(1)).
\end{align*}
Thus programmable address order closes the factor-four binary constant gap
for this stronger interface.
\end{corollary}

\begin{lemma}[Finite transitive-orbit codebook]
\label{lem:orbit-codebook}
Let a finite group $G$ act transitively on a finite set $\Omega$.  For nonempty
$L\subseteq\Omega$, put $N=|\Omega|$, $M=|L|$, and
$P_g=\operatorname{Unif}(gL)$.  Then
\[
 \frac1{|G|}\sum_{g\in G}P_g=\operatorname{Unif}(\Omega).
\]
If $G_1,\ldots,G_K$ are iid uniform on $G$ and
$Q_K=K^{-1}\sum_{j=1}^K P_{G_j}$, then
\[
 \mathbb E\,d_{\rm TV}\!\left(Q_K,\operatorname{Unif}(\Omega)\right)
 \le \frac12\sqrt{\frac{N/M-1}{K}}.
\]
\end{lemma}

\begin{proof}
Orbit--stabilizer gives
$|\{g:x\in gL\}|=|G|M/N$ for every $x$, proving the exact orbit average.
For uniform $G$, $P_G(x)$ equals $M^{-1}$ with probability $M/N$ and zero
otherwise.  Independence therefore yields
\[
 \mathbb E\,\chi^2\!\left(Q_K\middle\|\operatorname{Unif}(\Omega)\right)
 =\frac{N/M-1}{K}.
\]
Apply $d_{\rm TV}(P,Q)\le\tfrac12\sqrt{\chi^2(P\|Q)}$ and Jensen's
inequality.
\end{proof}

\begin{proof}[Proof of Corollary~\ref{cor:order-programmable-frontier}]
Use the centered corridor $L_{m,S}=\mathcal C_{m,S}$ from
Lemma~\ref{lem:centered-corridor-support}.  Its uniform law is a width-$S$
plan.  A hard-coded coordinate permutation $\pi$ produces
$\operatorname{Unif}(\pi L_{m,S})$ with the same live width.  The symmetric
group acts transitively on $\Omega_m$, and
\[
 \frac{|L_{m,S}|}{|\Omega_m|}=2^{-m\iota_S+o(m)}.
\]
For $K_m=2^{m(\iota_S+\eta_m)}$, with $\eta_m\downarrow0$ slowly enough,
Lemma~\ref{lem:orbit-codebook} gives a deterministic sublibrary whose TV
error vanishes.  Conversely, Lemma~\ref{lem:centered-corridor-support} bounds
every width-$S$ plan support by $\lambda_S^m$.  If a $K_m$-plan mixture is
within TV $\epsilon$ of the flat law on $\Omega_m$, the union of its supports
contains at least $(1-\epsilon)|\Omega_m|$ words, so
\[
 K_m\lambda_S^m\ge(1-\epsilon)|\Omega_m|.
\]
Taking per-symbol logarithms proves the matching lower rate.  The Taylor
expansion of $-\log_2\cos(\pi/(2S+1))$ gives the final expression.
\end{proof}

\paragraph{Interface boundary.}
The permutation is selected before generation and requires addressable output
coordinates.  It is not free on an immutable stream, and its table-description
length is a separate provisioning currency.  Orbit symmetrization is standard
\citep{ahlswedeDueck1982}; the paper-specific input is the sharp fixed-width
support fraction in Lemma~\ref{lem:centered-corridor-support}.

%% file: sections/plan_memory_qary_appendix.tex

\section{Plan-memory exchange for balanced type classes}
\label{sec:qary-plan-memory-law}

Fix an alphabet $[q]$, put $d=q-1$, and assume $q\mid m$.  Let
\[
 \Omega_{m,q}=\left\{x\in[q]^m:
 N_a(x)=m/q\ \text{for every }a\in[q]\right\}.
\]
Let $U_{m,q}=\operatorname{Unif}(\Omega_{m,q})$.
A fixed-order width-$S$ plan is a time-inhomogeneous stochastic generator in
the sense of Definition~\ref{def:fixed-order-executor}
that emits $x_1,\ldots,x_m$ from left to right, has at most $S$ live states
per layer, and is supported on $\Omega_{m,q}$.  Its transition tables may
depend on time and on a front-loaded plan value; private randomness is free.
Let $K^{\rm fix}_{m,q,\epsilon}(S)$ be the fewest such component laws whose
mixture is within TV $\epsilon$ of $U_{m,q}$, and define
\[
 \mathcal R^{\rm fix}_q(S)
 =\lim_{\epsilon\downarrow0}\limsup_{\substack{m\to\infty\\q\mid m}}
   \frac1m\log_2K^{\rm fix}_{m,q,\epsilon}(S).
\]

\begin{theorem}[Balanced-type resource law]
\label{thm:qary-resource-law}
For every fixed $q\ge2$, constants $0<a_q\le b_q<\infty$ exist such that
\begin{equation}
 a_q S^{-2/(q-1)}
 \ \le\ \mathcal R^{\rm fix}_q(S)\ \le\
 b_q S^{-2/(q-1)}
 \label{eq:qary-resource-law}
\end{equation}
for all sufficiently large $S$.  Thus, for this balanced type-class family,
the $q-1$ independent count constraints set the exchange exponent between
front-loaded coordination and online memory.  The constants are effective for
each fixed $q$ but inherit nonoptimized graph-isoperimetric and channel bounds.
For $q=2$, the sharper sandwich in Theorem
\ref{thm:fixed-frontier-v2} applies.
\end{theorem}

The lower bound is spectral: an $S$-state generator cannot support enough
balanced paths.  The upper bound is distributional and fixed-order: a
Bayes-reversed finite-state channel produces a mixture of exact balanced
generators at the matching rate.  This second direction is the load-bearing
posterior-realization result.

\subsection{The root-lattice Bayes channel}

Represent discrepancies in $\mathbb Z^d$.  Let
$v_1=e_1,\ldots,v_d=e_d,v_q=0$.  For an integer $R\ge1$, define
\[
 \mathcal D_R=\{-R,\ldots,R\}^d,
 \qquad S_R=|\mathcal D_R|=(2R+1)^d.
\]
For $z\in\mathbb Z^d$, set
\[
 g_R(z)=\prod_{j=1}^d
 \cos\!\left(\frac{\pi z_j}{2(R+1)}\right)
 \mathbf 1\{z\in\mathcal D_R\}.
\]
Set $h_R(z)=g_R(z)^2$,
and $H_R=\sum_z h_R(z)=(R+1)^d$.  Initialize
$A=D_0$ with $P(A=z)=\pi_R(z):=h_R(z)/H_R$.  Given an iid uniform source
symbol $X_t=x$ and current discrepancy $D_{t-1}=z$, emit $B_t=b$ according to
\begin{align}
 Q_R(b\mid x,z)
 &=\frac{h_R(z+v_x-v_b)}
 {\sum_{c=1}^q h_R(z+v_x-v_c)},\label{eq:qary-test-channel}\\
 D_t&=z+v_x-v_b.\notag
\end{align}
At least the term $c=x$ in the denominator equals $h_R(z)>0$, so the channel
is well defined and never leaves $\mathcal D_R$.

\begin{remark}[Ground-state bias, not a hidden homogeneity assumption]
Equation~\eqref{eq:qary-test-channel} uses the positive ground-state weight
$h_R$, in the spirit of an $h$-transform.  It is not, in general, one
homogeneous Doob transform: its normalizer depends on both $(x,z)$.  The iid
auxiliary law below follows from the exact numerator-denominator cancellation,
not from an unstated stationarity assumption.
\end{remark}

\begin{lemma}[Exact uniform auxiliary process]
\label{lem:qary-uniform-aux}
For every $t$, $B_t$ is uniform on $[q]$ and independent of $B^{t-1}$, while
$D_t\mid B^t$ has law $\pi_R$.
\end{lemma}

\begin{proof}
Assume $D_{t-1}\mid B^{t-1}=b^{t-1}$ has law $\pi_R$.  Fix an output symbol
$b$ and a next state $y$.  The predecessor paired with source symbol $x$ is
$y-v_x+v_b$.  Therefore
\begin{align*}
 &P(B_t=b,D_t=y\mid B^{t-1})\\
 &\quad=\frac1q\sum_{x=1}^q
 \frac{h_R(y-v_x+v_b)}{H_R}
 \frac{h_R(y)}{\sum_{c=1}^q h_R(y+v_b-v_c)}\\
 &\quad=\frac{h_R(y)}{qH_R}=\frac1q\pi_R(y).
\end{align*}
The cancellation uses the fact that the numerator sum over $x$ is exactly
the denominator sum over $c$.  Induction starts from $D_0=A\sim\pi_R$.
\end{proof}

Let $E_{m,q}$ be the event that $X^m$ has exactly $m/q$ occurrences of each
symbol, and condition the joint law on $E_{m,q}$.  The resulting source is
$U_{m,q}$ and the front-loaded plan is $Z_m=(A,B^m)$.

\begin{lemma}[Exact fixed-order q-ary posterior realization]
\label{lem:qary-bayes-compiler}
The Bayes decomposition
\[
 U_{m,q}(x)=\sum_zP(z\mid E_{m,q})P(x\mid z,E_{m,q})
\]
is exact.  For every $z=(a,b^m)$ of positive probability,
$P(X^m\mid z,E_{m,q})$ is generated by a fixed-order width-$S_R$ stochastic
machine.
\end{lemma}

\begin{proof}
Only implementability needs proof.  For fixed $(a,b^m)$, the posterior is a
product of the local factors in \eqref{eq:qary-test-channel}, restricted to
the balanced endpoint.  A suffix partition function indexed by
$D_t\in\mathcal D_R$ gives exact left-to-right transitions.  Moreover,
writing $N_{1:d}(w)=(N_1(w),\ldots,N_d(w))$, we have
\[
 N_{1:d}(X_{1:t})=N_{1:d}(B_{1:t})+D_t-a.
\]
The last symbol count is $t$ minus the displayed counts.  Hence $D_t$ alone
determines every remaining budget once the plan is fixed, so the endpoint
constraint requires no state beyond the $S_R$ discrepancy values.
\end{proof}

\paragraph{Executable suffix recursion and cost.}
For a fixed positive-probability plan $z=(a,b^m)$, define the backward
partition value $F_t(u)$ for $u\in\mathcal D_R$ by
\begin{align*}
 F_m(u)&=\mathbf 1\!\left\{
 N_{1:d}(b^m)+u-a=\tfrac mq\mathbf 1_d\right\},\\
 F_t(u)&=\frac1q\sum_{x=1}^q
 Q_R(b_{t+1}\mid x,u)
 F_{t+1}(u+v_x-v_{b_{t+1}}).
\end{align*}
Values outside $\mathcal D_R$ are zero.  Whenever $F_t(u)>0$, the exact
posterior transition, with $u_x^+=u+v_x-v_{b_{t+1}}$, is
\[
\begin{aligned}
 &\Pr\{X_{t+1}=x\mid D_t=u,z,E_{m,q}\}\\
 &\qquad=\frac{q^{-1}Q_R(b_{t+1}\mid x,u)}{F_t(u)}
 F_{t+1}(u_x^+).
\end{aligned}
\]
The backward pass takes $O(mqS_R)$ arithmetic operations.  Materializing all
time-dependent transition tables takes $O(mqS_R)$ cells; online sampling then
uses $O(q)$ work per token and stores only $D_t$, i.e.,
$\lceil\log_2S_R\rceil$ persistent bits.  The theorem prices the latter live
state and the plan index, not this offline table description; the augmented
ledger in the main paper reports that distinction explicitly.

\subsection{The posterior construction has quadratically small plan rate}

Define the stationary per-symbol information rate
\begin{equation}
 c_{q,R}
 =\log_2q-
 \frac1q\sum_{x=1}^q\sum_{z\in\mathcal D_R}
 \pi_R(z)H_2\big(Q_R(\cdot\mid x,z)\big).
 \label{eq:qary-information-rate}
\end{equation}

\begin{lemma}[Dirichlet information bound]
\label{lem:qary-dirichlet-bound}
For every fixed $q$, $c_{q,R}\le C_qR^{-2}$.
\end{lemma}

\begin{proof}
For a fixed $(z,x)$, write
$a_b=g_R(z+v_x-v_b)$, $W=\sum_ba_b^2$, and
$p_b=a_b^2/W=Q_R(b\mid x,z)$.  If $u_q$ is uniform on $[q]$, then
\begin{equation}
\begin{aligned}
\log_2q-H_2(p)
&=D_2(p\|u_q)\\
&\le\frac{1}{\ln2}\chi^2(p\|u_q)
\le\frac{4}{\ln2}\frac{\sum_{b,c}(a_b-a_c)^2}{W}.
\end{aligned}
\label{eq:entropy-hellinger-bound}
\end{equation}
The first inequality is $\ln t\le t-1$.  For the second, use
$|p_b-q^{-1}|\le2|\sqrt{p_b}-q^{-1/2}|$ and
\[
 \sum_b(\sqrt{p_b}-q^{-1/2})^2
 \le\frac1{qW}\sum_{b,c}(a_b-a_c)^2.
\]
Because the choice $b=x$ gives $a_x^2=h_R(z)$, we have
$h_R(z)/W\le1$.  Averaging \eqref{eq:entropy-hellinger-bound} with weight
$\pi_R(z)/q$, extending $g_R$ by zero outside the box, and translating the
lattice sums gives
\begin{equation}
 c_{q,R}
 \le\frac{4}{H_R\ln2}
 \sum_{b,c=1}^q\sum_{y\in\mathbb Z^d}
 \big(g_R(y-v_b)-g_R(y-v_c)\big)^2.
 \label{eq:qary-dirichlet-energy}
\end{equation}
Every shift $v_b-v_c$ is either a coordinate step or a difference of two
coordinate steps.  By translation invariance and the triangle inequality in
$\ell_2$, the right side is at most a $q$-dependent constant times the sum of
the $d$ coordinate Dirichlet energies.  The product-cosine ground state has
the exact one-coordinate ratio
\begin{align*}
 \frac{\sum_y(g_R(y+e_j)-g_R(y))^2}{\sum_yg_R(y)^2}
 &=2\left(1-\cos\frac{\pi}{2(R+1)}\right)\\
 &\le\frac{\pi^2}{4(R+1)^2}.
\end{align*}
Since $\sum_yg_R(y)^2=H_R$, Equation
\eqref{eq:qary-dirichlet-energy} is $O_q(R^{-2})$.
\end{proof}

\begin{lemma}[Primitive discrepancy chain]
\label{lem:qary-primitive-chain}
Under the iid source and channel \eqref{eq:qary-test-channel}, the discrepancy
chain on $\mathcal D_R$ is irreducible and aperiodic. Moreover,
$\log_2(qQ_R(B_t\mid X_t,D_{t-1}))$ is bounded on every transition of positive
probability.
\end{lemma}

\begin{proof}
Every state has a positive self-loop by choosing $B_t=X_t$.  Whenever
$z+e_i\in\mathcal D_R$, choosing $(X_t,B_t)=(i,q)$ moves from $z$ to $z+e_i$
with positive probability; whenever $z-e_i\in\mathcal D_R$, choosing
$(X_t,B_t)=(q,i)$ moves to $z-e_i$.  Coordinate moves connect the entire box,
so the chain is irreducible, and the self-loops make its period one.  Although
some formal symbol pairs leave the box and have probability zero, the set of
positive-probability triples $(z,x,b)$ is finite and every corresponding
$Q_R(b\mid x,z)$ is strictly positive.  The displayed log score is therefore
bounded on its support.
\end{proof}

\begin{lemma}[Type-conditioned information stability]
\label{lem:qary-information-stability}
For fixed $(q,R)$,
\[
 \frac1m\imath_{E_{m,q}}(X^m;A,B^m)\longrightarrow c_{q,R}
 \quad\text{in probability}.
\]
\end{lemma}

\begin{proof}
Lemma~\ref{lem:qary-uniform-aux} gives $P(B^m)=q^{-m}$, so before
conditioning,
\begin{align*}
 \imath(X^m;A,B^m)
 &=\sum_{t=1}^m\log_2\big(qQ_R(B_t\mid X_t,D_{t-1})\big)\\
 &\quad+\log_2\frac{\pi_R(A)}{P(A\mid B^m)}.
\end{align*}
By Lemma~\ref{lem:qary-primitive-chain}, the sum is a bounded edge-additive
functional of a finite primitive Markov chain, with mean $mc_{q,R}$ and
exponentially small deviations by a finite-matrix
Perron-Frobenius Chernoff bound.  Zero-probability symbol pairs are absent
from the tilted transition matrix rather than assigned an infinite score.
Since $A$ has only $S_R$
values,
\[
 P\{-\log_2P(A\mid B^m)>u\}\le S_R2^{-u},
\]
so the final term is $o_P(m)$.  Stirling's formula gives
\[
 P(E_{m,q})
 =\frac{m!}{(m/q)!^q}q^{-m}
 =\Theta_q\big(m^{-(q-1)/2}\big).
\]
Thus exponential concentration survives conditioning on $E_{m,q}$.  The
identity
\begin{align*}
 \imath_{E_{m,q}}&=\imath+\log_2P(E_{m,q})\\
 &\quad-\log_2P(E_{m,q}\mid A,B^m)
\end{align*}
finishes the argument because, for every $\delta>0$,
\begin{align*}
 &P_{A,B^m\mid E_{m,q}}
 \{P(E_{m,q}\mid A,B^m)<2^{-m\delta}\}\\
 &\qquad\le\frac{2^{-m\delta}}{P(E_{m,q})}\longrightarrow0.
\end{align*}
\end{proof}

For completeness, we isolate the standard finite-length
variational-distance bound used here rather than attributing this exact display
to a numbered theorem in the cited papers.
\begin{lemma}[One-shot finite-codebook resolvability]
\label{lem:qary-one-shot-resolvability}
For every codebook size $M$ and threshold $r$ in bits, some uniform
$M$-position input codebook has output law $\widetilde P$ with
\begin{equation}
 d_{\rm TV}(\widetilde P,U_{m,q})
 \le P\{\imath_{E_{m,q}}(X^m;A,B^m)>r\}
    +\frac12\sqrt{\frac{2^r}{M}}.
 \label{eq:qary-finite-resolvability}
\end{equation}
\end{lemma}
\begin{proof}
Draw $Z_1,\ldots,Z_M$ iid from the plan marginal
$P_Z=P_{A,B^m\mid E_{m,q}}$ and let a uniform index map deterministically to
$Z_J$.  Truncate $P_{X^m\mid Z}$ to the event
$\imath_{E_{m,q}}\le r$.  The expected TV contribution of the discarded part
is at most its joint probability.  On the retained part,
$P_{X^m\mid Z}(x\mid z)\le2^rU_{m,q}(x)$; independence and Cauchy-Schwarz
therefore bound the expected centered empirical fluctuation by
$\frac12\sqrt{2^r/M}$.  Some deterministic realization of the iid codebook
attains the displayed expectation bound.  This is the standard one-shot
truncation proof behind general-source/channel resolvability
\citep{hanVerdu1993,yagi2017}.
\end{proof}
The codebook is a multiset.  Merging repeated inputs produces arbitrary
mixture weights on at most $M$ distinct transition tables and therefore does
not increase plan-library size.  The lemma, together with Lemmas
\ref{lem:qary-bayes-compiler} and \ref{lem:qary-information-stability}, yields
a uniform codebook of at most $2^{m(c_{q,R}+o(1))}$ fixed-order width-$S_R$
plans whose mixture converges in TV to $U_{m,q}$.  Therefore
\begin{equation}
\begin{aligned}
 \mathcal R^{\rm fix}_q(S_R)&\le c_{q,R}\le C_qR^{-2}\\
 &=O_q\big(S_R^{-2/(q-1)}\big).
\end{aligned}
 \label{eq:qary-upper}
\end{equation}
For arbitrary large $S$, choose the largest $R$ with $S_R\le S$ and use
monotonicity.

The constants can be read from the proof.  For grid widths
$S_R=(2R+1)^{q-1}$, the shift triangle inequality in
\eqref{eq:qary-dirichlet-energy} gives the explicit, nonoptimized bound
\[
 c_{q,R}\le \frac{4\pi^2q^2}{\ln2\,(R+1)^2}
 \le \frac{16\pi^2q^2}{\ln2}\,S_R^{-2/(q-1)}.
\]
If $\kappa_q$ is the uniform Faber-Krahn constant in
Lemma~\ref{lem:moving-layer-support}, then
$-\log_2(1-u)\ge u/\ln2$ gives the readable lower constant
$a_q=\kappa_q/(q\ln2)$ on the same grid.  Passing to arbitrary widths changes
these constants only by a fixed $q$-dependent factor.

\subsection{Matching spectral converse}

Lemma~\ref{lem:moving-layer-support} limits the support of one width-$S$ plan
to
\[
 \big(q-\kappa_qS^{-2/(q-1)}\big)^m
\]
for some $\kappa_q>0$.  Since
$|\Omega_{m,q}|=q^m m^{-(q-1)/2}\Theta_q(1)$, any mixture within fixed TV
$\epsilon<1$ needs
\[
 K^{\rm fix}_{m,q,\epsilon}(S)
 \ge(1-\epsilon)\frac{|\Omega_{m,q}|}
 {\big(q-\kappa_qS^{-2/(q-1)}\big)^m}.
\]
Taking per-symbol logarithms gives
\begin{equation}
\begin{aligned}
 \mathcal R^{\rm fix}_q(S)
 &\ge\log_2q-\log_2\big(q-\kappa_qS^{-2/(q-1)}\big)\\
 &=\Omega_q\big(S^{-2/(q-1)}\big),
\end{aligned}
 \label{eq:qary-lower}
\end{equation}
which combines with \eqref{eq:qary-upper} to prove Theorem
\ref{thm:qary-resource-law}.

\begin{remark}[What is classical and what is new]
The finite-state spectral gap, product ground state, and channel
resolvability are classical ingredients.  The claimed contribution is their
operational coupling: the root-lattice posterior channel converts a
resolvability codebook into exact fixed-order conservation-respecting plans,
and thereby identifies the number $q-1$ of independent count constraints as
the exponent parameter of the front-loaded-plan versus online-memory exchange
within balanced type classes.
\end{remark}

\begin{remark}[Numerical verification]
Exact enumeration verifies iid-uniform $B^m$ and exact Bayes reconstruction
to errors below $7\times10^{-15}$ for $(q,m,R)=(3,6,1)$ and $(4,4,1)$.
Stationary calculations over widths up to $12{,}321$ for $q=3$ and $79{,}507$
for $q=4$ give log-log slopes $-0.967$ and $-0.630$, approaching the
predicted $-1$ and $-2/3$; the binary slope is $-1.933$, approaching $-2$.
\end{remark}

%% file: sections/seeded_library_appendix.tex
\section{Linear-description seeded plan libraries}
\label{app:seeded-library}

The selector-rate achievability in Theorem~\ref{thm:qary-resource-law}
first draws a random resolvability codebook.  Listing its exponentially many
plans would make the deployment ledger vacuous even though the per-use selector
rate is small.  We now show that the same covering argument needs only a
linear-length, random-access seed.  Pairwise-independent convex-split ensembles
and algebraically structured soft-covering codebooks are not new
\citep{anshuJainWarsi2019,atifPadakandlaPradhan2021}; the point here is that
they inherit the exact-support finite-state posterior implementation and close
its provisioning ledger.

Let $W_z$ be a finite channel, let $Q_Z$ be an input law, and put
$Q_X=\sum_zQ_Z(z)W_z$.  For a threshold $r$ in bits, define
\[
 \mathcal T_r=\left\{(z,x):
 \log_2\frac{W_z(x)}{Q_X(x)}\le r\right\},
 \qquad \delta_r=Q_{ZX}(\mathcal T_r^c).
\]

\begin{proposition}[Pairwise-independent one-shot covering]
\label{prop:pairwise-one-shot-covering}
Suppose $Z_1,\ldots,Z_K$ are pairwise independent and each has law $Q_Z$.
For $Q_{\mathcal C}=K^{-1}\sum_{i=1}^K W_{Z_i}$,
\begin{equation}
 \mathbb E\,d_{\rm TV}(Q_{\mathcal C},Q_X)
 \le \delta_r+\frac12\sqrt{\frac{2^r}{K}}.
 \label{eq:pairwise-one-shot-covering}
\end{equation}
Thus full codeword independence is unnecessary in
Lemma~\ref{lem:qary-one-shot-resolvability}.
\end{proposition}
\begin{proof}
Write $W_z^{\mathcal T}(x)=W_z(x)\mathbf1\{(z,x)\in\mathcal T_r\}$ and
$Q_X^{\mathcal T}=\mathbb E W_Z^{\mathcal T}$.  The expected omitted empirical
mass and the omitted population mass contribute together at most $\delta_r$
in TV.  Pairwise independence cancels every cross-codeword covariance, so for
each $x$,
\[
 \operatorname{Var}\!\left[\frac1K\sum_i
 W_{Z_i}^{\mathcal T}(x)\right]
 \le \frac1K\mathbb E W_Z^{\mathcal T}(x)^2.
\]
On $\mathcal T_r$,
$W_Z^{\mathcal T}(x)\le2^rQ_X(x)$, and hence
\[
 \mathbb E W_Z^{\mathcal T}(x)^2
 \le2^rQ_X(x)\mathbb E W_Z^{\mathcal T}(x)
 \le2^rQ_X(x)^2.
\]
Jensen's inequality, followed by summation over $x$, bounds the expected
retained $\ell_1$ fluctuation by $\sqrt{2^r/K}$.  Dividing by two and adding
the omitted mass proves~\eqref{eq:pairwise-one-shot-covering}.
\end{proof}

\begin{lemma}[Two-field-element random-access library]
\label{lem:affine-seeded-library}
Let $P_Z$ have finite support, let
$p_{\min}=\min_{z:P_Z(z)>0}P_Z(z)$, and fix $K$ and $0<\eta<1/2$.
Put
\[
 M=\max\left\{K,\left\lceil\frac{2}{\eta p_{\min}}\right\rceil\right\}.
\]
There is a prime $M\le p<2M$
and an inverse-interval map $f_p:\mathbb F_p\to\operatorname{supp}P_Z$
whose image law $Q_Z$ obeys
$(1-\eta)P_Z(z)\le Q_Z(z)\le(1+\eta)P_Z(z)$.
For $(a,b)\sim\operatorname{Unif}(\mathbb F_p^2)$,
\begin{equation}
 Z_i=f_p(ai+b),\qquad 0\le i<K,
 \label{eq:affine-seeded-library}
\end{equation}
are pairwise independent with marginal $Q_Z$.  The whole library is therefore
random-access from $2\lceil\log_2p\rceil$ seed bits.
\end{lemma}
\begin{proof}
Round the $pP_Z(z)$ values by the largest-remainder rule.  The resulting
integer counts sum to $p$ and differ from $pP_Z(z)$ by at most one.  The choice
of $p$ gives the stated relative error; assigning consecutive field elements
according to these counts defines $f_p$.  For distinct $i,j<p$, the linear map
$(a,b)\mapsto(ai+b,aj+b)$ is a bijection of $\mathbb F_p^2$, proving pairwise
independence.  Bertrand's postulate supplies a prime below twice the displayed
threshold.
\end{proof}

\begin{corollary}[Seeded deployment of the posterior compiler]
\label{cor:seeded-posterior-deployment}
For fixed $q$ and discrepancy radius $R$, the fixed-order posterior
achievability in Theorem~\ref{thm:qary-resource-law} admits codebooks with the
same asymptotic selector rate, vanishing TV error, exact balanced support, and
width $S_R=(2R+1)^{q-1}$, while the entire codebook has an
$O_{q,R}(m)$-bit executable description.  Consequently
\begin{equation}
 R_{\rm deploy}^{(N)}
 \le \mathcal R_q^{\rm fix}(S_R)+o_m(1)+O_{q,R}(N^{-1}).
 \label{eq:seeded-deployment-rate}
\end{equation}
If $S$ grows, the same construction has description length
$O_q(m\log S)$; preserving the $S^{-2/(q-1)}$ selector scale is guaranteed by
$N\gg S^{2/(q-1)}\log S$.
\end{corollary}
\begin{proof}
Let $P_Z=P_{A,B^m\mid E_{m,q}}$ be the plan marginal in
Section~\ref{sec:qary-plan-memory-law}.  At fixed $(q,R)$, let $\theta_{q,R}>0$
be the smallest nonzero transition probability of the finite test channel and
let $\pi_{q,R}>0$ be the smallest initial-state mass.  Every plan in the
support has at least one compatible balanced source word, so
\[
 p_{\min,m}\ge \pi_{q,R}(\theta_{q,R}/q)^m.
\]
Choose $\eta_m\downarrow0$ subexponentially and apply
Lemma~\ref{lem:affine-seeded-library}.  The induced input and output laws are
multiplicatively within $1\pm\eta_m$ of the originals.  Their information
densities therefore differ by $o(1)$, while their joint laws differ by
$o(1)$ in TV.  Proposition~\ref{prop:pairwise-one-shot-covering} and
Lemma~\ref{lem:qary-information-stability} now give the same spectral
sup-information rate and a good fixed affine seed.  Moreover
$\log_2p=O_{q,R}(m)$, proving~\eqref{eq:seeded-deployment-rate}.

No exponential CDF table is hidden in this description.  Prefix masses of
$(A,B_1,\ldots,B_t)$ are computed by dynamic programming over the discrepancy
and the remaining source histogram, with $O_q(S_Rm^{q-1})$ states per layer.
An interval decoder of the Han--Hoshi type \citep{hanHoshi1997} then unpacks
$f_p(u)$.  In this finite-field construction every interval endpoint is an
integer in $[0,p]$, so all endpoint comparisons use $O(\log p)$-bit integers;
the bit-width claim follows directly here rather than being attributed to the
cited paper.  Finally, the cosine ground state gives
$-\log_2\theta_{q,R}=O_q(\log R)$ and
$-\log_2\pi_{q,R}=O_q(\log R)$, yielding the growing-$S$ statement.
\end{proof}

\begin{remark}[What is and is not optimal]
The result removes exponential library storage and gives a polynomial-time
random-access descriptor.  It does not claim universal minimum program length:
a specially structured target may admit a shorter deterministic description.
The proved optimality remains the selector-rate exponent under the stated
shared-library interface.
\end{remark}

An exact finite-alphabet calculation at $(m,R,S)=(10,2,5)$ compares iid and affine
pairwise-independent libraries over 500 trials.  At
$K=4,16,64,256$, their mean TVs were respectively
$(.1599,.0851,.0410,.0211)$ and $(.1605,.0797,.0398,.0181)$.
The experiment checks the construction; Proposition
\ref{prop:pairwise-one-shot-covering} is the guarantee.

%% file: sections/quota_safe_router_appendix.tex
\section{Restricted-interface corollary: quota-safe streaming allocation}
\label{app:quota-safe-routing}

Fix a batch of $m$ token representations and $q\mid m$ experts.  A
$(B,C)$ \emph{quota-safe streaming router} has a control plane that observes
the batch and selects a value in an alphabet of size at most $2^C$ before
routing begins.  Conditional on that value, the data plane processes tokens
in a fixed order, retains at most $B$ bits of persistent state, may use free
private randomness, and routes every realization with exactly $m/q$ tokens
per expert.  The fixed batch, time index, and plan-specific transition tables
are free side information.  This generous convention only weakens the lower
bound. Here $B$ counts persistent routing-state bits, not expert-activation
memory, and $C$ counts the logarithm of the preselected transition-table
library, not wall-clock communication. Existing MoE implementations may use
joint dispatch, token dropping, or post-hoc rebalancing and need not instantiate
this interface; we report no systems-level latency or memory measurement.

For a target routing law $P_m$ on the balanced assignments, define its smooth
support
\[
 s_{P_m}(\epsilon)=\min\{|A|:P_m(A)\ge1-\epsilon\}.
\]
Put $S=2^B$ and let
\[
 \Lambda_q(S)=q-c_qS^{-2/(q-1)},
\]
where $c_q$ is the constant in
Lemma~\ref{lem:moving-layer-deficit}.  For $q=2$ use the sharper value
$\Lambda_2(S)=2\cos\{\pi/(2S+1)\}$.

\begin{proposition}[Restricted quota-safe control bound]
\label{prop:quota-safe-router}
If a $(B,C)$ quota-safe streaming router induces $Q_m$ with
$d_{\rm TV}(P_m,Q_m)\le\epsilon$, then
\begin{equation}
 C\ge
 \left[\log_2s_{P_m}(\epsilon)
       -m\log_2\Lambda_q(2^B)\right]_+.
 \label{eq:quota-safe-router-lb}
\end{equation}
For the uniform balanced routing law, the optimal asymptotic control-plane
rate satisfies, for every fixed $q$ and all sufficiently large $B$,
\begin{equation}
 \lim_{\epsilon\downarrow0}\limsup_{m\to\infty:q\mid m}
 \frac{C^*_{m,q,\epsilon}(B)}m
 =\Theta_q\!\left(2^{-2B/(q-1)}\right),
 \label{eq:router-memory-law}
\end{equation}
where $C^*$ is the minimum number of control bits in the model above.
\end{proposition}

\begin{proof}
Condition on one control-plane value.  The resulting data plane is a
exact-support fixed-order width-$2^B$ generator, so
Lemma~\ref{lem:moving-layer-support} bounds its output support by
$\Lambda_q(2^B)^m$.  If $T$ is the union of the supports over all control
values, then
\[
 |T|\le2^C\Lambda_q(2^B)^m,
 \qquad Q_m(T)=1.
\]
Total-variation closeness gives $P_m(T)\ge1-\epsilon$, and therefore
$s_{P_m}(\epsilon)\le|T|$, proving
Equation~\eqref{eq:quota-safe-router-lb}.  For the uniform balanced law,
$m^{-1}\log_2s_{P_m}(\epsilon)\to\log_2q$.  The converse just proved and the
exact-support posterior construction in Theorem~\ref{thm:qary-resource-law} give the
matching order in Equation~\eqref{eq:router-memory-law} after substituting
$S=2^B$.
\end{proof}

\begin{corollary}[No-free exact-balanced dispatch]
\label{cor:quota-safe-escape}
Fix $q$ and a sequence $\epsilon_m\le\epsilon_0<1$.  A quota-safe router whose
target is uniform over balanced assignments obeys the following: constants
$A_q,D_q>0$ exist such that, with
\begin{equation}
\begin{aligned}
 \Delta_m=C_m+D_q\!\left(1+\log_2(m+1)
 +\log_2\frac1{1-\epsilon_m}\right),
 \label{eq:quota-safe-delta}
\end{aligned}
\end{equation}
one has
\begin{equation}
 B_m\ge \frac{q-1}{2}\log_2\frac{m}{\Delta_m}-A_q
 \label{eq:quota-safe-memory-escape}
\end{equation}
whenever $\Delta_m=o(m)$.  In particular, bounded persistent routing state
forces $C_m=\Omega_q(m)$, while sublinear front-loaded control forces
$B_m\to\infty$.
\end{corollary}
\begin{proof}
For the uniform law,
$s_{P_m}(\epsilon_m)=\lceil(1-\epsilon_m)|\Omega_{m,q}|\rceil$ and Stirling's
formula gives
\[
 \log_2|\Omega_{m,q}|
 =m\log_2q-\frac{q-1}{2}\log_2m+O_q(1).
\]
For all sufficiently large $S$, the moving-layer deficit implies
$\log_2q-\log_2\Lambda_q(S)\ge a_qS^{-2/(q-1)}$ for a constant $a_q>0$.
Choose $D_q$ large enough to absorb the uniform Stirling remainder and the
rounding in the smooth support.  Substitution in
\eqref{eq:quota-safe-router-lb} and rearrangement then give
$\Delta_m\ge a_qm2^{-2B_m/(q-1)}$, which is equivalent to
\eqref{eq:quota-safe-memory-escape}.  Bounded $S$ gives the first consequence
directly.
\end{proof}

The corollary is an architectural design trilemma.  A high-entropy balanced
router cannot simultaneously keep fixed physical order, exact no-drop quotas,
bounded persistent state, sublinear front-loaded control, and vanishing TV
error.  Table~\ref{tab:router-escape-audit} maps prominent MoE designs to the
capability each one purchases to avoid this frontier.  The comparison explains
resource placement; it does not claim that these systems instantiate the
restricted executor.

\begin{table*}[t]
\centering
{\small
\setlength{\tabcolsep}{4pt}
\begin{tabular}{@{}p{.15\textwidth}p{.40\textwidth}p{.37\textwidth}@{}}
\toprule
System & Architectural relaxation purchased & Frontier cost thereby avoided \\
\midrule
Switch Transformer \citep{fedus2022switch}
& Overflow tokens bypass the sparse layer through the residual path.
& Relaxes exact no-drop quotas, so overflow mass need not be represented by
exact-support plans. \\
Expert Choice \citep{zhou2022expertchoice}
& Experts select batchwise top-$k$ tokens, and a token may reach a variable
number of experts.
& Replaces fixed-order streaming by global batchwise assignment, moving
coordination into the scheduler. \\
MegaBlocks \citep{gale2023megablocks}
& Dynamic block-sparse dispatch accommodates realized, generally imbalanced
expert loads without token dropping.
& Replaces bounded-state equal-quota streaming by dynamic dispatch, spending
runtime scheduling and storage instead. \\
\bottomrule
\end{tabular}
}
\caption{Architectural escape map for
Corollary~\ref{cor:quota-safe-escape}.  Deployed systems avoid the predicted
coordination--memory frontier by purchasing different capabilities: dropping,
batchwise assignment, or dynamic dispatch.  The table concerns resource
placement, not latency or activation memory.}
\label{tab:router-escape-audit}
\end{table*}

\paragraph{Affinity-aware targets and the boundary of the claim.}
For scores $u_{t,a}$, an entropy-regularized balanced target has
\[
 P_\beta(x\mid u)\propto
 \exp\!\left(\beta\sum_{t=1}^m u_{t,x_t}\right)
 \mathbf 1\{N_a(x)=m/q\ \forall a\}.
\]
Equation~\eqref{eq:quota-safe-router-lb} applies pointwise in the observed
batch after computing or lower-bounding $s_{P_\beta}(\epsilon)$.  It becomes
vacuous, correctly, when the desired routing output is one deterministic
balanced assignment.  The result certifies the cost of preserving a
stochastic routing law under exact no-drop quotas; it is not a lower bound
for merely finding one feasible assignment, nor an assertion that existing
routers use this interface.

%% file: sections/conditional_poisson_beam_appendix.tex
\section{Conditional-Poisson beam admission}
\label{app:conditional-poisson-beam}

Conditional-Poisson stochastic beam search (CPSBS) replaces a beam-search
top-\(k\) operation by sampling exactly \(k\) candidates without replacement
\citep{meister2021conditional}. For a candidate pool
\(\mathcal B=\{1,\ldots,n\}\) and positive model weights
\(w=(w_1,\ldots,w_n)\), its one-step law is
\begin{equation}
\begin{aligned}
 Q_w(x)
 &=\frac{\prod_{i=1}^n w_i^{x_i}}{e_k(w)}\\[-1mm]
 &\quad\times\mathbf 1\!\left\{\sum_{i=1}^n x_i=k\right\},\\
 e_k(w)&=\sum_{|A|=k}\prod_{i\in A}w_i .
\end{aligned}
 \label{eq:cpsbs-law}
\end{equation}
This target is exactly, rather than approximately, in our model class.
Indeed, if independent Bernoulli variables have odds
\(p_i/(1-p_i)=w_i\), their joint mass is
\(\prod_i(1+w_i)^{-1}\prod_iw_i^{x_i}\). Conditioning on
\(\sum_iX_i=k\) cancels the first factor and gives~\eqref{eq:cpsbs-law}.

\begin{proposition}[Irrevocable candidate-admission bound]
\label{prop:cpsbs-admission}
Suppose a controller chooses one of at most \(2^C\) programs before candidates
are visited in a fixed order. Conditional on that program, an executor has at
most \(B\) persistent bits, may use fresh private randomness, and emits each
admission bit irrevocably. Every realization admits exactly \(k\) candidates.
If its output law \(\widehat Q\) satisfies
\(d_{\rm TV}(\widehat Q,Q_w)\leq\epsilon\), then
\begin{equation}
 C\geq
 \left[
 \log_2s_{Q_w}(\epsilon)
 -n\log_2\!\left\{
 2\cos\!\left(\frac{\pi}{2^{B+1}+1}\right)
 \right\}
 \right]_+ .
 \label{eq:cpsbs-admission-lb}
\end{equation}
Conversely, \(Q_w\) has an exact one-program implementation with a
width-\((k+1)\) remaining-quota state and a precomputed
\((n+1)(k+1)\) suffix table.
\end{proposition}

\begin{proof}
The odds calculation above identifies \(Q_w\) with a heterogeneous binary
product conditioned on its count. A \(B\)-bit irrevocable admission program is
therefore a plan of state width at most \(2^B\). Applying the binary
moving-layer support bound
\(\Lambda_2(S)=2\cos\{\pi/(2S+1)\}\) and taking the union over at most \(2^C\)
programs proves~\eqref{eq:cpsbs-admission-lb}.

For achievability, define suffix elementary symmetric polynomials
\[
 E_{i,r}=e_r(w_i,\ldots,w_n),\qquad
 E_{i,r}=E_{i+1,r}+w_iE_{i+1,r-1}.
\]
At candidate \(i\), with \(r\) admissions remaining, admit it with probability
\(w_iE_{i+1,r-1}/E_{i,r}\). The recurrence makes this a valid probability;
boundary states force exactly \(k\) admissions, and multiplying transition
probabilities telescopes to~\eqref{eq:cpsbs-law}. The online state is
\(r\in\{0,\ldots,k\}\). The table contains \((n+1)(k+1)\) entries.
\end{proof}

The proposition does not lower-bound the runtime of standard CPSBS. Its global
elementary-symmetric dynamic program is precisely the table-based endpoint
above. Nor does it cover reservoir samplers that retain and replace earlier
choices: those do not emit irrevocable admission bits and generally implement
a different without-replacement law. The result prices the narrower design
question of reducing persistent admission state while preserving exact
fixed-size support.

\paragraph{Exact sharded decoding.}
The same law has a second operational consequence when a vocabulary-parallel
LM head leaves candidate weights on different devices. Let
\(\mathcal B=I_1\mathbin{\dot\cup}\cdots\mathbin{\dot\cup}I_g\), put
\(e_{j,r}=e_r((w_i)_{i\in I_j})\), and write
\(R_j=\sum_{i\in I_j}X_i\) for the number of accepted candidates on shard \(j\).

\begin{proposition}[Shardwise factorization and minimum coordination]
\label{prop:cpsbs-sharded}
Under \(Q_w\),
\begin{equation}
 \Pr\{R=\mathbf r\}
 =\frac{\mathbf 1\{\sum_jr_j=k\}}{e_k(w)}
   \prod_{j=1}^g e_{j,r_j}.
 \label{eq:cpsbs-quota-profile}
\end{equation}
Conditional on \(R=\mathbf r\), the shard subsets are independent
conditional-Poisson samples of sizes \(r_1,\ldots,r_g\).
Among exact mixtures of shard-product component laws, the minimum common-plan
entropy and support size are respectively
\[
 \min H(Z)=H(R),
 \qquad
 \min|\supp Z|=|\supp R|.
\]
Consequently, \(Q_w\) can be sampled without gathering the vocabulary logits:
each shard sends the nonconstant coefficients \(e_{j,1},\ldots,e_{j,k}\) (with \(e_{j,0}=1\) known), a coordinator samples
\(\mathbf r\) from~\eqref{eq:cpsbs-quota-profile}, and shard \(j\) runs its
local suffix dynamic program with quota \(r_j\).
\end{proposition}

\begin{proof}
Summing~\eqref{eq:cpsbs-law} over subsets with profile \(\mathbf r\) factors
the unnormalized mass into \(\prod_je_{j,r_j}\), proving
\eqref{eq:cpsbs-quota-profile}. Dividing each factor by \(e_{j,r_j}\) gives
the conditional product law.

For the converse, every positive-weight component in an exact mixture must
itself assign zero mass outside \(\sum_jR_j=k\). Within a shard-product
component the \(R_j\) are independent, so
\(\sum_j\operatorname{Var}(R_j\mid Z=z)=0\); hence every \(R_j\) is
deterministic given \(z\). Thus \(R\) is a function of \(Z\), which gives
\(H(Z)\ge H(R)\) and \(|\supp Z|\ge|\supp R|\). Choosing \(Z=R\) and using
the conditional factorization attains both bounds. The stated protocol
implements this construction. Each coefficient vector takes
\(O(|I_j|k)\) arithmetic and \(O(k)\) working memory; coordinator convolution
and the local suffix recurrences are exact in real arithmetic.
\end{proof}

\begin{theorem}[Minimal deterministic shard interface]
\label{thm:cpsbs-interface}
Fix a shard \(I_j\) such that both it and its complement contain at least
\(k\) candidates. Before any sampling, let the shard send a deterministic
continuous summary \(M_j(w_{I_j})\in\mathbb R^d\). If a one-round coordinator
must reproduce the exact quota law for every positive weight vector, then
\(d\ge k\), even if all complementary weights are revealed to it. Sending
\((e_{j,1},\ldots,e_{j,k})\) attains equality.

For every fixed \(k\), there is also a compact regular weight family on which
the optimal deterministic message length at quota-law TV error \(\eta\) is
\begin{equation}
 B_k^\star(\eta)=k\log_2(1/\eta)+O_k(1),
 \qquad \eta\downarrow0.
 \label{eq:cpsbs-interface-bits}
\end{equation}
Thus the coefficient interface is locally precision-optimal as well as
exact-dimension-optimal.
\end{theorem}

\begin{proof}
Write \(a_r=e_r(w_{I_j})\), \(b_r=e_r(w_{\mathcal B\setminus I_j})\), and
\(a_0=b_0=1\). The marginal quota law is
\[
 p_r:=\Pr\{R_j=r\}
 =\frac{a_rb_{k-r}}{\sum_{\ell=0}^ka_\ell b_{k-\ell}},
 \qquad r=0,\ldots,k.
\]
Fix any positive complementary weights, so every \(b_r>0\). If two local
weight vectors have the same message, the coordinator receives the same input
and must produce the same \(p\). But
\begin{equation}
 a_r=\frac{p_r}{p_0}\frac{b_k}{b_{k-r}},
 \qquad r=1,\ldots,k,
 \label{eq:cpsbs-coeff-recovery}
\end{equation}
so their first \(k\) elementary-symmetric coefficients coincide.

Fix all but \(k\) local weights and vary the rest in the open chamber
\(0<w_1<\cdots<w_k\). Multiplication by the fixed weights' generating
polynomial applies a unit lower-triangular map to the first \(k\) coefficients,
so those of the varying factor remain recoverable. They determine
\(\prod_i(t+w_i)\), hence the ordered weights. Therefore \(M_j\) restricts to
a continuous injection from an open subset of \(\mathbb R^k\) into
\(\mathbb R^d\). For \(d<k\), padding its image with zeros contradicts
Brouwer's invariance-of-domain theorem. Hence \(d\ge k\); the protocol in
Proposition~\ref{prop:cpsbs-sharded} gives equality.

For~\eqref{eq:cpsbs-interface-bits}, the elementary-symmetric Jacobian on the
ordered chamber has absolute determinant
\(\prod_{i<\ell}|w_i-w_\ell|>0\). Together with
\eqref{eq:cpsbs-coeff-recovery}, the quota map is a local diffeomorphism.
Choose a small compact cube where it is bi-Lipschitz. Two parameters sharing
one \(B\)-bit deterministic message have quota laws within \(2\eta\) TV, so a
separated grid of \(\Omega_k(\eta^{-k})\) parameters requires distinct
messages. This gives the lower bound. Quantizing the \(k\) coefficient
coordinates on the same compact image gives \(O_k(\eta^{-k})\) messages and
the upper bound. Conditional shard laws still use the original weights, so the
final subset-law TV error equals the quota-law error.
\end{proof}

This extends the hierarchical \(k=1\) categorical decomposition used by
distributed exact samplers such as FlashSampling
\citep{ruiz2026flashsampling}: one local mass is then necessary and sufficient.
Gumbel-top-\(k\) sampling implements a different without-replacement law.
The theorem concerns deterministic one-round summaries; it is not an
interactive, randomized, or global floating-point lower bound. Log-coefficients
and log-sum-exp give the stable implementation. In the released local-family
audit (4,000 points for each \(k\in\{2,4,6,8\}\)), log worst-case TV has slope
between \(-1.018\) and \(-1.013\) per coefficient bit; 14 bits per coefficient
gives worst TV at most \(1.45\times10^{-6}\). The entropy statement above
separately prices common random coordination after the weight-dependent
profile is available.

\paragraph{Finite real-logit and sharding audits.}
On 100 pinned UltraFeedback prompts, a cached DistilGPT2 model supplies the top
22 logits and all $\binom{22}{11}=705{,}432$ candidate sets are enumerated.
At temperature one, the median $95\%$-smooth support is 202,149, the two-state
plan lower bound is six, and that lower bound is nonvacuous on 98 prompts.
For exact 16-of-50,257 sharded sampling, direct polynomial convolution gives
maximum factorization error $1.95\times10^{-16}$; with deterministic balanced
hash shards, the median quota entropies for $g=2,4,8$ are respectively
$2.86$, $7.53$, and $14.56$ bits.  Equal-precision compiled post-logit timing
shows an overhead crossover rather than monotone scaling: median speedups are
$1.13$, $1.25$, and $0.97$.  These values audit the target-side quantities and
interface implementation, not decoding quality or multi-GPU latency.